\documentclass[a4paper, USenglish, cleveref, thm-restate]{lipics-v2021}

\title{Admission Control of Quasi-Reversible Queueing Systems: Optimization and Reinforcement Learning}

\titlerunning{Admission Control of Quasi-Reversible Queueing Systems}

\author{Céline Comte}{LAAS-CNRS, Université de Toulouse, CNRS, Toulouse, France}{celine.comte@laas.fr}{https://orcid.org/0009-0005-9413-7124}{}

\author{Pascal Moyal}{IECL, Université de Lorraine / Inria PASTA, Nancy, France}{pascal.moyal@univ-lorraine.fr}{https://orcid.org/0000-0002-6638-5551}{}

\authorrunning{C. Comte and P. Moyal}

\Copyright{Céline Comte and Pascal Moyal}



\acknowledgements{Thank you to Fabien Mathieu for insightful discussions that helped construct \Cref{sec:cost}.}

\nolinenumbers

\hideLIPIcs

\usepackage{algorithm}
\usepackage{algorithmicx}
\usepackage{algpseudocode}
\usepackage{makecell}
\usepackage[dvipsnames]{xcolor}

\usepackage[acronym]{glossaries}
\glsdisablehyper

\newacronym{AC}{AC}{actor-critic}
\newacronym{BAC}{BAC}{balanced admission control}
\newacronym{CCDF}{CCDF}{complementary cumulative distribution function}
\newacronym{CTMC}{CTMC}{continuous-time Markov chain}
\newacronym{DTMC}{DTMC}{discrete-time Markov chain}
\newacronym{FCFM}{FCFM}{first-come-first-matched}
\newacronym{FCFS}{FCFS}{first-come-first-served}
\newacronym{LP}{LP}{linear program}
\newacronym{MDP}{MDP}{Markov decision process}
\newacronym{OI}{OI}{order-independent}
\newacronym{PASTA}{PASTA}{Poisson arrivals see time averages}
\newacronym{PGRL}{PGRL}{policy-gradient reinforcement learning}
\newacronym{RL}{RL}{reinforcement learning}
\newacronym{SAGE}{SAGE}{score-aware gradient estimator}

\usepackage{amsmath}
\usepackage{amssymb}
\usepackage{bbm}  
\usepackage{stmaryrd}  

\newcommand\bN{\mathbb{N}}
\newcommand\bR{{\mathbb{R}}}

\newcommand\bs{{s}}
\newcommand\bt{{t}}

\newcommand\cA{\mathcal{A}}
\newcommand\cB{\mathcal{B}}
\newcommand\cF{\mathcal{F}}
\newcommand\cI{\mathcal{I}}
\newcommand\cP{\mathcal{P}}

\newcommand\cR{\mathcal{R}}
\newcommand\cS{\mathcal{S}}
\newcommand\cV{\mathcal{V}}
\newcommand\cX{\mathcal{X}}

\newcommand\cL{\mathcal{L}}

\newcommand\un{\llbracket 1,n\rrbracket}

\newcommand\bE[1]{\mathbb{E}\left(#1\right)}

\newcommand\one{\mathbbm{1}}
\newcommand\indicator[1]{\one\left[#1\right]}

\newcommand\rcont{r_{\text{cont}}}
\newcommand\rdisc{r_{\text{disc}}}

\DeclareMathOperator*{\argmin}{arg\,min}
\DeclareMathOperator*{\argmax}{arg\,max}

\newcommand\gammabac{{\partial\Gamma}}

\newcommand\rec{{\text{rec}}}
\newcommand\card{{\textrm{card}}}
\newcommand\supp{\textrm{supp}}

\allowdisplaybreaks

\usepackage{tikz}
\usetikzlibrary{calc}

\theoremstyle{definition}
\newtheorem{assumption}{Assumption}

\begin{document}

\maketitle

\begin{abstract}
	In this paper, we introduce a versatile scheme for optimizing the arrival rates of quasi-reversible queueing systems. We first propose an alternative definition of quasi-reversibility that encompasses reversibility and highlights the importance of the definition of customer classes. Then we introduce balanced arrival control policies, which generalize the notion of balanced arrival rates introduced in the context of Whittle networks, to the much broader class of quasi-reversible queueing systems. We prove that supplementing a quasi-reversible queueing system with a balanced arrival-control policy preserves the quasi-reversibility, and we specify the form of the stationary measures. We revisit two canonical examples of quasi-reversible queueing systems, Whittle networks and order-independent queues. Lastly, we focus on the problem of admission control and leverage our results in the frameworks of optimization and reinforcement learning. 
\end{abstract}

\section{Introduction} \label{sec:introduction}

In this paper, we introduce a versatile scheme for the control of arrival rates in {\em quasi-reversible} 
queueing systems, a broad class of multi-class queueing systems that has been studied extensively in the literature~\cite{kelly,K11,C11}. Our definition of quasi-reversibility differs slightly from the definitions given in these references.  
In a nutshell, under our definition, the quasi-reversibility of a queueing system implies the existence of
an aggregation of the states of the natural Markov chain 
of the system into \emph{macrostates}, such that the macrostate process thereby obtained satisfies a reversibility property 
with respect to the stationary measures of the Markov chain. Note that the 
macrostate process is not Markov in general, yet it coincides in distribution with a reversible Markov chain, see \Cref{sec:reversibility} below.
This aggregation is obtained through a so-called \emph{counting function}, which maps the microstate to the macrostate by only retaining the number of customers of each class in the system.

This alternative definition of quasi-reversibility is instrumental to our first main contribution, namely, the definition of \emph{balanced} arrival-control policies. These generalize the balanced arrival rates introduced in the context of Whittle networks without internal routing~\cite{BJP04,J10}, to the much broader class of quasi-reversible queueing systems.
As our first main result (\Cref{theo:balanced_policies_preserve_quasi_reversibility}), we show that the quasi-reversibility of the queueing system under consideration is preserved by any balanced arrival-control policy, and we specify the form of the stationary measures in this case. Roughly speaking, we show that supplementing a quasi-reversible queueing system with a balanced arrival control boils down to multiplying the stationary measures by the balance function associated with the arrival control.

To illustrate the definition of balanced policies,
we then focus on the special case of admission control,
where arrival control is interpreted
as admitting or rejecting incoming customers.
For a quasi-reversible queueing system with a finite state space, we formulate the problem of finding the best balanced admission control policy as a linear optimization problem. This allows us to show that there exists a policy optimal among balanced policies that is deterministic (\Cref{theo:polytope}), thus generalizing a result shown in \cite{J10} in the context of Whittle networks.
We conclude with a digression, to compare the performance and structure of the best balanced policy with that of an optimal policy. To do so, we leverage both insights from the literature on (quasi-)reversible queueing systems and numerical results in a toy example.

Lastly, we show the potential interest of quasi-reversibility and balanced arrival control for implementing efficient \gls{RL} algorithms that optimize the access control. Given a family of parameterized balanced policies, we show that, thanks to quasi-reversibility, the stationary distribution of the parameterized system satisfies a simple gradient equation with respect to the system parameter. This allows us to construct a gradient estimator that we use in a gradient-ascent type algorithm common in \gls{PGRL}. The exploitation of product-form like stationary measures to implement this class of \gls{RL} algorithms, is reminiscent of the approach recently developed in \cite{CJSS24}. 

To illustrate these results, we discuss two canonical examples of quasi-reversible systems, which gather a wide class of queueing systems:
multi-type Whittle networks as defined in \cite{serfozo}, and \gls{OI} queues introduced in~\cite{B95,B96}. After checking that these two queueing systems (already known to be quasi-reversible, at least under some assumptions) satisfy our definition of quasi-reversibility,
we show that the \gls{PGRL} procedure defined above can be fruitfully applied to optimize admission control into a variant, with customer abandonment, of the redundancy model with customer-server compatibility introduced in \cite{G16} which, as observed in \cite{BC17}, is an example of \gls{OI} queues.

\subsection{Overview}

This paper is organized as follows. In \Cref{sec:quasi_reversible_queueing_systems}, we introduce quasi-reversible systems, and we review the literature on reversibility and discuss other variants of quasi-reversibility. Second, in \Cref{sec:balanced_arrival_rates}, we introduce the notion of \emph{balanced arrival control} in quasi-reversible systems and their interest in the context of \gls{RL}. 
In \Cref{sec:two_canonical_examples}, we show how this approach can be applied to two well-known general examples of quasi-reversible systems, namely, order-independent queues in \Cref{sec:oi}, and multi-class Whittle systems in \Cref{sec:multi_class_whittle}. Then, in \Cref{sec:admission_control} we set our optimization problem for admission control. Optimality of deterministic policies among balanced policies is shown in 
\Cref{sec:optimal_balanced_policies} for quasi-reversible queueing systems with a finite state space. The cost of implementing balanced policies is investigated in \Cref{sec:cost}.
In \Cref{sec:numerics}, we present a case study, and we implement an \gls{RL} algorithm to optimize the admission control in a redundancy model subject to server-customer compatibility. \Cref{sec:conclu} concludes this work.

\subsection{Notation} \label{sec:notation}

We let $\bN$, $\bN_{> 0}$, $\bR$, and $\bR_{>0}$ denote
the sets of non-negative integers, positive integers, reals, and positive reals, respectively. For $a\le b\in\bN$, we use 
$\llbracket a, b \rrbracket$ to denote the integer interval $\{a,a+1,\ldots,b\}$, and define $\llbracket a, b \llbracket$, $\rrbracket a, b \llbracket$ and $\rrbracket a, b \rrbracket$ accordingly. 

Fix $m,n\in\bN_{>0}$. We let $\mathscr M_{n,m}(\bN)$ denote the set of matrices with $n$ rows and $m$ columns of 
integer entries. With some abuse of notation, the null vectors of $\bR^n$ and the null matrix of $\mathscr M_{n,m}(\bN)$ are both denoted by $\mathbf 0$. For any $i\in\llbracket 1,n \rrbracket$, we let $e_i$ be the vector of $\bR^n$ with one in component~$i$ and zero elsewhere. Likewise, for all $i\in\un$ and $k\in\llbracket 1,m \rrbracket$, we let $e_{ik}$ be the matrix of $\mathscr M_{n,m}(\bN)$ with one in component $ik$, and zero elsewhere.
The vector space $\bR^n$ is equipped with the coordinate-wise semi-ordering, namely, for all $x=(x_1,\ldots,x_n),y=(y_1,\ldots,y_n)\in\bR^n$, $x\le y $ if and only if $x_i\le y_i$ for all $i\in\llbracket 1,n \rrbracket$. 
We also let $\bR_{>0}^n$ be the subset of vectors of $\bR^n$ having positive entries. 

A subset $\cX \subseteq \bN^n$ is called a \emph{Ferrers set}
if $0 \in \cX$ and $\cX$ is coordinate convex,
in the sense that, for $x \in \bN^n$, $x \in \cX$ implies $y \in \cX$
for each $y \in \bN^n$ such that $y \le x$.

Let $\cA$ be a finite set. We denote by $\card(\cA)$, the cardinality of $\cA$. We let $\star$ denote the Kleene-star operator of~$\cA$, namely, $\cA^\star$ is the set of words on the alphabet $\cA$. A word $w\in \cA^\star$ is then written as $w=w_1\cdots w_\ell$, where $\ell\in\bN_{>0}$ is the \emph{length} of $w$. The empty word of $\cA^\star$ is denoted by $\varnothing$, is supposed to be of null length. 

The identity function of any set into itself is denoted by~$\textrm{Id}$. Given a real-valued function~$f$, $\supp(f)$ denotes the support of $f$, i.e., the preimage of $\bR \setminus \{0\}$ by $f$.

\glsresetall

\section{Quasi-reversible queueing systems} \label{sec:quasi_reversible_queueing_systems}

In this section, we introduce quasi-reversible queueing systems,
a broad class of multi-class queueing systems
described for instance in \cite[Section~3.2]{kelly}.
In \Cref{sec:queueing_systems}, we define what we intend by 
a \emph{queueing system}: It is formally defined  
as a (discrete-state-space) \gls{CTMC}
whose state contains enough information
to count the number of customers of each class.
\Cref{sec:quasi_reversibility} gives our definition of quasi-reversibility.
To prepare the ground for our first main result
(\Cref{theo:balanced_policies_preserve_quasi_reversibility}
in \Cref{sec:balanced_arrival_rates_preserve_quasi_reversibility}),
we pay particular attention to assumptions
regarding the \gls{CTMC}'s transition diagram.
In case the reader would like to make these definitions more concrete,
we encourage them to take a look at \Cref{sec:two_canonical_examples} right away,
as it describes two canonical examples
of quasi-reversible systems, 
and applies our results to these examples.

\subsection{Queueing systems} \label{sec:queueing_systems}

In this paper, a queueing system is defined as
a triple $Q = (\cS, |\cdot|, q)$,
where $\cS$ is called the (micro-)state space,
$|\cdot| : \cS \to \bN^n$ the counting function,
and $q: \cS \times \cS \to \bR_{\ge 0}$ the transition kernel.
Let us specify each of these elements.

\subsubsection{State space, counting function, and microstate-independence} \label{sec:state_space}

Consider a queueing system
with a finite number~$n \in \bN_{> 0}$ of customer classes.
Intuitively, a class is a label that is attached to a customer
and that may carry information about this customer,
such as its service requirement.
As we will see later, we assume that the class of a customer
is fixed throughout its stay in the queueing system. 
Let $\cS$ denote the state space of the system,
assumed to be countable.

The system state is assumed to
contain enough information to recover
the number of customers of each class
via a mapping $|\cdot|$ called the \emph{counting} function:
for each $s \in \cS$,
$|s| = (|s|_1, |s|_2, \ldots, |s|_n) \in \bN^n$
gives the number of customers of each class
when the system is in state~$s$.
We let $|\cS| = \{|s|; s \in \cS\}$
denote the image of $\cS$ by~$|\cdot|$,
and we assume that $|\cS|$ is a Ferrers set.
Also, for each $x \in \bN^n$
we let $\cS_x = \{s \in \cS; |s| = x\}$
(which is empty if $x \notin |\cS|$ by definition). 
An element of $\cS$ is called a \emph{microstate}
and an element of $|\cS|$ a \emph{macrostate}\footnote{%
	The state spaces $\cS$ and $|\cS|$ may not be disjoint in general
	(for instance, if $|\cdot|$ is the identify function).
	In this case, the formal distinction between microstates and macrostates
	is mostly intuitive and related to the counting function:
	a microstate is an input of this function, while a macrostate is an output.
	As we will see in \Cref{sec:kernel,sec:reversibility},
	the stochastic process describing the evolution of the microstate over time
	satisfies the Markov property,
	while the stochastic process describing the evolution of the macrostate does not in general.%
}.
We assume that there exists exactly one state in~$\cS$,
called the \emph{empty state} and denoted by $\varnothing$,
such that $| \varnothing | = 0$. 

To illustrate this definition,
we mention the case of a multi-class queueing system
(again, with a set $\llbracket 1, n \rrbracket$ of customer classes),
such as the \gls{OI} queue of \Cref{sec:oi},
where the service rate
depends on the arrival order of customers.
The microstate can be defined as
the sequence $s = s_1 s_2 \cdots s_\ell \in \cS \subseteq \llbracket 1, n \rrbracket^*$,
where $\ell \in \bN$ is the number of customers in the queue and,
if $\ell \ge 1$, $s_p$ is the class of the $p$-th oldest customer,
for each $p \in \llbracket 1, \ell \rrbracket$
(in particular, $s_1$ is the class of the oldest customer).
Then the counting function is given by
$|s|_i = \sum_{p = 1}^\ell \mathbf{1}_{s_p = i}$,
for each $i \in \llbracket 1, n \rrbracket$.

We will often consider functions $f: \cS \to \bR$
that are \emph{microstate-independent} in the sense that
$f(s) = f(t)$ for each $s, t \in \cS$ such that $|s| = |t|$.
This is equivalent to
the existence of a function~$\tilde{f}: |\cS| \to \bR$
such that $f = \tilde{f} \circ | \cdot |$,
that is, $f(s) = \tilde{f}(|s|)$ for each $s \in \cS$.
With a slight abuse of notation,
we will neglect the distinction between $f$ and $\tilde{f}$
and see~$f$ as a function defined on both~$\cS$ and~$|\cS|$,
possibly writing $f(x)$ for $f(s)$ when $x = |s|$,
depending on what is convenient from a notation point of view;
whether the input is a microstate or a macrostate
will be clear from the notation.
In the same spirit, given a function $f: |\cS| \to \bR$,
we will often write $f(s)$ for $f(x)$ when $x = |s|$.

\subsubsection{Transition kernel} \label{sec:kernel}

We assume that the evolution of the microstate over time
defines a \gls{CTMC}
with transition kernel $q: \cS \times \cS \to \bR_{\ge 0}$.
This \gls{CTMC} will be called the \emph{microstate chain}
and will often be identified with the kernel~$q$
or the triple~$Q = (\cS, |\cdot|, q)$.
For each $s, t \in \cS$,
$q(s, t)$ gives the rate at which the system
transitions to state~$t$ when it is in state~$s$.
In particular, $q(s, s)$ gives the transition rate
from a state~$s \in \cS$ to itself
(so that $q$ is not an infinitesimal generator).
We assume that transitions out of each state~$s \in \cS$
can be partitioned as follows:
\begin{itemize}
	\item \emph{Arrivals} of class~$i$, that is, 
	transitions to a state $t \in \cS_{|s| + e_i}$,
	for each $i \in \un$;
	\item \emph{Departures} of class~$i$, i.e., 
	transitions to a state $t \in \cS_{|s| - e_i}$,
	for each $i \in \un$;
	\item \emph{Internal transitions}, that is, 
	transitions to a state~$t \in \cS_{|s|}$.
\end{itemize}
In particular, we assume that $q(s, t) = 0$ if
$t \notin \cS_{|s|} \cup \bigcup_{i = 1}^n (\cS_{|s| + e_i} \cup \cS_{|s| - e_i})$.
Intuitively, customers arrive and leave one at a time,
and a customer's class is constant over time. 

The arrival (resp.\ departure) transitions coincide with customer arrivals to (resp.\ departures from) the queueing system. For example, if the arrival rate of class~$i$ customers is constant equal to~$\nu_i$, the model will have arrival transitions with rate $q(s, t) = \nu_i$ for states $s, t \in \cS$ with $|t| = |s| + e_i$, whenever a transition from~$s$ to~$t$ is feasible. Internal transitions can be used to capture many transitions that, while modifying the state (in the sense of the state of a Markov chain), do not modify the number of customers of each class. For example, in \Cref{sec:multi_class_whittle}, we consider a queueing network where internal transitions correspond to movements of customers from one queue to another. Since internal transitions play an important role in the remainder, we urge the reader to read \Cref{sec:multi_class_whittle} now if they do not understand this concept.

To avoid pathological cases
(e.g., to guarantee that the empty state is reachable) and simplify the definitions in \Cref{sec:definition}, we make the following technical assumption.
\begin{assumption} \label{ass:unichain}
	There exists $\cS_\rec \subseteq \cS$ such that:
	\begin{enumerate}
		\item \label{ass:unichain-1}
		For each $s \in \cS_\rec$,
		there is a sequence $\varnothing = t_1, t_2, \ldots, t_\ell = s$
		of elements of $\cS$, with $\ell \ge 1$,
		such that $|t_k| \le |t_{k+1}|$ and $q(t_k, t_{k+1}) > 0$
		for each $k \in \llbracket 1,\ell-1 \rrbracket$.
		\item \label{ass:unichain-2}
		For each $s \in \cS \setminus \cS_\rec$,
		there is no sequence $\varnothing = t_1, t_2, \ldots, t_\ell = s$
		in~$\cS$, with $\ell \ge 1$,
		such that $q(t_k, t_{k+1}) > 0$
		for each $k \in \llbracket 1,\ell-1 \rrbracket$.
		\item \label{ass:unichain-3}
		For each $s \in \cS$,
		there is a sequence $s = t_1, t_2, \ldots, t_\ell = \varnothing$
		in~$\cS$, with $\ell \ge 1$,
		such that $|t_k| \ge |t_{k+1}|$ and $q(t_k, t_{k+1}) > 0$
		for each $k \in \llbracket 1,\ell-1 \rrbracket$.
		\item \label{ass:unichain-4}
		The set $|\cS_\rec| = \{|s|; s \in \cS_\rec\}$ is a Ferrers set.
	\end{enumerate}
\end{assumption}
\Cref{ass:unichain} implies in particular that~$Q$ is \emph{unichain}:
the microstate
chain admits $\cS_\rec$ as a single recurrent class, 
plus a possibly empty set $\cS \setminus \cS_\rec$ of transient states.
More specifically, under \Cref{ass:unichain},
the empty state~$\varnothing$ is recurrent,
and other states in $\cS_\rec$ communicate with~$\varnothing$
via paths that are monotonic with respect to~$| \cdot |$.
We will see in \Cref{sec:two_canonical_examples} that,
despite these restrictions,
this structure is sufficiently versatile
to capture broad families of systems,
such as Whittle networks and \gls{OI} queues,
equipped with sophisticated routing mechanisms,
such as those described by the method of stages~\cite{K76,B76}.

Unless stated otherwise, a queueing system
will be understood as a triple $Q = (\cS, |\cdot|, q)$ 
whose associated microstate chain satisfies
\Cref{ass:unichain} and is in particular unichain.

\subsection{Stationary measures and quasi-reversibility} \label{sec:quasi_reversibility}

\subsubsection{General definition} \label{sec:quasi_reversibility_definition}

Consider a queueing system $Q = (\cS, |\cdot|, q)$
as defined in \Cref{sec:queueing_systems}.
Due to \Cref{ass:unichain},
the microstate chain has a stationary measure~$\Pi: \cS \to \bR_{\ge 0}$ that is unique up to a positive multiplicative constant, 
and characterized by the balance equations 
\begin{align*}
	\Pi(s) \sum_{t \in \cS} q(s, t)
	&= \sum_{t \in \cS} \Pi(t) q(t, s),
	\quad \text{for each } s \in \cS,
\end{align*}
along with the normalizing condition $\sum_{s \in \cS} \Pi(s) = 1$.
As we assumed $q(s, t) = 0$
for each $s, t \in \cS$ with
$t \notin \cS_{|s|} \cup \bigcup_{i = 1}^n (\cS_{|s| + e_i} \cup \cS_{|s| - e_i})$,
these balance equations rewrite as follows:
\begin{multline}
	\label{eq:balance-equations}
	\Pi(s) \left(
	\sum_{i = 1}^n \sum_{t \in \cS_{|s| + e_i}} q(s, t)
	+ \sum_{i = 1}^n \sum_{t \in \cS_{|s| - e_i}} q(s, t)
	+ \sum_{t \in \cS_{|s|}} q(s, t)
	\right) \\
	= \left(
	\sum_{i = 1}^n \sum_{t \in \cS_{|s| + e_i}} \Pi(t) q(t, s)
	+ \sum_{i = 1}^n \sum_{t \in \cS_{|s| - e_i}} \Pi(t) q(t, s)
	+ \sum_{t \in \cS_{|s|}} \Pi(t) q(t, s)
	\right),
	\quad \text{for each } s \in \cS.
\end{multline}

\begin{definition}\label{def:quasirev}
	We say that a queueing system $Q = (\cS, |\cdot|, q)$ is \textit{quasi-reversible}
	if its stationary measures~$\Pi$
	also satisfy the following set of equations:
	\begin{align}
		\label{eq:partial-balance-1}
		\Pi(s) \sum_{t \in \cS_{|s| + e_i}} q(s, t)
		&= \sum_{t \in \cS_{|s| + e_i}} \Pi(t) q(t, s),
		\quad \text{for each } s \in \cS  \text{ and } i \in \un.
	\end{align}
\end{definition}

Equation~\eqref{eq:partial-balance-1} in \Cref{def:quasirev}
states that the probability flow out of state~$s$
due to a class-$i$ arrival
is equal to the probability flow into state~$s$
due to a class-$i$ departure,
for each $s \in \cS$ and $i \in \un$,
which is equivalent to \cite[Equation~(3.11)]{kelly}.
Subtracting~\eqref{eq:partial-balance-1}
from~\eqref{eq:balance-equations} yields as a by-product
\begin{multline} \label{eq:partial-balance-2}
	\Pi(s) \left(
	\sum_{i = 1}^n \sum_{t \in \cS_{|s| - e_i}} q(s, t)
	+ \sum_{t \in \cS_{|s|}} q(s, t)
	\right) \\
	= \sum_{i = 1}^n \sum_{t \in \cS_{|s| - e_i}} \Pi(t) q(t, s)
	+ \sum_{t \in \cS_{|s|}} \Pi(t) q(t, s),
	\quad \text{for each } s \in \cS.
\end{multline}
In fact, satisfying
\eqref{eq:balance-equations}--\eqref{eq:partial-balance-1}
is equivalent to satisfying
\eqref{eq:partial-balance-1}--\eqref{eq:partial-balance-2},
hence quasi-reversibility is stronger than stationarity.
Also, observe that quasi-reversibility
does not require positive recurrence,
as it is stated as a property of the stationary measures. 
Lastly note that, while Assumption \ref{ass:unichain} is not needed to properly define quasi-reversibility (in fact, even the unichain property is not necessary), this (slightly) restrictive assumption will be needed to 
define balanced access controls, and to show that quasi-reversibility is preserved by such controls (\Cref{theo:balanced_policies_preserve_quasi_reversibility}). For simplicity, we thus decided to impose Assumption~\ref{ass:unichain} from the start. A convenient consequence is that the stationary measure is unique up to a positive multiplicative constant. As we show in \Cref{sec:two_canonical_examples}, Assumption~\ref{ass:unichain} is satisfied by many classical queueing systems.

\subsubsection{Other definitions of quasi-reversibility} \label{sec:quasi_reversibility_alternative}

For clarity, let us compare \Cref{def:quasirev}
with two other classical definitions of quasi-reversibility,
namely, \cite[Section~3.2]{kelly} and \cite[Definition~1]{C11}.
All three definitions have many consequences in common,
but they also have differences:
\begin{enumerate}[(i)]
	\item \label{diff1}
	\Cref{def:quasirev} is more general than
	\cite[Section~3.2]{kelly} and \cite[Definition~1]{C11}
	in that it does not require
	the arrival rate of each class to be independent of the current state,
	i.e., we do not impose the existence, for each $i \in \un$,
	of a $\lambda_i > 0$ such that
	\begin{align}
		\label{eq:constant-arrival}
		\sum_{t \in \cS_{|s| + e_i}} q(s, t) = \lambda_i \indicator{|s| + e_i \in \cX},
		\quad \text{for each } s \in \cS.
	\end{align}
	\item \label{diff2}
	\Cref{def:quasirev} is more restrictive than \cite[Definition~1]{C11}
	in that it restricts the sets of states towards which
	each type of transitions (arrival, departure, and internal) can lead.
	This restriction imposed by \Cref{def:quasirev}
	is necessary to prove~\Cref{theo:balanced_policies_preserve_quasi_reversibility}.
	It rules out the negative customers considered in \cite{C11},
	but is still general enough to capture complex queueing systems
	such as those of \Cref{sec:two_canonical_examples}.
\end{enumerate}
All in all,
both \Cref{def:quasirev} and \cite[Definition~1]{C11}
are strict extensions of \cite[Section~3.2]{kelly},
but neither definition can be seen as a special case of the other.
Point~\ref{diff1} above has the following consequences:
\begin{itemize}
	\item Under \Cref{def:quasirev},
	quasi-reversibility cannot be interpreted
	as a condition of independence between
	the system state and the arrival and departure processes,
	as it was the case in \cite[Section~3.2]{kelly} and \cite[Definition~2]{C11}.
	However, other consequences of quasi-reversibility continue to hold;
	in particular, a network of quasi-reversible queueing systems is still quasi-reversible
	(see \cite[Theorem~3.7(iv)]{kelly} and \Cref{sec:multi_class_whittle}).
	\item Relaxing assumption~\eqref{eq:constant-arrival}
	allows \Cref{def:quasirev} to capture
	queueing systems that are (truly) reversible,
	even when the arrival rates are state dependent.
	On this question,
	see \Cref{sec:reversibility} below, as well as
	the discussion in \cite[Equation~(3.11)]{kelly}.
\end{itemize}

\subsubsection{Quasi-reversibility and (true) reversibility} \label{sec:reversibility}

We make a small digression to clarify the relationship between quasi-reversibility and reversibility. The first-time reader can skip this section if necessary, and possibly come back to it after reading the equivalence result of \Cref{theo:equivalence_whittle_oi}, which is a special case.

\paragraph*{Reversibility as a special case of quasi-reversibility}

Consider a queueing system~$Q^\star = (\cS^\star, |\cdot|^\star, q^\star)$ where
$| \cdot |^\star = \textrm{Id}$ is the identity function:
microstates and macrostates are identical,
$\cS^\star \subseteq \bN^n$,
and $\cS^\star_{|x|} = \{x\}$ for each $x \in \cS^\star$.
Then the definition of quasi-reversibility (\Cref{def:quasirev})
simplifies as follows:
\begin{align*}
	\Pi^\star(x) q^\star(x, x + e_i) = \Pi^\star(x + e_i) q^\star(x + e_i, x),
	\quad \text{for each } x \in \cS^\star \text{ and } i \in \un,
\end{align*}
where $\Pi^\star$ is the stationary distribution of $Q^\star$.
This is exactly the definition of reversibility.
In this way, under the assumption that
$| \cdot |^\star = \textrm{Id}$,
quasi-reversibility is equivalent to reversibility.

\paragraph*{Reversibility as an aggregation of quasi-reversibility}

Now consider an arbitrary quasi-reversible queueing system~$Q = (\cS, |\cdot|, q)$.
To simplify the discussion, let us assume that $\cS = \cS_\rec$,
so that the stationary distribution~$\Pi$ of the microstate chain of~$Q$ is positive on $\cS$.
In that case, quasi-reversibility of the microstate chain
implies reversibility of the macrostate process,
even if this process does not have the Markov property in general.  
To see this, let $x \in |\cS|$ and $i \in \un$ such that $x + e_i \in |\cS|$.
Summing~\eqref{eq:partial-balance-1} over all $s \in \cS_x$ yields
\begin{align*}
	\sum_{s \in \cS_x} \Pi(s) \sum_{t \in \cS_{x + e_i}}  q(s, t)
	&= \sum_{s \in \cS_x} \sum_{t \in \cS_{x + e_i}} \Pi(t) q(t, s).
\end{align*}
Given an arbitrary distribution~$\Pi^\star$ with support~$|\cS|$,
we can rewrite this equation as
\begin{align} \label{eq:aggregation}
	\Pi^\star(x)
	\underbrace{
		\sum_{s \in \cS_x} \sum_{t \in \cS_{x + e_i}}
		\frac{\Pi(s)}{\Pi^\star(x)} q(s, t)
	}_{q^\star(x, x + e_i)}
	&= \Pi^\star(x + e_i)
	\underbrace{
		\sum_{s \in \cS_x} \sum_{t \in \cS_{x + e_i}}
		\frac{\Pi(t)}{\Pi^\star(x + e_i)} q(t, s)
	}_{q^\star(x + e_i, x)}.
\end{align}
Using~\eqref{eq:aggregation},
there are two ways of interpreting $q^\star$ and $\Pi^\star$:
\begin{itemize}
	\item If we abstract ourselves away from the original queueing system~$Q$,
	we can simply see $Q^\star = (|\cS|, \textrm{Id}, q^\star)$
	as a reversible queueing system.
	In particular, $q^\star$ is the transition kernel
	of a reversible \gls{CTMC} with stationary distribution~$\Pi^\star$.
	\item Focusing on the distribution~$\Pi^\star$
	given by $\Pi^\star(x) = \sum_{s \in \cS_x} \Pi(s)$
	for each $x \in |\cS|$,
	\eqref{eq:aggregation} shows that,
	even if the macrostate process does not have the Markov property in general,
	this process is reversible,
	and under stationarity it has the same transition rates
	as a reversible \gls{CTMC} with transition kernel~$q^\star$.
\end{itemize}

\section{Balanced arrival control} \label{sec:balanced_arrival_rates}

In this section, we consider \emph{balanced} arrival control, a way of modifying the arrival rates in a quasi-reversible queueing system that preserves the quasi-reversibility property. Balanced arrival control was already considered in the special case of Whittle networks, see \cite{serfozo,BJP04,J10}. We recall the definition in \Cref{sec:definition} and provide examples in \Cref{sec:examples_of_balanced_arrival_rates}. Our main contribution, in \Cref{sec:balanced_arrival_rates_preserve_quasi_reversibility}, consists of showing that balanced arrival control is also compatible with quasi-reversibility, in the sense that supplementing a quasi-reversible queueing system with a balanced arrival control again yields a quasi-reversible queueing system. Lastly, \Cref{sec:gradient_of_parameterized_systems} shows a consequence of this result that is useful for optimizing arrival rates via gradient descent.

\subsection{Definition} \label{sec:definition}

Consider a quasi-reversible queueing system
$Q = (\cS, |\cdot|, q)$.
An arrival-control policy, or policy for short,
is a function $\gamma: \cS \to \bR_{\ge 0}^n$.
The queueing system under policy~$\gamma$ is
$Q_\gamma = (\cS, |\cdot|, q_\gamma)$,
where $q_\gamma: \cS \times \cS \to \bR_{\ge 0}$ is given as follows:
for each $s, t \in \cS$,
\begin{align} \label{eq:tilde-q}
	q_\gamma(s, t) &= \begin{cases}
		q(s, t) \gamma_i(s)
		&\text{if $t \in \cS_{|s| + e_i}$ for some $i \in \un$}, \\
		q(s, t)
		&\text{otherwise}.
	\end{cases}
\end{align}
If $\gamma_i(s) \in [0, 1]$
for each $s \in \cS$ and $i \in \llbracket 1, n \rrbracket$,
then $\gamma$ can be interpreted as an admission control policy,
i.e., $Q_\gamma$ is a variant of $Q$ where some incoming customers are rejected.
If $\sum_{i = 1}^n \gamma_i(s) = 1$
for each $s \in \cS$,
then $\gamma$ can be interpreted as a load-balancing policy.
Examples of such policies in the context of Whittle networks without internal transitions were considered in \cite{BJP04,J10}.

Let us briefly discuss the assumptions of \Cref{sec:queueing_systems}.
$Q_\gamma$ again satisfies $q_\gamma(s, t) = 0$
if $t \notin \cS_{|s|} \cup
\bigcup_{i = 1}^n (\cS_{|s| + e_i} \cup \cS_{|s| - e_i})$,
it is unichain (which follows from~\eqref{eq:tilde-q}
and the fact that $Q$ satisfies \Cref{ass:unichain}), and it admits $\varnothing$ as a recurrent state. 
Therefore, $Q_\gamma$ has a stationary measure $\Pi_\gamma$ that is unique up to a positive multiplicative constant,
and this measure satisfies $\Pi_\gamma(\varnothing) > 0$.
However, in general,
$Q_\gamma$ does not satisfy \Cref{ass:unichain} and is not quasi-reversible.

This motivates us to introduce \emph{balanced} policies in \Cref{def:balanced_policies} below. As we will see in \Cref{sec:balanced_arrival_rates_preserve_quasi_reversibility}, if the policy $\gamma$ is balanced, then $Q_\gamma$ does satisfy \Cref{ass:unichain} and is quasi-reversible.
This result will be exploited in the remainder of the paper to provide efficient ways of searching balanced policies that optimize certain performance criteria.

\begin{definition}[Balanced policy] \label{def:balanced_policies}
	Consider a queueing system $Q = (\cS, |\cdot|, q)$.
	A policy $\gamma: \cS \to [0, 1]^n$ is said to be \emph{balanced}
	if it satisfies the following two assumptions:
	\begin{enumerate}
		\item \label{def:balanced_policiesd-1} \emph{Microstate-independence:}
		$\gamma(s) = \gamma(t)$ 
		for each $s, t \in \cS$ such that $|s|=|t|$.
		\item \label{def:balanced_policiesd-2} \emph{Balance condition:}
		There exists a function
		$\Gamma: |\cS| \to \bR_{\ge 0}$
		such that $\Gamma(0) = 1$, and
		\begin{align} \label{eq:balance_function}
			\Gamma(x) \gamma_i(x) &= \Gamma(x + e_i),
			\quad \text{for each } x \in |\cS| \text{ and } i \in \un
			\text{ so that } x + e_i \in |\cS|.
		\end{align}
	\end{enumerate}
	In this case,
	we say that the policy $\gamma$ is $\Gamma$-balanced,
	or equivalently that $\Gamma$ is the \emph{balance function} of the policy~$\gamma$, and we write $\gamma = \gammabac$.
\end{definition}

\begin{figure}[ht]
	\centering
	\begin{tikzpicture}[xscale=2.5, yscale=2]
		\foreach \i in {0, ..., 4} {
			\foreach \j in {0, ..., 2} {
				\node (\i\j) at (\i, \j)
				{\small $\Gamma(\i, \j)$};
			}
		}
		
		\foreach \i\k in {0/1, 1/2, 2/3, 3/4} {
			\foreach \j in {0, 1, 2} {
				\draw[->, thick] (\i\j)
				-- node[midway, below]
				{\small $\gamma_1(\i, \j)$}
				(\k\j);
			}
		}
		
		\foreach \i in {0, ...,4} {
			\foreach \j/\k in {0/1, 1/2} {
				\draw[->, thick] (\i\j)
				-- node[midway, rotate=90, above]
				{\small $\gamma_2(\i, \j)$}
				(\i\k);
			}
		}
	\end{tikzpicture}
	\caption{%
		An intuitive description of the balance condition
		in a toy queueing system of dimension $n = 2$
		with macrostate space $\cX = \{x \in \bN^2\,;\, x_1 \le 4, x_2 \le 2\}$.
		For each $x \in \cX$,
		$\Gamma(x)$ is the product of admission probabilities
		along an arbitrary increasing path going from the origin to~$x$.
		The balance condition guarantees that
		$\Gamma(x)$ is well-defined in the sense that
		the product is independent of the path.
	}
	\label{fig:balance_condition}
\end{figure}

Before giving examples, let us discuss several points in \Cref{def:balanced_policies}. 
\Cref{def:balanced_policies} consists of two conditions. Microstate-independence is tied to the notions of \emph{classes}  and \emph{counting function} introduced in \Cref{sec:state_space}. In general, there may be several ways of defining classes and the counting function in a queueing system, and these may lead to different definitions of balance. If a policy~$\gamma$ is microstate-independent, it can be seen either as a (microstate-independent) function on~$\cS$, or as a function on $\cX = |\cS|$. In the remainder, we will often take the later point of view, in which case we will ignore~$\cS$ and say that $\gamma$ is a policy on~$\cX$.
The balance condition, which is the second part of \Cref{def:balanced_policies}, is illustrated in \Cref{fig:balance_condition}.
Intuitively, for any $x$, Equation~\eqref{eq:balance_function} states that
$\Gamma(x)$ is the product of the policy's transition rates
over any increasing path from the origin to~$x$,
so that this product is independent of the chosen path.
In fact, one can verify that there exists a balance function~$\Gamma$ that satisfies~\eqref{eq:balance_function} if and only if
\begin{align} \label{eq:balance_condition}
	\gamma_i(x) \gamma_j(x + e_i)
	= \gamma_i(x + e_j) \gamma_j(x),
	\quad \left\{
	\begin{aligned}
		&\text{for each } x \in |\cS|, i, j \in \llbracket 1, n \rrbracket
		\text{ such that} \\
		&x + e_i, x + e_j \in \supp(\Gamma),
		x + e_i + e_j \in |\cS|.
	\end{aligned}
	\right.
\end{align}
Observe that~\eqref{eq:balance_function}
would still be satisfied if we replaced $\Gamma$ with
$\frac1Z \Gamma$ for any $Z > 0$.
The original queueing system~$Q$
is retrieved by taking $\Gamma(x) = 1$ for each $x \in |\cS|$.

Given a Ferrers set $\cX \subseteq \bN^n$
interpreted as a queueing system's macrostate space,
we will see that there is a bijection between
the set of balanced policies on~$\cX$
and the set of balance functions on~$\cX$.
The set of balance functions on~$\cX$ is given by
\begin{align}
	\cB_\cX &= \left\{
	\Gamma: \cX \to \bR_{\ge 0}\,;\,
	\Gamma(0) = 1
	\text{ and $\supp(\Gamma)$ is a Ferrers set}
	\right\}.\label{eq:defBX}
\end{align}
Indeed, if $\Gamma$ satisfies the balance condition~\eqref{eq:balance_function} for a certain policy~$\gamma$,
then $\Gamma(x) = 0$ for some $x \in \cX$
implies that $\Gamma(y) = 0$ for each $y \in \cX$ such that $x \le y$,
which means that $\supp(\Gamma)$ is a Ferrers set.
Conversely, for each $\Gamma \in \cB_\cX$, we can construct a policy~$\gamma$ on~$\cX$ that satisfies~\eqref{eq:balance_function} by letting
\begin{align} \label{eq:balance_function_bis}
	\gamma_i(x) = \frac{\Gamma(x + e_i)}{\Gamma(x)},
	\quad \text{for each } x \in \supp(\Gamma) \text{ and } i \in \llbracket 1, n \rrbracket \text{ such that } x + e_i \in \cX,
\end{align}
and $\gamma(x) = 0$ for each $x \in \cX \setminus \supp(\Gamma)$.
Now, if we let $\gamma$ and $\gamma'$ denote two $\Gamma$-balanced policies, \eqref{eq:balance_function_bis} implies that $\gamma_i(x) = \gamma'_i(x)$ for each $x \in \supp(\Gamma)$ and $i \in \llbracket 1, n \rrbracket$ such that $x + e_i \in \cX$.
In particular, the restrictions of $Q_\gamma$ and $Q_{\gamma'}$ to $\supp(\Gamma)$ are identical, i.e., they have the same transition rates.
Since one can also verify that the recurrent sets of $Q_\gamma$ and $Q_{\gamma'}$ are included into $\supp(\Gamma)$, it follows that $Q_\gamma$ and $Q_{\gamma'}$ have the same stationary distribution. Therefore, by assimilating two policies if they lead to the same stationary distribution if necessary, it follows that~\eqref{eq:balance_function} defines a bijection between the set $\cB_\cX$ of balance functions and the set of balanced policies.
In general, when we write $\gamma = \gammabac$,
we mean that $\gamma$ satisfies~\eqref{eq:balance_function_bis}
and takes an arbitrary value on $\cX \setminus \supp(\Gamma)$,
say $\gamma(x) = 0$ for each $x \in \cX \setminus \supp(\Gamma)$.

Balanced arrival control has been studied extensively
in the context of Whittle networks without internal transitions~\cite{BJP04,J10}.
The key novelty of \Cref{def:balanced_policies}
is to apply balanced policies
in the much more general context of quasi-reversible queueing systems,
via the definitions of the counting function~$| \cdot |$
and of microstate-independence.

\subsection{Examples} \label{sec:examples_of_balanced_arrival_rates}

Consider a queueing system~$Q = (\cS, |\cdot|, q)$.
To make the definition of balanced policies more concrete,
let us focus on the special case of admission control,
where $\gamma_i(x) \in [0, 1]$ is the admission probability
of class-$i$ customers
when the macrostate is~$x$,
for each $x \in |\cS|$ and $i \in \llbracket 1, n \rrbracket$.
Examples of balanced policies
are shown in \Cref{tab:examples}.

\begin{table}[ht]
	\scalebox{0.88}{
		\renewcommand{\arraystretch}{2}
		\begin{tabular}{|c|c|c|c|}
			\hline
			Name
			& \makecell{Admission probability \\ (with $\gamma_i(x) = 0$ if \\ $\Gamma(x) = 0$ or $x + e_i \notin \cX$)}
			& \makecell{Balance function \\ (with $\Gamma(x) = 0$ if $x \notin \cX$)}
			& Parameter
			\\ \hline
			Static
			& \makecell{$\displaystyle \gamma_i(x) = \alpha_i$}
			& \makecell{$\displaystyle \Gamma(x) = \prod_{i = 1}^n {\alpha_i}^{x_i}$}
			& \makecell{Reals $\alpha_i \in [0, 1]$}
			\\ \hline
			Decentralized
			& \makecell{$\displaystyle \gamma_i(x) = \psi_i(x_i)$}
			& \makecell{$\displaystyle \Gamma(x) = \prod_{i = 1}^n \prod_{\ell = 0}^{x_i - 1} \psi_i(\ell)$}
			& \makecell{Functions $\psi_i: \overline{\cX_i} \to [0, 1]$, \\ with $\overline{\cX_i} = \{x_i; x \in \cX\}$}
			\\ \hline
			Size-based
			& \makecell{$\displaystyle \gamma_i(x) = \psi(|x|_1)$}
			& \makecell{$\displaystyle \Gamma(x) = \prod_{\ell = 0}^{|x|_1 - 1} \psi(\ell)$}
			& \makecell{Function $\psi: \overline{\cX} \to [0, 1]$, \\ with $\overline{\cX} = \{\lVert x \rVert_1; x \in \cX\}$}
			\\ \hline
			Deterministic
			& \makecell{$\displaystyle \gamma_i(x) = \indicator{x + e_i \in \cA}$}
			& \makecell{$\displaystyle \Gamma(x) = \indicator{x \in \cA}$}
			& \makecell{A Ferrers set~$\cA \subseteq \cX$ called \\ the policy's mask}
			\\ \hline
			\makecell{Cum-prod (for \\ cumulative-product)}
			& \makecell{$\displaystyle \gamma_i(x) = \prod_{\substack{y \in \cX: y \le x + e_i, \\ y_i = x_i + 1}}^{\phantom{1}} \psi(y)$}
			& \makecell{$\displaystyle \Gamma(x) = \prod_{\substack{y \in \cX: y \le x}} \psi(y)$}
			& \makecell{Function $\psi: \cX \to [0, 1]$}
			\\ \hline
		\end{tabular}
	}
	\caption{Examples of balanced policies over an arbitrary Ferrers set $\cX \subseteq \bN^n$. Here, $x \in \cX$ and $i \in \un$.
	}
	\label{tab:examples}
\end{table}

The first row shows a simple but common family
of balanced policies: \emph{static policies},
where the admission decision
does not depend on the current state.
The next two rows show that balanced policies
include two important families of admission policies:
\emph{decentralized policies},
where the decision to admit a customer of a given class
only depends on the number of customers \emph{of this class} in the system,
and \emph{size-based policies},
where the decision to admit a customer
only depends on the \emph{total} number of customers in the system
(irrespective of their classes).
It follows that,
if we know \textit{a priori} that the optimal policy
belongs to one of these two classes,
then we can search the optimal policy among balanced policies.
\emph{Deterministic policies},
on the fourth row of \Cref{tab:examples},
will play a fundamental role
in \Cref{sec:admission_control},
as we will see that the best balanced policy
(in a sense that we will define in \Cref{sec:optimal_balanced_policies})
can be chosen of this form,
at least when the state space is finite.
Observe that deterministic balanced policies
form a subset of deterministic possibly-imbalanced policies.
Indeed, under a deterministic balanced policy,
the only states~$x \in \cX$ and classes $i, j \in \un$
such that $\gamma_i(x) \neq \gamma_j(x)$
are such that $x + e_i \in \cA$ and $x + e_j \notin \cA$
or $x + e_i \notin \cA$ and $x + e_j \in \cA$,
where $\cA \subseteq |\cS|$ is called the policy's mask.
\emph{Cum-prod policies},
on the fifth row of \Cref{tab:examples},
is a family of balanced policies
that includes deterministic policies
but admit a continuous parameterization.

The simple examples of \Cref{tab:examples} can be generalized in various ways depending on the problem at hand. For instance, decentralized policies can be augmented by factors involving multiple variables, for instance,
\begin{align*}
	\Pi(x) = \left( \prod_{i = 1}^n \prod_{\ell = 0}^{x_i - 1} \psi_i(\ell) \right)
	\left( \prod_{\{i, j\} \subseteq \un}
	\prod_{\ell = 0}^{x_i + x_j - 1}
	\psi_{i, j}(\ell) \right),
	\quad x \in \cX,
\end{align*}
where $\psi_i$ is as defined for decentralized policies,
and $\psi_{i, j}: \overline{\cX}_{i, j} \to [0, 1]$
with $\overline{\cX}_{i, j} = \{x_i + x_j; x \in \cX\}$
for each $\{i, j\} \subseteq \un$.
One may restrict the latter product
to only some pairs of classes,
using for instance information on the structure of the system.

One may also consider balance functions~$\Gamma$ of the form $\Gamma(x) = f(x)$ for each $x \in \cX$, where $f: [0, +\infty)^n \to [0, +\infty)$ is a differentiable function with non-positive first partial derivatives (in particular, $f$ is componentwise non-increasing). One broad class of functions that satisfy this requirement are \glspl{CCDF} of multivariate distributions, such as multivariate Gaussian distributions.

\subsection{Balanced arrival rates preserve quasi-reversibility} \label{sec:balanced_arrival_rates_preserve_quasi_reversibility}

\Cref{theo:balanced_policies_preserve_quasi_reversibility} hereafter shows that
supplementing a quasi-reversible queueing system with a
($\Gamma$-)balanced policy
preserves quasi-reversibility
and boils down to modulating 
the stationary distribution by $\Gamma$,
in the sense of~\eqref{eq:piquasi}.
\Cref{theo:balanced_policies_preserve_quasi_reversibility} was inspired by
two observations:
(i) \cite[Proposition~1.23]{serfozo} and \cite[Equation~(10)]{BJP04},
two special cases of \Cref{theo:balanced_policies_preserve_quasi_reversibility}
that were stated in a family of reversible queueing systems,
namely Whittle networks without internal transitions (see \Cref{sec:multi_class_whittle}),
and (ii) the reversibility of the macrostate process
of a quasi-reversisble queueing system
discussed in \Cref{sec:reversibility}.

\begin{theorem} \label{theo:balanced_policies_preserve_quasi_reversibility}
	Consider a quasi-reversible queueing system
	$Q = (\cS, |\cdot|, q)$,
	and let $\Pi$ denote one of its stationary measures.
	For each balance function~$\Gamma \in \cB_{|\cS|}$,
	the triplet
	$Q_\gammabac = (\cS, |\cdot|, q_\gammabac)$
	is a quasi-reversible queueing system,
	and its stationary measures are obtained
	by taking $\frac1Z \Pi_\gammabac$ for any $Z > 0$, where
	\begin{align} \label{eq:piquasi}
		\Pi_{\gammabac}(\bs) = \Pi(\bs) \Gamma(\bs),
		\quad \text{for each } s \in \cS.
	\end{align}
\end{theorem}

\begin{proof}
	We proceed in two steps: first we prove that $Q_\gammabac$ satisfies \Cref{ass:unichain}, and then we prove (at once) that it is quasi-reversible and that its stationary measures are proportional to \eqref{eq:piquasi}.
	
	\vspace{.2cm}
	
	\noindent \emph{Step 1: $Q_\gammabac$ satisfies \Cref{ass:unichain}.}
	Let us first show that $Q_\gammabac$ satisfies \Cref{ass:unichain}
	and that its recurrent class is $\cS_{\gammabac, \rec} = \cS_\rec \cap \cS_{\Gamma > 0}$,
	where $\cS_\rec$ is the recurrent class of $Q$
	and $\cS_{\Gamma > 0} = \{s \in \cS; \Gamma(s) > 0\}$.
	We verify each point of \Cref{ass:unichain} one after the other:
	\begin{enumerate}
		\item Let $s \in \cS_{\gammabac, \rec}$.
		Since $s \in \cS_\rec$, there is a sequence $\varnothing = t_1, t_2, \ldots, t_\ell = s$ in $\cS$
		such that $|t_k| \le |t_{k+1}|$ and $q(t_k, t_{k+1}) > 0$ for each $k \in \llbracket 1,\ell-1 \rrbracket$.
		Combining $\Gamma(s) > 0$, \eqref{eq:tilde-q}, \eqref{eq:balance_function_bis}, and the microstate-independence of $\gammabac$,
		we obtain that $q_\gammabac(t_k, t_{k+1}) > 0$ for $k \in \llbracket 1,\ell-1\rrbracket$. 
		\item Let $s \in \cS \setminus \cS_{\gammabac, \rec}$:
		either $s \notin \cS_\rec$, or $s \notin \cS_{\Gamma > 0}$, or both:
		\begin{itemize}
			\item If $s \notin \cS_\rec$, we know
			that there is no sequence $\varnothing = t_1, t_2, \ldots, t_\ell = s$ in $\cS$
			such that $q(t_k, t_{k+1}) > 0$ for each $k \in \llbracket 1,\ell-1 \rrbracket$,
			and the same follows if we replace $q$ with $q_\gammabac$ by~\eqref{eq:tilde-q}.
			\item If $s \notin \cS_{\Gamma > 0}$, we have $\Gamma(t) = 0$ for each $t \in \cS$ such that $|t| \ge |s|$ because $\cS_{\Gamma > 0}$ is a Ferrers set.
			Combining this observation with~\eqref{eq:tilde-q} and~\eqref{eq:balance_function_bis},
			we conclude that for each sequence $\varnothing = t_1, t_2, \ldots, t_\ell = s$,
			there must exit an index $k \in \llbracket 1,\ell-1 \rrbracket$ such that $q_\gammabac(t_k, t_{k+1}) = 0$.
		\end{itemize}
		\item Let $s \in \cS$.
		There is a sequence $s = t_1, t_2, \ldots, t_\ell = \varnothing$ in $\cS$
		such that $|t_k| \ge |t_{k+1}|$ and $q(t_k, t_{k+1}) > 0$ for each $k \in \llbracket 1,\ell-1 \rrbracket$.
		To conclude, it suffices to observe that
		$q_\gammabac(t_k, t_{k+1}) = q(t_k, t_{k+1})$ for each $k \in \llbracket 1,\ell-1 \rrbracket$.
		\item We have $|\cS_{\gammabac, \rec}| = |\cS_\rec| \cap \supp(\Gamma)$,
		where $\supp(\Gamma) = \{x \in |\cS|; \Gamma(x) > 0\}$.
		The set $|\cS_\rec|$ is a Ferrers set
		because $Q$ satisfies \Cref{ass:unichain}.
		The set $\supp(\Gamma)$ also is also a Ferrers set because $\Gamma \in \cB_{|\cS|}$.
		Hence, $|\cS_{\gammabac, \rec}|$ also satisfies these two properties.
	\end{enumerate}
	
	\vspace{.2cm}
	
	\noindent \emph{Step~2: Stationary distribution and quasi-reversibility.}
	Now, to show that $Q_\gammabac$ is quasi-reversible and admits~$\Pi_{\gammabac}$
	given by~\eqref{eq:piquasi} as a stationary measure, it is enough to show that $\Pi_{\gammabac}$ 
	satisfies the following partial balance equations:
	\begin{align}
		\label{eq:partial-balance-1-pi}
		\Pi_{\gammabac}(\bs)
		\sum_{t \in \cS_{|\bs| + e_i}}
		q(\bs, \bt) \gammabac_i(s)
		&= \sum_{\bt \in \cS_{|\bs| + e_i}}
		\Pi_{\gammabac}(\bt)
		q(\bt, \bs),
		\text{for each } \bs \in \cS, i \in \un,
	\end{align}
	and
	\begin{multline}
		\label{eq:partial-balance-2-pi}
		\Pi_{\gammabac}(\bs) \left(
		\sum_{i = 1}^n
		\sum_{\bt \in \cS_{|\bs| - e_i}} q(\bs, \bt)
		+ \sum_{t \in \cS_{|\bs|}} q(s, t)
		\right) \\
		= \sum_{\bt \in \cS_{|\bs| - e_i}} \Pi_{\gammabac}(\bt) q(\bt, \bs) \gammabac_i(s)
		+ \sum_{t \in \cS_{|\bs|}} \Pi_{\gammabac}(t) q(t, s),
		\quad \text{for each } s \in \cS.
	\end{multline}
	Indeed, \eqref{eq:partial-balance-1-pi}--\eqref{eq:partial-balance-2-pi}
	is the counterpart of \eqref{eq:partial-balance-1}--\eqref{eq:partial-balance-2}
	in $Q_{\gammabac}$,
	and we observed below \Cref{def:quasirev}
	that \eqref{eq:partial-balance-1}--\eqref{eq:partial-balance-2}
	was equivalent to stationarity and quasi-reversibility.
	It will be useful to observe that,
	if we see $\Gamma$ as a function on the microstate space~$\cS$,
	\eqref{eq:balance_function} can be rewritten as 
	\begin{equation}
		\label{eq:balance_functionC}
		\gammabac_i(\bs)\Gamma(\bs)= \Gamma(\bt),
		\quad \text{for each }
		\bs \in \cS, i\in \un, \text{ and } \bt \in \cS_{|\bs|+e_i}. 
	\end{equation} 
	
	Let us first prove~\eqref{eq:partial-balance-1-pi}.
	For each $s \in \cS$ and $i \in \un$,
	we have successively,
	starting from the left-hand side of~\eqref{eq:partial-balance-1-pi},
	\begin{align*}
		\Pi_{\gammabac}(s) \sum_{t \in \cS_{|s| + e_i}} q(s, t) \gammabac_i(s)
		&= \Gamma(s) \gammabac_i(s) \Pi(s) \sum_{t \in \cS_{|s| + e_i}} q(s, t),
		&& \text{by \eqref{eq:piquasi},} \\
		&= \Gamma(s) \gammabac_i(s) \sum_{t \in \cS_{|s| + e_i}} \Pi(t) q(t, s),
		&& \text{by \eqref{eq:partial-balance-1},} \\
		&= \sum_{t \in \cS_{|s| + e_i}} \Pi(t) \Gamma(t) q(t, s),
		&& \text{by \eqref{eq:balance_functionC},} \\
		&= \sum_{t \in \cS_{|s| + e_i}} \Pi_{\gammabac}(t) q(t, s),
		&& \text{by \eqref{eq:piquasi}.}
	\end{align*}
	Now let us prove~\eqref{eq:partial-balance-2-pi}.
	For each $s \in \cS$,
	we have successively,
	starting from the left-hand side of~\eqref{eq:partial-balance-2-pi},
	\begin{multline*}
		\Pi_{\gammabac}(\bs) \left(
		\sum_{i = 1}^n
		\sum_{\bt \in \cS_{|\bs| - e_i}} q(\bs, \bt)
		+ \sum_{t \in \cS_{|\bs|}} q(s, t)
		\right) \\
		\begin{aligned}
			&= \Gamma(s) \Pi(s) \left(
			\sum_{i = 1}^n
			\sum_{\bt \in \cS_{|\bs| - e_i}} q(\bs, \bt)
			+ \sum_{t \in \cS_{|\bs|}} q(s, t)
			\right),
			&& \text{by \eqref{eq:piquasi}}, \\
			&= \Gamma(s) \left(
			\sum_{i = 1}^n \sum_{t \in \cS_{|s| - e_i}} \Pi(t) q(t, s) + \sum_{t \in \cS_{|s|}} \Pi(t) q(t, s)
			\right),
			&& \text{by \eqref{eq:partial-balance-2}}, \\
			&= \sum_{i = 1}^n \sum_{t \in \cS_{|s| - e_i}}
			\Pi(t) \Gamma(t) q(t, s) \gammabac_i(t)
			+ \sum_{t \in \cS_{|s|}} \Pi(t) \Gamma(t) q(t, s),
			&& \text{by \eqref{eq:balance_functionC}}, \\
			&= \sum_{i = 1}^n \sum_{t \in \cS_{|s| - e_i}}
			\Pi_{\gammabac}(t) q(t, s) \gammabac_i(t)
			+ \sum_{t \in \cS_{|s|}} \Pi_{\gammabac}(t) q(t, s),
			&& \text{by~\eqref{eq:piquasi}}. 
		\end{aligned}
	\end{multline*}
	This completes the proof of \Cref{theo:balanced_policies_preserve_quasi_reversibility}.
\end{proof}

\subsection{Gradient of parameterized systems} \label{sec:gradient_of_parameterized_systems}

In the literature, results like
\Cref{theo:balanced_policies_preserve_quasi_reversibility}
have been applied mainly to evaluate performance,
i.e., to compute an average performance criterion that can be written as an expectation under the stationary distribution, when this expectation exists;
see for instance \cite{BJP04,J10}.
Here, we instead show a consequence of
\Cref{theo:balanced_policies_preserve_quasi_reversibility}
in case of parameterized balance functions,
as explained below.
This results is useful for sensitivity analysis
and in the context of \gls{PGRL},
as we will develop in \Cref{sec:numerics}.

Consider a quasi-reversible queueing system
$Q = (\cS, |\cdot|, q)$
and let $d \in \bN_{> 0}$.
We call \emph{parameterization of the balance function} a mapping $\theta \in \bR^d \mapsto \Gamma_\theta \in \cB_{|\cS|}$ such that:
\begin{enumerate}
	\item For each $s \in \cS$, the function
	$\theta \in \bR^d \mapsto \Gamma_\theta(s)$
	is differentiable and positive,
	so that in particular $\nabla_\theta \log \Gamma_\theta(s)$ is well defined for each $\theta \in \bR^d$.
	\item For each $\theta \in \bR^d$,
	the queueing system
	$Q_\theta \triangleq Q_{\partial \Gamma_\theta}$,
	which is quasi-reversible by \Cref{theo:balanced_policies_preserve_quasi_reversibility},
	is also positive recurrent.
\end{enumerate}
A simple example of parameterization is given by
$\Pi_\theta(s) = \prod_{i = 1}^n \sigma(\theta_i)^{|s|_i}$,
where $d = n$ and $\sigma: x \in \bR \mapsto 1 / (1 + \textrm{e}^{-x}) \in \bR$ is a sigmoid function.
One can then verify that the set of balanced policies
generated by letting $\theta$ vary in $\bR^n$
is the set of static policies mentioned in \Cref{sec:examples_of_balanced_arrival_rates}.
In general, in practice, the parameterization should be such that we can compute $\Gamma_\theta(s)$ and $\nabla_\theta \log \Gamma_\theta(s)$ efficiently, for each $\theta \in \bR^d$ and $s \in \cS_\rec$ (e.g., in closed form or by automatic differentiation).

For each $\theta \in \bR^d$,
let $\Pi_\theta \triangleq \Pi_{\partial \Gamma_\theta}$
denote the stationary distribution of $Q_\theta$,
so that
\begin{align}
	\label{eq:log_stationary_distribution}
	\log \Pi_\theta(s)
	&= \log \Pi(s) + \log \Gamma_\theta(s) - \log Z_\theta,
	\quad \text{for each } s \in \cS_\rec, \\
	\label{eq:normalizing_constant}
	Z_\theta
	&= \sum_{s \in \cS_\rec} \Pi(s) \Gamma_\theta(s),
\end{align}
where \eqref{eq:log_stationary_distribution}
follows from~\eqref{eq:piquasi},
and \eqref{eq:normalizing_constant}
from the normalizing condition
$\sum_{s \in \cS} \Pi_\theta(s) = 1$.
\Cref{theo:gradient} below
gives a closed-form expression for the gradient $\nabla_\theta \log \Pi_\theta(s)$
for each $s \in \cS_\rec$, when this gradient is properly defined.
This result can be seen as an extension of \cite[Equation~(12) in Theorem~1]{CJSS24},
and it remains valid if we replace $\bR^d$ with any open set $\Omega \subseteq \bR^d$.
The key take-away of~\eqref{eq:gradient} is that~$\nabla_\theta \log \Pi_\theta(s)$
can be written as an expectation under the distribution~$\Pi_\theta$.

\begin{corollary} \label{theo:gradient}
	Consider a quasi-reversible queueing system
	$Q = (\cS, |\cdot|, q)$
	and a parameterization $\theta \in \bR^d \mapsto \Gamma_\theta \in \cB_{|\cS|}$
	of the balance function.
	Assume that, for each $\theta \in \bR^d$,
	$Q_\theta$ is positive recurrent
	and the random vector
	$\nabla_\theta \log \Pi_\theta(S_\theta)$,
	where $S_\theta \sim \Pi_\theta$, is integrable.
	Then, for each $\theta \in \bR^d$, we have
	\begin{align} \label{eq:gradient}
		\nabla_\theta \log \Pi_\theta(s)
		&= \nabla_\theta \log \Gamma_\theta(s)
		- \bE{\nabla_\theta \log \Gamma_\theta(S_\theta)},
		\quad \text{for each } s \in \cS_\rec.
	\end{align}
\end{corollary}

\begin{proof}
	Let $\theta \in \bR^d$.
	Let us first compute $\nabla_\theta \log Z_\theta$,
	and then we will inject this result to compute $\nabla_\theta \log \Pi_\theta(s)$ for each $s \in \cS$.
	We have successively:
	\begin{align}
		\nabla_\theta \log Z_\theta
		&\overset{\text{(a)}}=
		\frac1{Z_\theta} \sum_{s \in \cS_\rec}
		\Pi(s) \Gamma_\theta(s)
		\nabla_\theta \log \Gamma_\theta(s), \notag\\
		&\overset{\text{(b)}}=
		\sum_{s \in \cS_\rec} \Pi_\theta(s)
		\nabla_\theta \log \Gamma_\theta(s), \notag\\
		\label{eq:log_normalizing}
		&\overset{\text{(c)}}=
		\bE{\nabla_\theta \log \Gamma_\theta(S_\theta)},
	\end{align}
	where (a) follows by applying the rule
	$\nabla_\theta f(\theta) = f(\theta) \nabla \log f(\theta)$
	to $f \in \{\Gamma, Z\}$
	and applying~\eqref{eq:normalizing_constant},
	(b) follows by definition of the stationary distribution~$\Pi_\theta$,
	and (c) follows by recognizing an expectation under this distribution.
	Now, for each $s \in \cS$, we have
	\begin{align*}
		\nabla_\theta \log \Pi_\theta(s)
		&\overset{\text{(d)}}=
		\nabla_\theta \log \Gamma_\theta(s)
		-\nabla_\theta \log Z_\theta
		\overset{\text{(e)}}=
		\nabla_\theta \log \Gamma_\theta(s)
		- \bE{\nabla_\theta \log \Gamma_\theta(S_\theta)},
	\end{align*}
	where (d) follows
	by differentiating~\eqref{eq:log_stationary_distribution} and (e), by injecting~\eqref{eq:log_normalizing}.
\end{proof}

\section{Two canonical examples} \label{sec:two_canonical_examples}

We consider two well-known queueing systems
that are quasi-reversible in the sense of \Cref{sec:queueing_systems}, 
and provide different interpretations
of the distinction between microstates and macrostates introduced in \Cref{sec:state_space}.
We then briefly explain how the concepts of \Cref{sec:balanced_arrival_rates}
can be applied in these two examples.
Notation that was already introduced in \Cref{sec:queueing_systems}
is reused in a consistent way.

\subsection{Order-independent queue} \label{sec:oi}

Consider a queue with~$n$ customer classes, where the queue state is described by the word 
$s = s_1s_2\cdots s_\ell \in \cS \subseteq \un^*$, where 
$\ell \in \bN$ is the number of customers in the queue 
and, if $\ell\ne 0$, $s_p$ is the class of the $p$-th oldest customer,
for each $p \in \llbracket 1,\ell \rrbracket$. 
We assume that the microstate space $\cS$ is:
\begin{itemize}
	\item Stable by prefix, in the sense that for any $s\in\cS$, $\breve s\in\cS$ for any prefix $\breve s$ of $s$;
	\item Stable by permutation, namely, if $s=s_1\cdots s_\ell\in \cS$, then any permutation $s_{\sigma(1)}\cdots s_{\sigma(\ell)}$ 
	of $s$ is also an element of $\cS$. 
\end{itemize}
To cast this model in the framework of \Cref{sec:queueing_systems}, the counting 
function $|\cdot|$ takes the form 
\[|\cdot|:\begin{cases}
	\cS &\longrightarrow \bN^{n};\\
	s=s_1\cdots s_\ell &\longmapsto |s|, \mbox{ s.t. }|s|_i=\sum_{k=1}^{\ell}\mathbf 1_{s_{k}=i},\quad i\in\llbracket 1,n \rrbracket. 
\end{cases}
\]

For each $i \in \un$,
we assume that class-$i$ customers arrive according
to a Poisson process with rate $\nu_i > 0$, and that there exists a function $\mu: |\cS| \to \bR_{\ge 0}$,
called the \emph{rate} function, such that for each microstate $s \in \cS$, $\mu(s)$ gives the total departure rate of customers in state $s$. 
We assume that the function~$\mu$ is such that
$\mu(\mathbf 0) = 0$,
$\mu(s) > 0$ for each $s \in \cS \setminus \{\varnothing\}$,
and $\mu$ is component-wise non-decreasing, 
in the sense that $\mu(|s|) \le \mu(|s| + e_i)$ for each $s \in \cS$ and $i \in \un$ so that $|s| + e_i \in |\cS|$.
For each $s = s_1s_2\cdots s_\ell \in \cS \setminus \{\varnothing\}$,
the departure rate of the customer in position~$p \in \llbracket 1,\ell\rrbracket$, of class~$s_p$, is given by
\begin{align*}
	\Delta\mu(s_1s_2\cdots s_p) := \mu(s_1s_2\cdots s_p) - \mu(s_1s_2\cdots s_{p-1}). 
\end{align*}
In particular, although the \emph{total} departure rate only depends on the macrostate,
how this rate is distributed among customers also depends on the microstate.
As a consequence, with the notation of \Cref{sec:queueing_systems}, the microstate CTMC's transition kernel is given by
\begin{align}
	\label{eq:transitionOI}
	q(s, t) &= \begin{cases}
		\nu_i
		&\text{if $t = s_1\cdots s_\ell i$, for each $i$ s.t. $s_1\cdots s_\ell i \in \cS$}, \\
		\Delta\mu(s_1s_2\cdots s_p)
		&\text{if $t = s_1\cdots s_{p-1}s_{p+1}\cdots s_\ell$, for each $p \in \llbracket 1,\ell\rrbracket$,}
	\end{cases}
\end{align}
and $q(s, t) = 0$ otherwise. More details about \gls{OI} queues
can be found in~\cite{K11}. We have the following results.

\begin{lemma}\label{lemma:OI-queueing}
	The triple $Q=(\cS,|\cdot|,q)$ is a queueing system in the sense of \Cref{sec:queueing_systems}. 
\end{lemma}
\begin{proof}
	See Appendix \ref{sec:preuveOI-queueing}.
\end{proof}

\Cref{prop:oi} below recalls that the \gls{OI} queue defined above is quasi-reversible,  implying again that, thanks to the results of \Cref{sec:balanced_arrival_rates}, we can apply balanced arrival rates to \gls{OI} queues and still obtain a quasi-reversible queueing system. While it was already known that \gls{OI} queues were quasi-reversible, we believe that the application of state-dependent arrival rates to \gls{OI} queues was only considered in \cite{C18} in the special case of redundancy systems. Besides this, to the best of our knowledge, the only state-dependent arrival rates that were considered so far in the context of \gls{OI} queues only depended on the \emph{total} number of customers in the system~\cite{B95,B96,K11}.

\begin{proposition}
	\label{prop:oi}
	The triple $Q$ is quasi-reversible in the sense of Definition \ref{def:quasirev}, and its stationary measures are given by 
	\begin{align} \label{eq:oi-pi}
		\Pi(s) &\propto
		\prod_{p = 1}^\ell \frac{\nu_{s_p}}{\mu(s_1s_2\cdots s_p)},
		\quad s_1s_2\cdots s_\ell \in \cS.
	\end{align}
\end{proposition}

\begin{proof}
	Equation~\eqref{eq:oi-pi} 
	is shown in \cite[Theorem~1]{B95}, and the quasi-reversibility property is exactly~(4) in \cite{B95}.
\end{proof}
\begin{example}\rm 
	As shown in \cite{BC17}, the redundancy model introduced in \cite{G16}
	can be seen as an \gls{OI} queue when the servers apply the \gls{FCFS} scheduling policy.
	An extension of this model with customer abandonments will be considered in the case study of \Cref{sec:numerics}.
	As shown in \cite{C22}, the \emph{stochastic matching model} introduced in \cite{MM16} is also an \gls{OI} queue when the matching policy is \gls{FCFM}. In \cite{MBM21}, the product-form result specializing \eqref{eq:oi-pi} to the FCFM matching model, is shown using a dynamic reversibility argument for an extended version of the state space $\cS$. 
	See also \cite{BMMR21} for the corresponding model supplemented with self-loops in the compatibility graph. 
\end{example}

\subsection{Multi-class Whittle process} \label{sec:multi_class_whittle}

We then address a particular case of \emph{multi-class Whittle processes} \cite[Section~3.1]{serfozo}.
Let $n$ and $m$ be two positive integers, 
and let $\cL=\llbracket 1,n \rrbracket\times\llbracket 1,m \rrbracket$. We consider a 
Markovian discrete-event system in which elements, which are labelled in $\cL$, come, change labels, and eventually leave. 
We suppose that the events are of one of the three following types: 
\begin{itemize}
	\item \emph{Arrival} of an element of \emph{label} $ik\in \cL$;
	\item \emph{Internal transition} of an element from a label $ik\in \cL$ to a label $i\ell\in \cL$, 
	\item \emph{Departure} of an element of label $ik\in \cL$. 
\end{itemize}
At any time, the state of the system is represented by the matrix 
$s\in\mathscr M_{n,m}(\bN)$, where for all $ik\in\cL$,  
$s_{ik}$ is the number of elements having label $ik$ in the system.  
We let $\cS \subseteq \mathscr M_{n,m}(\bN)$ be the state space of the process, to be  
be entailed by the transitions below. We denote by $\mathbf 0\in\cS$, the null matrix. 
We assume that the only non-zero transitions of the underlying \gls{CTMC} are the following ones, 
\begin{equation}\label{eq:transitionWhittle}
	\left\{
	\begin{aligned}
		q(s,s+e_{i\ell})
		&=P_{0,i\ell}\phi_0,
		&& s\in\cS,\,i\ell\in\cL,\\
		q(s,s-e_{ik}+e_{i\ell})
		&=P_{ik,i\ell}\phi_{ik}(s),
		&& s\in\cS,\,k,\ell\in\llbracket 1,m \rrbracket^2,\,i\in\llbracket 1,n \rrbracket\mbox{ s.t. }\,s_{ik}>0,\\
		q(s,s-e_{ik})
		&=P_{ik,0}\phi_{ik}(s),
		&& s\in\cS,\,ik\in\cL\mbox{ s.t. }\,s_{ik}>0,
	\end{aligned}
	\right.
\end{equation}
where $\phi_0>0$,
$\phi_{ik}(s)>0$ for each $s \in \cS$ and $ik \in \cL$ such that $s_{ik}>0$, 
and $P = \left(P_{x,y}\right)_{(x,y)\in \left(\{0\}\cup (\cL)\right)^2}$ is a stochastic matrix on $\left(\{0\}\cup (\cL)\right)^2$ that is \emph{irreducible}, in the sense that for any 
$ik\in \cL$, there exists $p\ge 1$ and a $p$-uple $(k_1,\cdots,k_{p})\in \llbracket 1,m \rrbracket^p$ that includes $k$ and is such that 
\begin{equation}\label{eq:whittle-irred}
	P_{0,ik_1}\left(\prod_{l=1}^{p-1}P_{ik_l,ik_{l+1}}\right)P_{ik_p,0}>0,
\end{equation}
where, throughout, a product over an empty set of indices is understood as 1.	
It follows that the set of equations
\begin{equation*}
	\lambda_{ik}=\lambda_0P_{0,ik}+\sum_{\ell=1}^m\lambda_{i\ell}P_{i\ell,ik}
	,\quad ik\in \cL,
\end{equation*}
admits a unique solution $(\lambda_x)_{x\in\{0\}\cup\cL}$ in $(\bR_+)^{1 + nm}$ up to a positive multiplicative constant.
Therefore, setting $\lambda_0 = 1$, the system
\begin{equation}
	\lambda_{ik}=P_{0,ik}+\sum_{\ell=1}^m\lambda_{i\ell}P_{i\ell,ik}
	,\quad ik\in \cL,\label{eq:trafficWhittle2}
\end{equation}
admits a unique solution $(\lambda_{ik})_{ik\in\cL}$ in $(\bR_+)^{nm}.$

As we now check, the model can be cast in the framework of \Cref{sec:queueing_systems}, by viewing  
\emph{elements} as \emph{customers}. It follows that the counting function $|\cdot|$ can be rewritten as 
\[|\cdot|:\begin{cases}
	\cS &\longrightarrow \bN^{n};\\
	s &\longmapsto |s| \mbox{ s.t. }|s|_i=\sum_{k=1}^ms_{ik},\quad i\in\llbracket 1,n \rrbracket,
\end{cases}
\]
and we have the following.
\begin{lemma}\label{lemma:whittle-queueing}
	The triple $Q=(\cS,|\cdot|,q)$ is a queueing system in the sense of \Cref{sec:queueing_systems}. 
\end{lemma}
\begin{proof}
	See Appendix \ref{sec:preuvewhittle-queueing}.
\end{proof}

\begin{definition}
	The service rate functions $\phi:=\left\{\phi_{x}\,;\, x\in\cL\right\}$ are said to be \emph{$\Phi$-balanced} if, for some function~$\Phi: \cS \to \bR_{> 0}$, 
	\begin{equation}
		\Phi(s)\phi_{0} = \Phi(s+e_{i\ell})\phi_{i\ell}(s+e_{i\ell}),
		\quad s \in \cS,\,i\ell \in \cL\,\mbox{ s.t. }s+e_{i\ell}\in\cS. 
		\label{eq:balancePhiWhittle2}
	\end{equation}	
\end{definition}

It is a well-known fact, that Whittle networks without internal routing (namely, no internal transitions in the definition above, or in other words, $P_{ik,i\ell}=0$ for all $i,k,\ell$) are reversible when the service rates are balanced, see e.g. Chapter 1 of \cite{serfozo}. Given the form of the counting function above, it is therefore natural that the same model becomes  quasi-reversible in the sense of \Cref{def:quasirev}, if internal routing are allowed. This fact is rigourously proven in \Cref{thm:whittle-quasirev} below. In particular, the results of \Cref{sec:balanced_arrival_rates} will then imply that balanced arrival rates can be applied to this network, to obtain a quasi-reversible queueing system. Note that the application of balanced arrival control to Whittle networks without internal routing was already considered in \cite{BJP04,J10}.

\begin{proposition}
	\label{thm:whittle-quasirev}
	Assume that the mappings $\phi$ are {$\Phi$-balanced}, and consider the positive measure $\Pi$ on $\cS$ defined by 
	\begin{equation}\label{eq:pi-whittlemulti}
		\Pi(s)=\Phi(s)\prod_{ik\in\cL}\lambda_{ik}^{s_{ik}},\quad s\in\cS.\end{equation} 
	Then, the queueing system $Q:=(\cS,|\cdot|,q)$
	has stationary measures proportional to 
	$\Pi$. Moreover $Q$ is quasi-reversible in the sense of \Cref{def:quasirev}. 
\end{proposition}

\begin{proof}
	Let us first verify that the assumptions of \cite[Theorem~3.1]{serfozo} are satisfied, which will allow us to conclude directly that $\Pi$ defined by \eqref{eq:pi-whittlemulti} is a stationary measure of the \gls{CTMC} of the model.
	First observe that for all $s\in\cS$ and all $(ik,i\ell)\in \cL^2$ such that $s_{ik}>0$, 
	\begin{equation*} 
		\Phi(s)\phi_{ik}(s) = \Phi(s-e_{ik}+e_{i\ell})\phi_{i\ell}(s-e_{ik}+e_{i\ell}). 
	\end{equation*}
	This equality follows by applying
	\eqref{eq:balancePhiWhittle2} to $s:=s-e_{ik}$, and successively to $i\ell$ 
	and to $ik$, to obtain that 
	\begin{align*}
		\Phi(s-e_{ik}+e_{i\ell})\phi_{i\ell}(s-e_{ik}+e_{i\ell})&=\Phi(s-e_{ik})\phi_0\\
		&=\Phi(s-e_{ik}+e_{ik})\phi_{ik}(s-e_{ik}+e_{ik})\\&=\Phi(s)\phi_{ik}(s). 
	\end{align*}
	Moreover, summing \eqref{eq:trafficWhittle2} over $i\in\llbracket 1,n \rrbracket$ for all fixed $k\in\llbracket 1,m \rrbracket$, yields to 
	\begin{equation*}
		\sum_{i=1}^n\lambda_{ik}P_{ik,0}+\sum_{i=1}^n\sum_{\ell=1}^m\lambda_{ik}P_{ik,i\ell}=\sum_{i=1}^nP_{0,ik}+\sum_{\ell=1}^m\sum_{i=1}^n\lambda_{i\ell}P_{i\ell,ik}, \quad k\in \llbracket 1,m \rrbracket,
	\end{equation*}
	which is in turn equivalent to  
	\begin{equation}
		\sum_{i=1}^nP_{0,ik}=\sum_{i=1}^n\lambda_{ik}P_{ik,0},\quad k\in\llbracket 1,m \rrbracket. \label{eq:trafficWhittle2a}
	\end{equation}
	In particular, this readily implies that
	\begin{equation*}
		\sum_{ik\in\cL}P_{0,ik}=\sum_{ik\in\cL}\lambda_{ik}P_{ik,0}. 
	\end{equation*}
	Therefore, all the assumptions of \cite[Theorem~3.1]{serfozo} are satisfied, showing that $\Pi$ defined by 
	\eqref{eq:pi-whittlemulti} is a stationary measure of the \gls{CTMC} of the model. 
	
	We finally check that the system $Q$ is quasi-reversible in the sense of \Cref{def:quasirev}. 
	For this, take $s\in\cS$ and $i\in\llbracket 1,n \rrbracket$. We have that 
	\begin{align*}
		\sum_{t\in\cS_{|s|+e_i}}\Pi(t)q(t,s)
		&=\sum_{\ell=1}^m\Pi(s+e_{i\ell})P_{i\ell,0}\phi_{i\ell}(s+e_{i\ell}),\\
		&=\sum_{\ell=1}^mP_{i\ell,0}\Phi(s+e_{i\ell})\phi_{i\ell}(s+e_{i\ell})\left(\prod_{ik\in\cL}\lambda_{ik}^{(s+e_{i\ell})_{ik}}\right).\\
		&=\Phi(s)\phi_{0}\sum_{\ell=1}^mP_{i\ell,0}\lambda_{i\ell}\left(\prod_{ik\in\cL}\lambda_{ik}^{s_{ik}}\right).\\
		&=\Pi(s)\phi_{0}\sum_{\ell=1}^mP_{i\ell,0}\lambda_{i\ell}.\\
		&=\Pi(s)\phi_{0}\sum_{\ell=1}^mP_{0,i\ell}.\\
		&=\Pi(s)\sum_{\ell=1}^m\phi_{0}P_{0,i\ell}.\\
		&=\Pi(s)\sum_{t\in\cS_{|s|+e_i}}q(s,t),
	\end{align*}
	applying \eqref{eq:balancePhiWhittle2} in the third equality and \eqref{eq:trafficWhittle2a} in the fifth one. 
	But as was noted above, $\Pi$ is a stationary measure for the considered \gls{CTMC}. This concludes the proof. 
\end{proof}

\begin{example}\label{ex:whittle}\rm
	This model is very versatile.
	The key assumption is that the first index of the label is the class of the customer and remains constant over time. The second index, on the other hand, can change and convey other information about the customer, such as the queue at which is currently in service, the sequence of queues it visits in the network, or its service requirement.
	Here is a (non-exhaustive) list of particular cases: 
	\begin{itemize}
		\item (Multiclass) Jackson networks, as pointed out in \cite[Section~3.1]{serfozo}, when $\phi_{ik}(s)$ depends only on $s_{ik}$, for all $s,i,k$; 
		\item (Simple) Whittle networks, when $m=1$;
		\item As was observed above, Whittle networks without internal routing, whenever $P_{ik,i\ell} = 0$ for $i\in\un$ and $k,\ell\in\llbracket 1,m \rrbracket$ such that $k \neq \ell$. 
		In this case, the queueing system is not only quasi-reversible, but also reversible. 
		\item The method of stages~\cite{K76,B76} can also be applied to capture more general service-time distributions.
	\end{itemize}
\end{example}

\subsection{Equivalence between Whittle networks and OI queues} \label{sec_examples_equivalence}

\Cref{theo:equivalence_whittle_oi} below shows that multiclass Whittle networks and \gls{OI} queues
supplemented with balanced arrival control are macroscopically equivalent in steady-state. 
This result extends \cite[Theorem~2]{BC17} and \cite[Theorem~3.2]{C19}.
In the special case where the Whittle network has no internal routing (and is therefore reversible - see Example \ref{ex:whittle}), this result can be seen as a special case of the digression of \Cref{sec:reversibility}. 

\begin{proposition} \label{theo:equivalence_whittle_oi}
	Consider two integers $n, m \in \bN_{> 0}$, a Ferrers set~$\cX \subseteq \mathscr M_{n,m}(\bN)$, 
	and balance function $\Gamma \in \cB_{\cX}$.
	Consider the following two systems: 
	\begin{itemize}
		\item A Whittle network with state space $\cX$ and $\Phi$-balanced service rates~$\phi$, as defined in \Cref{sec:multi_class_whittle}; 
		\item An \gls{OI} queue as defined in \Cref{sec:oi}, with $nm$ classes, state space 
		$\cS$ such that $|\cS|=\cX$ when replacing $n$ with $nm$ in the definition of $|\cdot|$, and with arrival rates 
		$\nu_{ik}:=P_{0,ik}\phi_0,\,ik\in\un\times\llbracket 1,m \rrbracket$, and departure rate function $\mu$ such that 
		\begin{align*}
			\mu(x) = \sum_{i = 1}^n\sum_{k=1}^m \frac{\Phi(x - e_{ik})}{\Phi(x)},
			\quad x \in \cX,
			\quad ik \in \un\times \llbracket 1,m \rrbracket: x_{ik} > 0. 
		\end{align*}
	\end{itemize} 
	Suppose that the two models are also supplemented with the same $\Gamma$-balanced admission control. Then, 
	the macrostate of the controlled \gls{OI} queue has the same distribution as the controlled Whittle network at equilibrium, 
	that is, if we denote by $\Pi_{\text{W}, \Gamma}$ and $\Pi_{\text{OI}, \Gamma}$, the stationary measures of these systems, and assume 
	that $\Pi_{\text{W}, \Gamma}(\mathbf 0) = \Pi_{\text{OI}, \Gamma}(\varnothing)$, we have
	\begin{equation*}
		\Pi_{\text{W}, \Gamma}(x)
		= \sum_{\bs \in \cS: |\bs| = x} \Pi_{\text{OI}, \Gamma}(\bs),
		\quad x \in \cX.
	\end{equation*}
\end{proposition}

\begin{proof}
	With these assumptions on the two uncontrolled models, we fall into the settings of Chapter 3 of \cite{C19}. Indeed, 
	in the Whittle network, in any state $x\in\cX$ the arrival rates to the various classes are precisely 
	$P_{0,ik}\phi_0=\nu_{ik}$, $ik\in\un\times\llbracket 1,m \rrbracket$, whereas the departure rates are given by $\phi_{ik}(x)$, $ik\in\un\times\llbracket 1,m \rrbracket$. Therefore, we can 
	apply \cite[Theorem~3.2]{C19}, entailing that the result holds for the uncontrolled systems ($\Gamma \equiv 1$), 
	that is, with obvious notation, 
	\begin{equation*}
		\Pi_{\text{W}, \textrm{Id}}(x)
		= \sum_{\bs \in \cS: |\bs| = x} \Pi_{\text{OI}, \textrm{Id}}(\bs),
		\quad x \in \cX,
	\end{equation*}
	since by our very assumption, we have $\Pi_{\text{W}, \textrm{Id}}(0) = \Pi_{\text{OI}, \textrm{Id}}(\varnothing)$. 
	But from \Cref{theo:balanced_policies_preserve_quasi_reversibility}, \Cref{thm:whittle-quasirev}, and \Cref{prop:oi}, we also have 
	$\Pi_{\text{W}, \Gamma}(x) = \Pi_{\text{W}, \textrm{Id}}(x) \Gamma(x)$
	and $\Pi_{\text{OI}, \Gamma}(\bs) = \Pi_{\text{OI}, \textrm{Id}}(\bs) \Gamma(x)$
	for each $x \in \cX$ and $s \in \cS$ such that $|s| = x$,
	which concludes the proof.
\end{proof}

\begin{remark}
	This result is reminiscent of the discussion of \Cref{sec:reversibility}. It shows that we in fact have 
	$\Pi^\star_{\text{OI}, \Gamma}=\Pi_{\text{W}, \Gamma}$, that is, the (truly reversible, as an immediate consequence 
	of \Cref{thm:whittle-quasirev} in the case $m=1$) CTMC representing the controlled simple Whittle network has the same steady state distribution as the (non-Markov) macrostate process of the controlled OI queue. 
\end{remark}

\section{Admission control} \label{sec:admission_control}

To illustrate balanced policies,
we focus on the admission control problem,
which we define formally
in \Cref{sec:admission_control_problem_and_policies}.
\Cref{sec:optimal_balanced_policies} characterizes
the set of balanced admission control policies
and proves that there exists a deterministic policy
that is optimal among balanced policies.
In \Cref{sec:cost}, we give insights into
the optimality gap of balanced policies,
i.e., the difference in performance between
the best balanced policy and the optimal policy.
To simplify the discussion, we focus throughout this section
on queueing systems with a finite state space.
We believe that the methodology developed in this section,
and in particular \Cref{theo:polytope},
can be adapted to other arrival-control problems.

\subsection{Admission control problem and policies} \label{sec:admission_control_problem_and_policies}

We define an admission control problem as
a 3-tuple $\mathcal{P} = (Q, \rcont, \rdisc)$,
where $Q = (\cS, |\cdot|, q)$ is a quasi-reversible queueing system,
$\rcont: \cS \to \bR$
is the occupation-reward(-rate) function,
and $\rdisc: \cS \times \cS \to \bR$
is the transition-reward function.
For example, one can take
$\rcont(s) = -\sum_{i = 1}^n \eta_i |s|_i$
and $\rdisc(s, t) = \sum_{i = 1}^n r_i \indicator{|t| = |s| + e_i}$ 
for each $s, t \in \cS$,
where $\eta_i$ is the holding cost rate of class~$i$
and $r_i$ the admission reward of class~$i$.
An admission control policy
is a function $\gamma: \cS \to [0, 1]^n$,
such that $\gamma_i(s)$ is interpreted as the probability that
a class-$i$ customer
is admitted into the system
while the microstate is~$s$,
for each $s \in \cS$ and $i \in \un$.
The queueing system under policy~$\gamma$
is denoted by~$Q_\gamma = (\cS, |\cdot|, q_\gamma)$
and is as defined in \Cref{sec:definition}.
We are interested in solving the following optimization problem:
\begin{subequations}
	\label{eq:problem}
	\begin{align}
		\label{eq:problem_1}
		\underset{\substack{\gamma: \cS \to [0, 1]^n, \\ \Pi_\gamma: \cS \to [0, 1]}}{\textrm{Maximize}} \quad
		&G_\gamma \triangleq \sum_{s \in \cS} \Pi_\gamma(s) \rcont(s)
		+ \sum_{s, t \in \cS} \Pi_\gamma(s) q_\gamma(s, t) \rdisc(s, t), \\
		\label{eq:problem_2}
		\textrm{Subject to} \quad
		& \Pi_\gamma(s) \sum_{t \in \cS} q_\gamma(s, t)
		= \sum_{t \in \cS} \Pi_\gamma(t) q_\gamma(t, s),
		\quad \text{for each } s \in \cS, \\
		\label{eq:problem_3}
		&\sum_{s \in \cS} \Pi_\gamma(s) = 1.
	\end{align}
\end{subequations}
Here, $G_\gamma$ is called the \emph{gain} of the policy~$\gamma$,
$\rcont(s)$ is the reward per time unit in microstate~$s$,
and $\rdisc(s, t)$ is a reward received upon transition from state~$s$ to state~$t$.
Equations~\eqref{eq:problem_2}--\eqref{eq:problem_3}
say that $\Pi_\gamma$ is the stationary distribution of $Q_\gamma$.
Existence and unicity of $\Pi_\gamma$ follow from the fact that $Q_\gamma$ is unichain, which follows from our assumption that $Q$ satisfies \Cref{ass:unichain}.
Following the standard approach in \glspl{MDP} solving
as described in \cite[Section~8.8.1]{P09},
\eqref{eq:problem} can be rewritten
as a linear program;
more details are given
in \Cref{app:lp_reversible_queueing_system}
for the special case of reversible queueing systems.
It follows in particular that the optimal policy is deterministic,
i.e., it satifies $\gamma_i(s) \in \{0, 1\}$ for each $s \in \cS$.

\subsection{Balanced policies} \label{sec:optimal_balanced_policies}

Our goal in this section is to characterize
the set of balanced admission control policies
and to identify a solution of~\eqref{eq:problem}
that is optimal among all balanced policies.
Consider an admission control problem
$\mathcal{P} = (Q, \rcont, \rdisc)$,
where $Q = (\cS, |\cdot|, q)$ is
a quasi-reversible queueing system
with a finite microstate space $\cS$,
and let $\cX = |\cS|$.

Since there is a one-to-one mapping between the set of balanced policies and the set $\cB_\cX$ of balance functions, looking for the best balanced admission control policy is equivalent to looking for the best balance function. 
By combining
\Cref{theo:balanced_policies_preserve_quasi_reversibility}
and~\eqref{eq:balance_function},
we can rewrite
the restriction of~\eqref{eq:problem} to balanced policies
as an optimization problem
that is linear with respect to the balance function:
\begin{subequations}
	\label{eq:problem_balanced}
	\begin{align}
		\label{eq:problem_balanced_1}
		\underset{\Gamma: \cX \to \bR}{\text{Maximize}}
		\quad &
		G^{\gammabac} =
		\begin{aligned}[t]
			&\sum_{x \in \cX} \sum_{s \in \cS} \Pi(s) \Gamma(x) \rcont(s) \\
			&+ \sum_{x \in \cX} \sum_{i = 1}^n
			\sum_{s \in \cS_x} \sum_{t \in \cS_{x + e_i}}
			\Pi(s) \Gamma(x + e_i) q(s, t) \rdisc(s, t) \\
			&+ \sum_{x \in \cX}  \sum_{s \in \cS_x}
			\sum_{t \in \cS \setminus \bigcup_{i = 1}^n \cS_{x + e_i}}
			\Pi(s) \Gamma(x) q(s, t) \rdisc(s, t),
		\end{aligned}
		\\
		\label{eq:problem_balanced_3}
		\underset{}{\text{Subject to}}
		\quad &
		\Gamma(x) \ge 0
		\text{ for each $x \in \cX$}, \\
		\label{eq:problem_balanced_2}
		&
		\Gamma(x) \ge \Gamma(x + e_i)
		\text{ for each $x \in \cX$ and $i \in \un$}, \\
		\label{eq:problem_balanced_4}
		&
		\sum_{s \in \cS} \Pi(s) \Gamma(s) = 1,
	\end{align}
\end{subequations}
where $\Pi$ is a stationary measure of $Q$,
given by~\eqref{eq:balance-equations}.
One can verify that
$\Gamma^*$ is an optimal solution of~\eqref{eq:problem_balanced}
if and only if $\gamma^* = \gammabac^*$ is
the best solution of~\eqref{eq:problem} among balanced policies.
Indeed, \Cref{eq:problem_balanced_1} is a specialization of~\eqref{eq:problem_1}, simplified using~\eqref{eq:balance_function} and~\eqref{eq:piquasi};
Equations~\eqref{eq:problem_balanced_3},
\eqref{eq:problem_balanced_2},
and~\eqref{eq:problem_balanced_4}
are satisfied if and only if
(i) there exists $Z > 0$ such that $\frac1Z \Gamma \in \cB_\cX$
and (ii) $\gammabac_i(x) \in [0, 1]$ for each $x \in \supp(\Gamma)$.
Equation~\eqref{eq:problem_balanced_4} guarantees in particular that $Z$ is such that $s \in \cS \mapsto \Pi(s) \Gamma(s)$ is the stationary distribution of~$Q_{\gammabac}$, and not an arbitrary stationary measure.

\Cref{theo:polytope} below gives the form of an optimal solution of \eqref{eq:problem_balanced} and, correspondingly, of an admission control policy with maximal gain among all balanced policies.
It shows in particular that, when we consider the restriction of problem~\eqref{eq:problem} to balanced policies, the solution is a (balanced) deterministic policy.

\begin{theorem} \label{theo:polytope}
	Consider an admission control problem
	$\mathcal{P} = (Q, \rcont, \rdisc)$,
	where $Q = (\cS, |\cdot|, q)$ is a quasi-reversible queueing system with a finite microstate space~$\cS$.
	There exists a Ferrers set $\cA \subseteq |\cS|$
	such that the function
	$\Gamma_{\cA}: x \in \cX \mapsto \frac1Z \indicator{x \in \cA}$
	is an optimal solution of~\eqref{eq:problem_balanced},
	where $Z = \sum_{y \in \cA} \Pi(y)$.
	Correspondingly, the following policy
	is an optimal solution of~\eqref{eq:problem}
	among all balanced admission policies:
	\begin{align*}
		\gamma_i(x) &= \indicator{x + e_i \in \cA},
		\text{ for each } x \in |\cS| \text{ and } i \in \llbracket 1, n \rrbracket.
	\end{align*}
\end{theorem}

\begin{proof}
	See Appendix~\ref{app:polytope}.
\end{proof}

Though the details of the proof of \Cref{theo:polytope} can be found in Appendix~\ref{app:polytope}, let us briefly discuss the intuition behind this result. Problem~\eqref{eq:problem_balanced} is a linear optimization problem, hence it has an optimal solution that is a vertex of its polytope $\tilde\cB^{\textrm{ac}}_\cX$ of feasible solutions:
\begin{align*}
	\tilde\cB^{\textrm{ac}}_\cX = \left\{
	\Gamma: \cX \to \bR_{\ge 0};
	\text{$\Gamma$ is componentwise non-increasing}
	\text{ and } \sum_{x \in \cX} \Pi(x) \Gamma(x) = 1
	\right\}.
\end{align*}
To prove \Cref{theo:polytope},
it suffices to show that $\tilde\cB^{\textrm{ac}}_\cX$
is the convex hull of $\{\Gamma_\cA; \cA \in \cF_\cX\}$,
where $\cF_\cX$ is the family of Ferrers subsets of $\cX$.
Characterizing the vertices of~$\tilde\cB^{\textrm{ac}}_\cX$
is made complicated by the normalizing condition
$\sum_{s \in \cS} \Pi(s) \Gamma(s) = 1$ so,
as an intermediary step,
we focus on the state~$\cB^{\text{ac}}_\cX \subseteq \cB_\cX$
defined as follows:
\begin{align*}
	\cB^{\textrm{ac}}_\cX &= \left\{
	\Gamma: \cX \to \bR_{\ge 0};
	\text{$\Gamma$ is componentwise non-increasing}
	\text{ and } \Gamma(0) = 1
	\right\}.
\end{align*}
The only difference between $\cB^{\textrm{ac}}_\cX$
and $\tilde\cB^{\textrm{ac}}_\cX$ is the normalizing condition,
given by $\Gamma(0) = 1$ for $\cB^{\textrm{ac}}_\cX$
and $\sum_{x \in \cX} \Pi(x) \Gamma(x) = 1$
for $\tilde\cB^{\textrm{ac}}_\cX$.
\Cref{prop:polytope} below characterizes the vertices of $\cB^{\textrm{ac}}_\cX$. It is used as an intermediary step in the proof of \Cref{prop:polytope_bis}, which characterizes the vertices of $\tilde\cB^{\textrm{ac}}_\cX$.

\begin{proposition}[Theorem~4.2 in \cite{J10}] \label{prop:polytope}
	Let $n \in \bN$
	and consider a Ferrers set $\cX \subseteq \bN^n$.
	The set $\cB_\cX$ is the convex hull of $\{\one_{\cA}; \cA \in \cF_\cX\}$, where
	$\cF_\cX$ is the family of Ferrers subsets of~$\cX$,
	and $\one_{\cA}: x \in \cX \mapsto \indicator{x \in \cA}$.
\end{proposition}

\begin{proposition} \label{prop:polytope_bis}
	Let $n \in \bN$
	and consider a Ferrers set $\cX \subseteq \bN^n$.
	The set $\tilde\cB_\cX$ is the convex hull of
	$\{\Gamma_\cA; \cA \in \cF_\cX\}$, where
	\begin{align*}
		\Gamma_\cA: x \in \cX \mapsto
		\frac1{\sum_{y \in \cA} \Pi(y)} \indicator{x \in \cA},
		\quad \text{for each } \cA \in \cF_\cX.
	\end{align*}
\end{proposition}

\begin{proof}[Proofs of \Cref{prop:polytope,prop:polytope_bis}]
	See Appendix~\ref{app:polytope}.
\end{proof}

Observe that \Cref{prop:polytope} is
a rephrasing of \cite[Theorem~4.2]{J10},
which was stated in the special case of Whittle networks without internal transitions.
We propose an alternative proof of this result
that is more constructive.
Our proof shows in particular that,
even if $\cB_\cX$ has as many vertices as
there are Ferrers subsets of~$\cX$,
a particular $\Gamma \in \cB_\cX$
can be written as a convex combination
of only $\textrm{card}(\cX)$ of these functions.

\subsection{The cost of balance} \label{sec:cost}

\Cref{sec:numerics} will illustrate
one benefit of balanced policies
when it comes to optimizing performance
in systems with an infinite state space.
Before that, we make a digression to give insight into
the performance loss of balanced policies
with respect to an optimal policy,
still focusing on admission control
and systems with a finite state space for simplicity.
We combine insights from the literature on (quasi-)reversibility
with numerical results in toy admission control examples
solved by linear programming.

\subsubsection{Dimension reduction} \label{sec:cost_dimension}

Consider an admission control problem
$\cP = (Q, \rcont, \rdisc)$,
where $Q = (\cS, |\cdot|, q)$ has a finite state space~$\cS$,
and let $\cX = |\cS|$.
An obvious but insightful observation is that
the balance condition reduces the degree of freedom
in the choice of the admission probabilities,
as the admission policy~$\gamma = \delta \Gamma$ is fully specified
once the balance function~$\Gamma$ is given.
The number of values necessary to specify these, 
is therefore reduced from $\Theta(n m)$ to $\Theta(m)$,
where $n$ is the number of classes
and $m = \card(\cX)$ is the number of states.

\paragraph*{Monotonicity}

One way of illustrating this reduction of dimensionality
consists of describing how it restricts
the set of feasible admission probabilities.
This is the point of \Cref{prop:monotonicity} below,
which is a direct consequence of the balance condition.
\begin{proposition} \label{prop:monotonicity}
	Let $n \in \bN_{> 0}$ and $i, j \in \un$.
	Consider $\gamma = \delta \Gamma$,
	where $\Gamma \in \cB_\cX$
	and $\cX \subseteq \bN^n$ is a Ferrers set.
	For each $\diamond \in \{<, \le, \ge, >\}$
	and $x \in \cX$ such that 
	$x + e_i, x + e_j \in \supp(\Gamma)$ and $x + e_i + e_j \in |\cS|$,
	we have $\gamma_i(x) \diamond \gamma_i(x + e_j)$
	if and only if $\gamma_j(x) \diamond \gamma_j(x + e_i)$.
	In particular, the class-$i$ admission probability~$\gamma_i$ is
	increasing (resp.\ non-decreasing, non-increasing, decreasing)
	with respect to the number of class~$j$ customers
	if and only if the class-$j$ admission probability~$\gamma_j$ is
	increasing (resp.\ non-decreasing, non-increasing, decreasing)
	with respect to the number of class-$i$ customers.
\end{proposition}

\begin{proof}
	Let $\diamond \in \{\le, <, \ge, >\}$
	and $x \in \cX$ such that
	$x + e_i, x + e_j \in \supp(\Gamma)$ and $x + e_i + e_j \in |\cS|$.
	The equivalence of $\gamma_i(x) \diamond \gamma_i(x + e_j)$ and $\gamma_j(x) \diamond \gamma_j(x + e_i)$
	follows by rewriting~\eqref{eq:balance_condition} as
	\begin{align*}
		\frac{\gamma_i(x + e_j)}{\gamma_i(x)}
		&= \frac{\gamma_j(x + e_i)}{\gamma_j(x)},
	\end{align*}
	after recalling that $x + e_i, x + e_j \in \supp(\Gamma)$
	implies that $\gamma_i(x) > 0$ and $\gamma_j(x) > 0$.
	Alternately, this equivalence follows by observing that,
	according to~\eqref{eq:balance_function},
	$\gamma_i(x) \diamond \gamma_i(x + e_j)$ and $\gamma_j(x) \diamond \gamma_j(x + e_i)$
	are both equivalent to
	$\Gamma(x + e_i) \Gamma(x + e_j) \diamond \Gamma(x + e_i + e_j) \Gamma(x)$.
\end{proof}

\paragraph*{Alternative parameterization}

\begin{figure}[b]
	\centering
	\begin{tikzpicture}[xscale=2.5, yscale=2]
		\foreach \i in {0, ..., 4} {
			\foreach \j in {0, ..., 2} {
				\node (\i\j) at (\i, \j)
				{\small $\Gamma(\i, \j)$};
			}
		}
		
		\foreach \i/\k in {0/1, 1/2, 2/3, 3/4} {
			\draw[->, thick] (\i0)
			-- node[midway, below]
			{\small $\gamma_1(\i, 0)$}
			(\k0);
		}
		
		\foreach \i\k in {0/1, 1/2, 2/3, 3/4} {
			\foreach \j in {1, 2} {
				\draw[->, densely dashed] (\i\j) -- (\k\j);
			}
		}
		
		\foreach \i in {0, ...,4} {
			\foreach \j/\k in {0/1, 1/2} {
				\draw[->, thick] (\i\j)
				-- node[midway, rotate=90, above]
				{\small $\gamma_2(\i, \j)$}
				(\i\k);
			}
		}
		
	\end{tikzpicture}
	\caption{%
		A complete and consistent definition
		of the balanced admission probabilities
		in a system $n = 2$ classes and
		macrostate space $\cX = \llbracket 0, 4 \rrbracket \times \llbracket 0, 2 \rrbracket$.
		The admission probabilities along the dashed arrows are fixed
		once we fix the admission probabilities along the solid edges.
	}
	\label{fig:balance_conditiond-admission-probabilities}
\end{figure}

To understand more concretely how the balance condition
restricts the choice of the admission probabilities,
let us consider an alternative way
of specifying the admission probabilities.
Consider a Ferrers set $\cX \subseteq \bN^n$,
a balance function $\Gamma \in \cB^{\text{ac}}_\cX$,
and a $\Gamma$-balanced admission policy~$\gamma$.
By replacing $\cX$ with $\supp(\Gamma)$ if necessary,
we can assume that $\Gamma$ is positive.
We can show that,
under the balance condition~\eqref{eq:balance_function}--\eqref{eq:balance_condition},
\emph{all} admission probabilities under~$\gamma$ are specified
once we set the values of the following admission probabilities:
$\gamma_i(x_1, x_2, \ldots, x_i, 0, \ldots, 0)$,
for each $i \in \un$
and $(x_1, x_2, \ldots, x_{i-1}, x_i + 1, 0, \ldots, 0) \in \cX$.
Indeed, unfolding recursion~\eqref{eq:balance_function} shows that, for each $x \in \cX$,
$\Gamma(x)$ can be computed by taking the product
of admission probabilities along the increasing path
going from the origin to~$x$
by going parallel first to the $x_1$-axis,
then to the $x_2$-axis, and so on:
\begin{align*}
	\Gamma(x)
	= \prod_{i = 1}^n \prod_{y = 0}^{x_i - 1}
	\gamma_i(x_1, \ldots, x_{i-1}, y, 0, \ldots, 0),
	\quad \text{for each } x \in \cX.
\end{align*}
An example is shown
in \Cref{fig:balance_conditiond-admission-probabilities}
for a queueing system with $n = 2$ classes
and macrostate space $\cX = \llbracket 0, 4 \rrbracket \times \llbracket 0, 2 \rrbracket$.
We only need to specify
$\gamma_1(x_1, 0)$ for $x_1 \in \llbracket 0, 4 \llbracket$
and $\gamma_2(x_1, x_2)$ for $x_1 \in \llbracket 0, 4 \rrbracket$ and $x_2 \in \llbracket 0, 2 \llbracket$.
The value of $\gamma_1(x_1, x_2)$
for $x_1 \in \llbracket 0, 4 \llbracket$ and $x_2 \in \rrbracket 0, 2 \rrbracket$
follows by writing
\begin{align*}
	\gamma_1(x_1, x_2)
	&= \frac{\Gamma(x_1 + 1, x_2)}{\Gamma(x_1, x_2)}
	= \gamma_1(x_1, 0)
	\prod_{j = 0}^{x_2 - 1} \frac{\gamma_2(x_1 + 1, j)}{\gamma_2(x_1, j)}\cdot
\end{align*}
Intuitively, the class-$1$ admission probability
in state $(x_1, x_2)$ is ``small''
if the class-$1$ admission probability
is already ``small'' in the absence of class~$2$
and/or if the class-$2$ admission probability
is ``significantly reduced''
by adding a class-$1$ customer
in a state $(x_1, y)$ for  $y \in \llbracket 0, x_2 \llbracket$.

\paragraph*{Toy example}

Let us illustrate this concept on a toy example.
Consider a 2-class processor-sharing queue.
Customers of each class arrive according to
an independent Poisson process with rate~$0.1$.
The queue can contain at most~5 customers of each class,
so that an incoming customer is rejected
if the queue already contains 5 customers of the corresponding class.
The service requirement of each customer is exponentially distributed with mean~1,
and the service capacity of the server is also~1.
In this way, we consider a queueing system $Q_\gamma = (\cS, \textrm{Id}, q_\gamma)$,
where $\cS = \llbracket 0, 5 \rrbracket^2$,
\begin{align*}
	\left.
	\begin{aligned}
		q_\gamma(s, s + e_i) &= 0.1 \gamma_i(s) \indicator{s_i < 5}, \\
		q_\gamma(s, s - e_i) &= \frac{s_i}{s_1 + s_2},
	\end{aligned}
	\right\}
	\quad \text{for each } s \in \cS
	\text{ and } i \in \llbracket 1, 2 \rrbracket,
\end{align*}
and other entries of~$q_\gamma$ are zero.
The reward functions are given by
$\rcont(s) = 0$ and
$\rdisc(s, t) = \indicator{s \in \mathcal{L}, t \in \mathcal{L}, \text{ and } |s| \le |t|}$, for each $s, t \in \cS$, where
\begin{align*}
\mathcal{L} = \Bigl\{(0, 0), (1, 0), (2, 0), (2, 1), (3, 1), (3, 2), (3, 3), (3, 4), (4, 4), (4, 5), (5, 5)\Bigl\}. 
\end{align*}
In other words, a reward is received each time an admission transition occurs along the path~$\mathcal{L}$, and no reward is received otherwise.
Note that this reward structure is chosen, not to be realistic,
but to illustrate the impact of the dimension reduction in balanced policies.
Recall that, by \Cref{sec:admission_control_problem_and_policies,sec:optimal_balanced_policies},
the optimal policy and the best balanced policy
can both be chosen to be deterministic.

\begin{figure}[htb]
	\subfloat[Optimal policy\label{fig:cost_dimension_policies_a}]{%
		\begin{tikzpicture}[
			state/.style={circle, draw, inner sep=1, outer sep=1, font=\tiny, text=BurntOrange, minimum size=20},
			action/.style={->, thick, SkyBlue},
			]
			\draw[->] (5.45, 0) -- node[pos=1, above] {$x_1$} (5.75, 0);
			\draw[->] (0, 5.45) -- node[pos=1, right] {$x_2$} (0, 5.75);
			
			\node[state, fill=SkyBlue!100] (00) at (0, 0) {-0.05};
			\node[state, fill=SkyBlue!71.76] (01) at (0, 1) {-3.57};
			\node[state, fill=SkyBlue!51.89] (02) at (0, 2) {-6.06};
			\node[state, fill=SkyBlue!41.53] (03) at (0, 3) {-7.35};
			\node[state, fill=SkyBlue!31.21] (04) at (0, 4) {-8.64};
			\node[state, fill=SkyBlue!10.4] (05) at (0, 5) {-11.24};
			\node[state, fill=SkyBlue!91.99] (10) at (1, 0) {-1.05};
			\node[state, fill=SkyBlue!74.15] (11) at (1, 1) {-3.27};
			\node[state, fill=SkyBlue!55.53] (12) at (1, 2) {-5.6};
			\node[state, fill=SkyBlue!46.17] (13) at (1, 3) {-6.77};
			\node[state, fill=SkyBlue!36.79] (14) at (1, 4) {-7.94};
			\node[state, fill=SkyBlue!16.63] (15) at (1, 5) {-10.46};
			\node[state, fill=SkyBlue!83.77] (20) at (2, 0) {-2.07};
			\node[state, fill=SkyBlue!75.88] (21) at (2, 1) {-3.06};
			\node[state, fill=SkyBlue!57.76] (22) at (2, 2) {-5.32};
			\node[state, fill=SkyBlue!49.16] (23) at (2, 3) {-6.4};
			\node[state, fill=SkyBlue!40.6] (24) at (2, 4) {-7.47};
			\node[state, fill=SkyBlue!20.98] (25) at (2, 5) {-9.92};
			\node[state, fill=SkyBlue!62.86] (30) at (3, 0) {-4.69};
			\node[state, fill=SkyBlue!67.68] (31) at (3, 1) {-4.08};
			\node[state, fill=SkyBlue!59.7] (32) at (3, 2) {-5.08};
			\node[state, fill=SkyBlue!51.71] (33) at (3, 3) {-6.08};
			\node[state, fill=SkyBlue!43.54] (34) at (3, 4) {-7.1};
			\node[state, fill=SkyBlue!24.39] (35) at (3, 5) {-9.49};
			\node[state, fill=SkyBlue!20.65] (40) at (4, 0) {-9.96};
			\node[state, fill=SkyBlue!26.24] (41) at (4, 1) {-9.26};
			\node[state, fill=SkyBlue!30.05] (42) at (4, 2) {-8.78};
			\node[state, fill=SkyBlue!32.98] (43) at (4, 3) {-8.42};
			\node[state, fill=SkyBlue!35.39] (44) at (4, 4) {-8.12};
			\node[state, fill=SkyBlue!27.21] (45) at (4, 5) {-9.14};
			\node[state, fill=SkyBlue!0] (50) at (5, 0) {-12.54};
			\node[state, fill=SkyBlue!6.23] (51) at (5, 1) {-11.76};
			\node[state, fill=SkyBlue!10.58] (52) at (5, 2) {-11.22};
			\node[state, fill=SkyBlue!13.99] (53) at (5, 3) {-10.79};
			\node[state, fill=SkyBlue!16.81] (54) at (5, 4) {-10.44};
			\node[state, fill=SkyBlue!19.22] (55) at (5, 5) {-10.14};
			
			\foreach \i/\j/\k in {0/0/1, 1/0/2, 1/1/2, 2/1/3, 2/2/3, 2/3/3, 3/4/4, 4/5/5} {
				\draw[action] (\i\j) -- (\k\j);
			}
			
			\foreach \i/\j/\k in {2/0/1, 3/1/2, 3/2/3, 3/3/4, 4/4/5} {
				\draw[action] (\i\j) -- (\i\k);
			}
		\end{tikzpicture}
	}
	\hfill
	\subfloat[Best balanced policy\label{fig:cost_dimension_policies_b}]{%
		\begin{tikzpicture}[
			emptystate/.style={circle, draw, inner sep=1, outer sep=1, font=\tiny, text=BurntOrange},
			state/.style={emptystate, minimum size=20},
			action/.style={->, thick, SkyBlue},
			]
			\draw[->] (5.2, 0) -- node[pos=1, above] {$x_1$} (5.75, 0);
			\draw[->] (0, 5.2) -- node[pos=1, right] {$x_2$} (0, 5.75);
			
			\node[state, fill=SkyBlue!100] (00) at (0, 0) {-0.05};
			\node[emptystate] (01) at (0, 1) {};
			\node[emptystate] (02) at (0, 2) {};
			\node[emptystate] (03) at (0, 3) {};
			\node[emptystate] (04) at (0, 4) {};
			\node[emptystate] (05) at (0, 5) {};
			\node[state, fill=SkyBlue!50] (10) at (1, 0) {-1.05};
			\node[emptystate] (11) at (1, 1) {};
			\node[emptystate] (12) at (1, 2) {};
			\node[emptystate] (13) at (1, 3) {};
			\node[emptystate] (14) at (1, 4) {};
			\node[emptystate] (15) at (1, 5) {};
			\node[state, fill=SkyBlue!0] (20) at (2, 0) {-2.05};
			\node[emptystate] (21) at (2, 1) {};
			\node[emptystate] (22) at (2, 2) {};
			\node[emptystate] (23) at (2, 3) {};
			\node[emptystate] (24) at (2, 4) {};
			\node[emptystate] (25) at (2, 5) {};
			\node[emptystate] (30) at (3, 0) {};
			\node[emptystate] (31) at (3, 1) {};
			\node[emptystate] (32) at (3, 2) {};
			\node[emptystate] (33) at (3, 3) {};
			\node[emptystate] (34) at (3, 4) {};
			\node[emptystate] (35) at (3, 5) {};
			\node[emptystate] (40) at (4, 0) {};
			\node[emptystate] (41) at (4, 1) {};
			\node[emptystate] (42) at (4, 2) {};
			\node[emptystate] (43) at (4, 3) {};
			\node[emptystate] (44) at (4, 4) {};
			\node[emptystate] (45) at (4, 5) {};
			\node[emptystate] (50) at (5, 0) {};
			\node[emptystate] (51) at (5, 1) {};
			\node[emptystate] (52) at (5, 2) {};
			\node[emptystate] (53) at (5, 3) {};
			\node[emptystate] (54) at (5, 4) {};
			\node[emptystate] (55) at (5, 5) {};
			
			\foreach \i/\k in {0/1, 1/2} {
				\foreach \j in {0} {
					\draw[action] (\i\j) -- (\k\j);
				}
			}
			
		\end{tikzpicture}
	}
	\caption{Comparison of the optimal policy and the best balanced policy in the toy example of \Cref{sec:cost_dimension}. Blue edges represent admissions that are authorized (with probability 1) by the policy. The value inside each state gives the logarithm in base 10 of the value of the stationary distribution in this state under the depicted policy.}
	\label{fig:cost_dimension_policies}
\end{figure}

\begin{figure}[htb]
	\includegraphics[width=.6\textwidth]{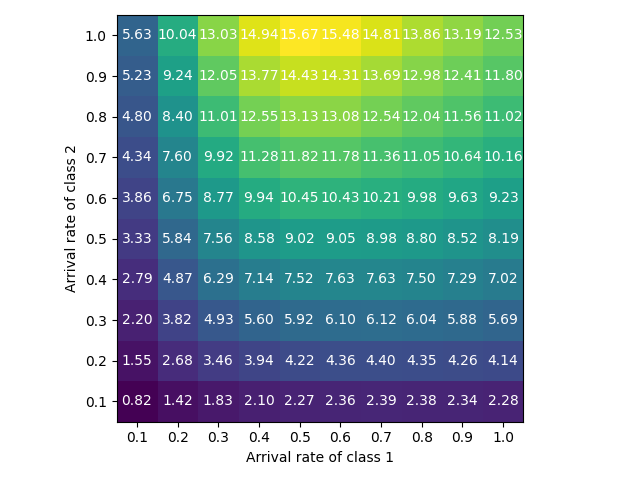}
	\caption{Percentage loss of the best balanced policy with respect to the optimal policy, as a function of the per-class arrival rates in the toy example of \Cref{sec:cost_dimension}.}
	\label{fig:cost_dimension_percentage_loss}
\end{figure}

\Cref{fig:cost_dimension_policies} shows
the optimal policy and the best balanced policy.
States are depicted as circles on a grid.
Thick blue edges represent admissions that are authorized by the corresponding policy,
and the value inside each state shows the logarithm of the stationary probability of being in this state under the corresponding policy.
In \Cref{fig:cost_dimension_policies_a}, we see that the optimal policy
allows transitions along the path~$\cP$,
plus a few additional transitions.
On the contrary, the best balanced policy shown in \Cref{fig:cost_dimension_policies_b}
admits only arrivals of class~1, and only in states $x = (0, 0)$ and $x = (0, 1)$.
With our choice of arrival rates,
such a significant difference in policy definition
has a small impact on the policy's gain,
as large states have a small stationary probability anyway.
More precisely, the suboptimality of the best balanced policy in this toy example is given by
\begin{align*}
	\textrm{Loss} &= \frac{G^{\gamma^*} - G^{\delta\Gamma*}}{G^{\gamma^*} - G^{\gamma_*}}
	\simeq 0.82\%,
\end{align*}
where $\gamma^*$ is an optimal policy, $\gamma_*$ a worst policy, and $\gammabac^*$ the best balanced policy. In this particular example, $\gamma_*$ is the reject-all policy and $G^{\gamma_*} = 0$.

\Cref{fig:cost_dimension_percentage_loss} shows the evolution of the loss of the best balanced policy with respect to the optimal policy, as a function of the arrival rates of the two classes. The loss of approximately 0.82\% mentioned above appears on the lower left square, corresponding to both classes having an arrival rate of 0.1. In general, the structure of the optimal and best balanced policies remain similar to \Cref{fig:cost_dimension_policies}, but the loss increases when larger states become more likely. The suboptimality of balanced policies is maximal when the arrival rate of class~1 is 0.5 and that of class~2 is 1. Detailed results (not shown here) reveal that the optimal policy tends to accept class-1 customers in more states than in \Cref{fig:cost_dimension_policies_a}, while the best balanced policy is the same as in \Cref{fig:cost_dimension_policies_b}.

\subsubsection{No net circulation in state space} \label{sec:cost_circulation}

Classical works on reversibility
also help shed light on the balance condition.
According to~\cite{W75},
``a reversible Markov process shows no net circulation in state-space''.
This means that,
in a reversible stationary \gls{CTMC},
a cycle in the transition diagram
is traversed equally frequently,
whether we orient this cycle one way or the other,
as formalized by Kolmogorov's criterion \cite[Section~1.5]{K11}.
We observed in \Cref{sec:reversibility}
that the macrostate-process of a quasi-reversible queueing system
has the same stationary distribution and transition rates
as a reversible \gls{CTMC}.
Therefore, \Cref{theo:balanced_policies_preserve_quasi_reversibility} implies that
a quasi-reversible queueing system subject to a balanced policy
also shows no net circulation in state space,
when looked at the macrostate level.

\paragraph*{Toy example}

Let us illustrate this concept with 
the same toy example as in \Cref{sec:cost_dimension},
namely, a processor-sharing queue with two customer classes,
each arriving as a Poisson process with rate $0.1$,
and with state space $\cS = \llbracket 0, 5 \rrbracket^2$.
Compared to \Cref{sec:cost_dimension},
we change the reward functions, which are now given by
$\rcont(s) = \indicator{s_1 = 5, s_2 = 0}$
and $\rdisc(s, t) = 0$,
for each $s, t \in \cS$.
In other words, the goal is to maximize the probability of being in the state
where the system is full of class-1 customers and empty of class-2 customers.
Again, recall that this example has been built
to provide insights on the impact of the balance condition,
and not to be realistic.

\begin{figure}[htb]
	\subfloat[Optimal policy\label{fig:cost_circulation_policies_a}]{%
		\begin{tikzpicture}[
			state/.style={circle, draw, inner sep=1, outer sep=1, font=\tiny, text=BurntOrange, minimum size=20},
			action/.style={->, thick, SkyBlue},
			]
			\draw[->] (5.45, 0) -- node[pos=1, above] {$x_1$} (5.75, 0);
			\draw[->] (0, 5.45) -- node[pos=1, right] {$x_2$} (0, 5.75);
			
			\node[state, fill=SkyBlue!99.04] (00) at (0, 0) {-0.10};
			\node[state, fill=SkyBlue!89.02] (01) at (0, 1) {-1.10};
			\node[state, fill=SkyBlue!78.99] (02) at (0, 2) {-2.10};
			\node[state, fill=SkyBlue!68.94] (03) at (0, 3) {-3.11};
			\node[state, fill=SkyBlue!58.88] (04) at (0, 4) {-4.11};
			\node[state, fill=SkyBlue!48.81] (05) at (0, 5) {-5.12};
			\node[state, fill=SkyBlue!89.05] (10) at (1, 0) {-1.09};
			\node[state, fill=SkyBlue!81.9 ] (11) at (1, 1) {-1.81};
			\node[state, fill=SkyBlue!73.47] (12) at (1, 2) {-2.65};
			\node[state, fill=SkyBlue!64.51] (13) at (1, 3) {-3.55};
			\node[state, fill=SkyBlue!55.27] (14) at (1, 4) {-4.47};
			\node[state, fill=SkyBlue!45.84] (15) at (1, 5) {-5.42};
			\node[state, fill=SkyBlue!79.36] (20) at (2, 0) {-2.06};
			\node[state, fill=SkyBlue!72.12] (21) at (2, 1) {-2.79};
			\node[state, fill=SkyBlue!63.64] (22) at (2, 2) {-3.64};
			\node[state, fill=SkyBlue!54.65] (23) at (2, 3) {-4.54};
			\node[state, fill=SkyBlue!45.35] (24) at (2, 4) {-5.47};
			\node[state, fill=SkyBlue!35.57] (25) at (2, 5) {-6.44};
			\node[state, fill=SkyBlue!69.59] (30) at (3, 0) {-3.04};
			\node[state, fill=SkyBlue!62.29] (31) at (3, 1) {-3.77};
			\node[state, fill=SkyBlue!53.78] (32) at (3, 2) {-4.62};
			\node[state, fill=SkyBlue!44.75] (33) at (3, 3) {-5.52};
			\node[state, fill=SkyBlue!35.4 ] (34) at (3, 4) {-6.46};
			\node[state, fill=SkyBlue!25.33] (35) at (3, 5) {-7.47};
			\node[state, fill=SkyBlue!59.77] (40) at (4, 0) {-4.02};
			\node[state, fill=SkyBlue!52.44] (41) at (4, 1) {-4.76};
			\node[state, fill=SkyBlue!43.9 ] (42) at (4, 2) {-5.61};
			\node[state, fill=SkyBlue!34.86] (43) at (4, 3) {-6.51};
			\node[state, fill=SkyBlue!25.44] (44) at (4, 4) {-7.46};
			\node[state, fill=SkyBlue!15.12] (45) at (4, 5) {-8.49};
			\node[state, fill=SkyBlue!49.9 ] (50) at (5, 0) {-5.01};
			\node[state, fill=SkyBlue!42.62] (51) at (5, 1) {-5.74};
			\node[state, fill=SkyBlue!34.11] (52) at (5, 2) {-6.59};
			\node[state, fill=SkyBlue!25.08] (53) at (5, 3) {-7.49};
			\node[state, fill=SkyBlue!15.64] (54) at (5, 4) {-8.44};
			\node[state, fill=SkyBlue!5.12] (55) at (5, 5) {-9.49};
			
			\foreach \i/\k in {0/1, 1/2, 2/3, 3/4, 4/5} {
				\foreach \j in {0, 1, ..., 5} {
					\draw[action] (\i\j) -- (\k\j);
				}
			}
			
			\foreach \i in {0, 1}  {
				\foreach \j/\k in {0/1, 1/2, 2/3, 3/4, 4/5} {
					\draw[action] (\i\j) -- (\i\k);
				}
			}
		\end{tikzpicture}
	}
	\hfill
	\subfloat[Best balanced policy\label{fig:cost_circulation_policies_b}]{%
		\begin{tikzpicture}[
			emptystate/.style={circle, draw, inner sep=1, outer sep=1, font=\tiny, text=BurntOrange},
			state/.style={emptystate, minimum size=20},
			action/.style={->, thick, SkyBlue},
			]
			\draw[->] (5.45, 0) -- node[pos=1, above] {$x_1$} (5.75, 0);
			\draw[->] (0, 5.2) -- node[pos=1, right] {$x_2$} (0, 5.75);
			
			\node[state, fill=SkyBlue!99.54] (00) at (0, 0) {-0.05};
			\node[emptystate] (01) at (0, 1) {};
			\node[emptystate] (02) at (0, 2) {};
			\node[emptystate] (03) at (0, 3) {};
			\node[emptystate] (04) at (0, 4) {};
			\node[emptystate] (05) at (0, 5) {};
			\node[state, fill=SkyBlue!89.54] (10) at (1, 0) {-1.05};
			\node[emptystate] (11) at (1, 1) {};
			\node[emptystate] (12) at (1, 2) {};
			\node[emptystate] (13) at (1, 3) {};
			\node[emptystate] (14) at (1, 4) {};
			\node[emptystate] (15) at (1, 5) {};
			\node[state, fill=SkyBlue!79.54] (20) at (2, 0) {-2.05};
			\node[emptystate] (21) at (2, 1) {};
			\node[emptystate] (22) at (2, 2) {};
			\node[emptystate] (23) at (2, 3) {};
			\node[emptystate] (24) at (2, 4) {};
			\node[emptystate] (25) at (2, 5) {};
			\node[state, fill=SkyBlue!69.54] (30) at (3, 0) {-3.05};
			\node[emptystate] (31) at (3, 1) {};
			\node[emptystate] (32) at (3, 2) {};
			\node[emptystate] (33) at (3, 3) {};
			\node[emptystate] (34) at (3, 4) {};
			\node[emptystate] (35) at (3, 5) {};
			\node[state, fill=SkyBlue!59.54] (40) at (4, 0) {-4.05};
			\node[emptystate] (41) at (4, 1) {};
			\node[emptystate] (42) at (4, 2) {};
			\node[emptystate] (43) at (4, 3) {};
			\node[emptystate] (44) at (4, 4) {};
			\node[emptystate] (45) at (4, 5) {};
			\node[state, fill=SkyBlue!49.54] (50) at (5, 0) {-5.05};
			\node[emptystate] (51) at (5, 1) {};
			\node[emptystate] (52) at (5, 2) {};
			\node[emptystate] (53) at (5, 3) {};
			\node[emptystate] (54) at (5, 4) {};
			\node[emptystate] (55) at (5, 5) {};
			
			\foreach \i/\k in {0/1, 1/2, 2/3, 3/4, 4/5} {
				\foreach \j in {0} {
					\draw[action] (\i\j) -- (\k\j);
				}
			}
		\end{tikzpicture}
	}
	\caption{Comparison of the optimal policy and the best balanced policy in the toy example of \Cref{sec:cost_circulation}.
		Blue edges represent admissions that are authorized (with probability 1) by the policy.
		The value inside each state gives the logarithm in base 10 of the value of the stationary distribution in this state under the depicted policy. We see that state $x = (5, 0)$ is more likely under the optimal policy
		($\log \Pi^{\gamma^*}(5, 0) = -5.01$) than under the best balanced policy
		($\log \Pi^{\delta \Gamma^*}(5, 0) = -5.05$).}
	\label{fig:cost_circulation_policies}
\end{figure}

\begin{figure}[htb]
	\includegraphics[width=.6\textwidth]{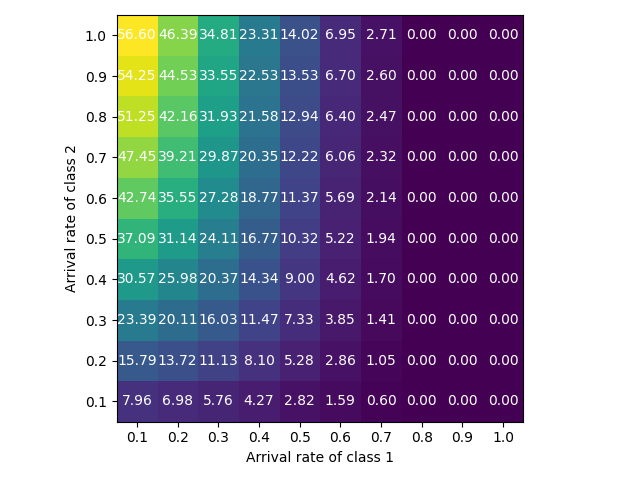}
	\caption{Percentage loss of the best balanced policy with respect to the optimal policy, as a function of the per-class arrival rates in the toy example of \Cref{sec:cost_circulation}.}
	\label{fig:cost_circulation_percentage_loss}
\end{figure}

The results are shown in
\Cref{fig:cost_circulation_policies}.
\Cref{fig:cost_circulation_policies_a} shows
the optimal policy,
and \Cref{fig:cost_circulation_policies_b} shows
the best balanced policy.
States are depicted as circles on a grid.
Thick edges represent admissions that are authorized
by the corresponding policy,
and the value inside each state shows
the logarithm of the stationary probability
of being in this state
under the corresponding policy.
The main observation is that,
compared to $\delta \Gamma^*$,
the optimal policy~$\gamma^*$
leads to a higher probability of being in state $(5, 0)$
by admitting class-2 customers in states with 0 or 1 class-1 customers. Intuitively, since the arrival rates are small compared to the service rates, it is advantageous to exploit circulation along cycles that visit states with a positive number of class-2 customers.
To give an order of magnitude of the suboptimality of balanced policies in this toy example, we have
\begin{align*}
	\text{Loss}
	= \frac
	{G^{\gamma^*} - G^{\delta \Gamma^*}}
	{G^{\gamma^*} - G^{\gamma_*}}
	\simeq
	7.96\%.
\end{align*}
\Cref{fig:cost_circulation_percentage_loss} shows the evolution of the loss as a function of the arrival rates of both classes.
In this example,
the loss of the best balanced policy
with respect to the optimal policy
decreases with the arrival rate of class~1
and increases with that of class~2.
Intuitively, as the arrival rate of class~1 increases,
the optimal policies shifts from the policy of \Cref{fig:cost_circulation_policies_a}
(leading to a maximal loss)
to the best balanced policy
shown in \Cref{fig:cost_circulation_policies_b}
(leading to a zero loss).

\subsubsection{A more realistic scenario} \label{sec:cost_realistic}

Let us consider one last variant of the toy example of \Cref{sec:cost_dimension},
corresponding to a more realistic application.
We still have a processor-sharing queue with two customer classes, each arriving as a Poisson process with rate 0.1, and with state space $\cS = \llbracket 0, 5 \rrbracket^2$. We consider the following cost structure: for each $s, t \in \cS$,
\begin{align*}
	\rcont(s) &= - s_1 - s_2,
	&
	\rdisc(s, t) &= 3 \indicator{|t| = |s| + e_1} + 6 \indicator{|t| = |s| + e_2}.
\end{align*}
In this way, $\rcont$ represents a holding cost of 1 per job per time unit, and $\rdisc$ represents an admission reward of 3 for class-1 jobs and 6 for class-2 jobs.

\begin{figure}[htb]
	\subfloat[Optimal policy\label{fig:cost_realistic_policies_a}]{%
		\begin{tikzpicture}[
			emptystate/.style={circle, draw, inner sep=1, outer sep=1, font=\tiny, text=BurntOrange},
			state/.style={emptystate, minimum size=20},
			action/.style={->, thick, SkyBlue},
			]
			\draw[->] (5.2, 0) -- node[pos=1, above] {$x_1$} (5.75, 0);
			\draw[->] (0, 5.45) -- node[pos=1, right] {$x_2$} (0, 5.75);
			
			\node[state, fill=SkyBlue!100] (00) at (0, 0) {-0.09};
			\node[state, fill=SkyBlue!80.33] (01) at (0, 1) {-1.09};
			\node[state, fill=SkyBlue!61.2] (02) at (0, 2) {-2.06};
			\node[state, fill=SkyBlue!41.89] (03) at (0, 3) {-3.04};
			\node[state, fill=SkyBlue!22.46] (04) at (0, 4) {-4.03};
			\node[state, fill=SkyBlue!2.74] (05) at (0, 5) {-5.03};
			\node[state, fill=SkyBlue!80.22] (10) at (1, 0) {-1.1};
			\node[state, fill=SkyBlue!66.46] (11) at (1, 1) {-1.8};
			\node[state, fill=SkyBlue!46.7] (12) at (1, 2) {-2.8};
			\node[state, fill=SkyBlue!26.96] (13) at (1, 3) {-3.8};
			\node[state, fill=SkyBlue!7.23] (14) at (1, 4) {-4.8};
			\node[emptystate] (15) at (1, 5) {};
			\node[state, fill=SkyBlue!59.94] (20) at (2, 0) {-2.13};
			\node[state, fill=SkyBlue!39.78] (21) at (2, 1) {-3.15};
			\node[state, fill=SkyBlue!19.72] (22) at (2, 2) {-4.17};
			\node[state, fill=SkyBlue!0.] (23) at (2, 3) {-5.17};
			\node[emptystate] (24) at (2, 4) {};
			\node[emptystate] (25) at (2, 5) {};
			\node[emptystate] (30) at (3, 0) {};
			\node[emptystate] (31) at (3, 1) {};
			\node[emptystate] (32) at (3, 2) {};
			\node[emptystate] (33) at (3, 3) {};
			\node[emptystate] (34) at (3, 4) {};
			\node[emptystate] (35) at (3, 5) {};
			\node[emptystate] (40) at (4, 0) {};
			\node[emptystate] (41) at (4, 1) {};
			\node[emptystate] (42) at (4, 2) {};
			\node[emptystate] (43) at (4, 3) {};
			\node[emptystate] (44) at (4, 4) {};
			\node[emptystate] (45) at (4, 5) {};
			\node[emptystate] (50) at (5, 0) {};
			\node[emptystate] (51) at (5, 1) {};
			\node[emptystate] (52) at (5, 2) {};
			\node[emptystate] (53) at (5, 3) {};
			\node[emptystate] (54) at (5, 4) {};
			\node[emptystate] (55) at (5, 5) {};
			
			\foreach \i/\j/\k in {0/0/1, 0/1/1, 1/0/2} {
				\draw[action] (\i\j) -- (\k\j);
			}
			
			\foreach \i/\j/\k in {0/0/1, 0/1/2, 0/2/3, 0/3/4, 0/4/5, 1/0/1, 1/1/2, 1/2/3, 1/3/4, 2/0/1, 2/1/2, 2/2/3} {
				\draw[action] (\i\j) -- (\i\k);
			}
		\end{tikzpicture}
	}
	\hfill
	\subfloat[Best balanced policy\label{fig:cost_realistic_policies_b}]{%
		\begin{tikzpicture}[
			emptystate/.style={circle, draw, inner sep=1, outer sep=1, font=\tiny, text=BurntOrange},
			state/.style={emptystate, minimum size=20},
			action/.style={->, thick, SkyBlue},
			]
			\draw[->] (5.2, 0) -- node[pos=1, above] {$x_1$} (5.75, 0);
			\draw[->] (0, 5.45) -- node[pos=1, right] {$x_2$} (0, 5.75);
			
			\node[state, fill=SkyBlue!100] (00) at (0, 0) {-0.1};
			\node[state, fill=SkyBlue!80] (01) at (0, 1) {-1.1};
			\node[state, fill=SkyBlue!60] (02) at (0, 2) {-2.1};
			\node[state, fill=SkyBlue!40] (03) at (0, 3) {-3.1};
			\node[state, fill=SkyBlue!20] (04) at (0, 4) {-4.1};
			\node[state, fill=SkyBlue!0] (05) at (0, 5) {-5.1};
			\node[state, fill=SkyBlue!80] (10) at (1, 0) {-1.1};
			\node[state, fill=SkyBlue!66.02] (11) at (1, 1) {-1.8};
			\node[state, fill=SkyBlue!49.54] (12) at (1, 2) {-2.62};
			\node[state, fill=SkyBlue!32.04] (13) at (1, 3) {-3.49};
			\node[emptystate] (14) at (1, 4) {};
			\node[emptystate] (15) at (1, 5) {};
			\node[state, fill=SkyBlue!60] (20) at (2, 0) {-2.1};
			\node[state, fill=SkyBlue!49.54] (21) at (2, 1) {-2.62};
			\node[emptystate] (22) at (2, 2) {};
			\node[emptystate] (23) at (2, 3) {};
			\node[emptystate] (24) at (2, 4) {};
			\node[emptystate] (25) at (2, 5) {};
			\node[emptystate] (30) at (3, 0) {};
			\node[emptystate] (31) at (3, 1) {};
			\node[emptystate] (32) at (3, 2) {};
			\node[emptystate] (33) at (3, 3) {};
			\node[emptystate] (34) at (3, 4) {};
			\node[emptystate] (35) at (3, 5) {};
			\node[emptystate] (40) at (4, 0) {};
			\node[emptystate] (41) at (4, 1) {};
			\node[emptystate] (42) at (4, 2) {};
			\node[emptystate] (43) at (4, 3) {};
			\node[emptystate] (44) at (4, 4) {};
			\node[emptystate] (45) at (4, 5) {};
			\node[emptystate] (50) at (5, 0) {};
			\node[emptystate] (51) at (5, 1) {};
			\node[emptystate] (52) at (5, 2) {};
			\node[emptystate] (53) at (5, 3) {};
			\node[emptystate] (54) at (5, 4) {};
			\node[emptystate] (55) at (5, 5) {};
			
			\foreach \i/\j/\k in {0/0/1, 1/0/2, 0/1/1, 1/1/2, 0/2/1, 0/3/1} {
				\draw[action] (\i\j) -- (\k\j);
			}
			
			\foreach \i/\j/\k in {0/0/1, 0/1/2, 0/2/3, 0/3/4, 0/4/5, 1/0/1, 1/1/2, 1/2/3, 2/0/1} {
				\draw[action] (\i\j) -- (\i\k);
			}
		\end{tikzpicture}
	}
	\caption{Comparison of the optimal policy and the best balanced policy in the toy example of \Cref{sec:cost_realistic}.
		Blue edges represent admissions that are authorized (with probability 1) by the policy.
		The value inside each state gives the logarithm in base 10 of the value of the stationary distribution in this state under the depicted policy.}
	\label{fig:cost_realistic_policies}
\end{figure}

\begin{figure}[htb]
	\includegraphics[width=.6\textwidth]{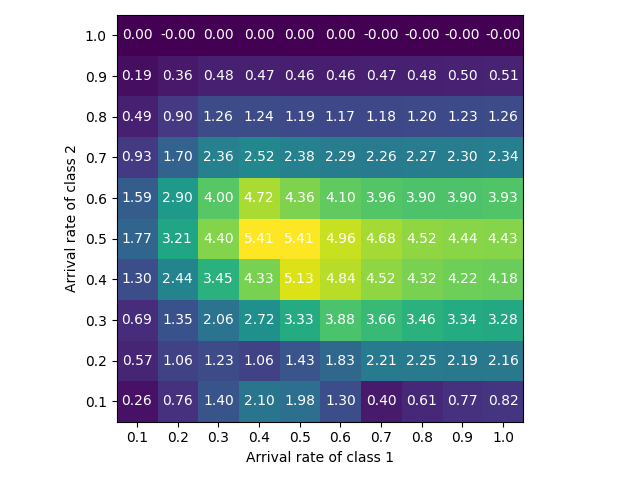}
	\caption{Percentage loss of the best balanced policy with respect to the optimal policy, as a function of the per-class arrival rates in the toy example of \Cref{sec:cost_realistic}.}
	\label{fig:cost_realistic_percentage_loss}
\end{figure}

The optimal policy and the best balanced policy are shown in \Cref{fig:cost_realistic_policies}.
The optimal policy and the best balanced policy have approximately the same set of recurrent states, but the admission probabilities inside this set are different. The optimal balanced policy is constrained by the admission space, i.e., there is a single deterministic balanced policy with a particuliar set of recurrent states. On the contrary, without the balance condition, there are many deterministic policies with a particular set of recurrent states. In this example, having approximately the same set of recurrent states seems to be enough to yield the good performance, as we have
\begin{align*}
	\textrm{Loss}
	&= \frac{G^{\gamma^*} - G^{\delta \Gamma^*}}{G^{\gamma^*} - G^{\gamma_*}}
	\simeq 0.26\%.
\end{align*}
\Cref{fig:cost_realistic_percentage_loss} shows the evolution of the loss when the arrival rates of the two classes vary. The loss is (approximately) zero when the arrival rate of class~2 is 1. Looking closer at the case where the arrival rate of class~1 is~0.5, we see that both the optimal policy and the best balanced policy always reject class~1 customers and only accept class~2. The admission thresholds of class~2 differ slightly: it is one value larger for the optimal policy than for the best balanced policy; we attribute this to rounding errors, as the difference in performance is of order $10^{-8}$. The worst case on \Cref{fig:cost_realistic_percentage_loss} is when both classes have arrival rate~0.5. In this case, the discussion is similar to that we had above: the optimal policy and the best balanced policy have approximately the same set of recurrent states, but within this set they sometimes make different decisions.

On a more general note, a systematic comparison of the performance of policies leading to (quasi-)reversible systems, to the optimal policy (obtained e.g. by solving an MDP), is an open research question. This problem appears to heavily depend on the model at hand, and on the optimality criterion. See for instance the discussion in \cite[Section~2]{Mattquesta22} (and references therein), in which a list of systems is mentioned for which both types of situations may arise: (i) balanced policies behave poorly with respect to the optimal policy, or (ii) their performance are close to the optimum.

\section{Case study} \label{sec:numerics}

To illustrate the results of this paper, we consider a
multi-class multi-server redundancy system with customer abandonment.
We will see that this queueing system can be described as an \gls{OI} queue
and is therefore quasi-reversible.
After describing the queueing system and the admission control problem,
we propose a gradient-based method to search numerically for the best policy
within a parametric family of balanced policies.
This method is inspired by~\cite{CJSS24}
and leverages the gradient result of \Cref{theo:gradient}.
The numerical results we obtain in this case study
illustrate two aspects of our work:
(i) we propose a stochastic gradient ascent algorithm that,
by exploiting the gradient expression of \Cref{theo:gradient},
converges faster than two classical \gls{RL} algorithms
in this example, and
(ii) we compare the performance of
different families of balance functions,
going from static to dynamic.

\subsection{Admission control problem} \label{sec:case_study_problem}

Consider a queueing system with $n$ customer classes and $m$ servers,
where $n, m \in \bN_{> 0}$.
For each $i \in \llbracket 1, n \rrbracket$,
class-$i$ customers enter the system
according to a Poisson process
with rate $\nu_i > 0$,
and each class-$i$ customer in the system
abandons at rate $\zeta_i > 0$.
Furthermore, for each $j \in \llbracket 1, m \rrbracket$,
the completion time of a customer at server~$j$
is exponentially distributed with rate $\mu_j > 0$.
We let $\nu = (\nu_1, \nu_2, \ldots, \nu_n)$,
$\zeta = (\zeta_1, \zeta_2, \ldots, \zeta_n)$,
and $\mu = (\mu_1, \mu_2, \ldots, \mu_m)$.
Compatibility between customers and servers
is described by a bipartite graph
with adjacency matrix
$B = (b_{i, j})_{i \in \llbracket 1, n \rrbracket \times j \in \llbracket 1, m \rrbracket}$,
such that $b_{i, j} = 1$ if class-$i$ customers can be processed by server~$j$,
and $b_{i, j} = 0$ otherwise.
An example compatiblity graph is shown in \Cref{fig:case_study_graph}.

\begin{figure}[ht]
	\centering
	\begin{tikzpicture}
		\foreach \i in {1, 2, 3} {
			\node (class\i) at (2.5*\i, 0) {Class~\i};
			\node (server\i) at (2.5*\i, -1.5) {Server~\i};
		}
		
		\foreach \i/\j in {1/1, 1/2, 2/2, 2/3, 3/3} {
			\draw[-] (class\i) -- (server\j);
		}
		
		\node at (11, -0.75) {%
			$\displaystyle
			B = \left[ \begin{matrix}
				1 & 1 & 0 \\
				0 & 1 & 1 \\
				0 & 0 & 1
			\end{matrix} \right]
			$
		};
	\end{tikzpicture}
	\caption{A compatiblity graph and its adjacency matrix.}
	\label{fig:case_study_graph}
\end{figure}

We assume that the scheduling discipline is redundancy cancel-on-complete,
and that each server processes its compatible jobs in first-come-first-served order.
In this way, at each time, each customer has a replica in service
on each server that is compatible with this customer class
but not with other customers that arrived earlier.
We assume that the service times of the replicas of a customer
are independent of each other
and of the service times of other customers' replicas.

Under the appropriate independence assumptions,
one can verify that this system is an order-independent queue,
as described in \Cref{sec:oi},
with rate function given by
\begin{align*}
	\mu(s_1, s_2, \ldots, s_n)
	&= \sum_{\substack{j = 1 \\ (|s|B)_j > 0}}^m \mu_j
	+ \sum_{i = 1}^n \zeta_i |s|_i,
	\quad \text{for each } s_1s_2\ldots s_\ell \in \cS,
\end{align*}
where $\cS = \llbracket 1, n \rrbracket^*$
and, for each $j \in \llbracket 1, m \rrbracket$,
$(|s|B)_j = \sum_{i = 1}^n |s|_i b_{i, j}$
is the number of customers in the system
that are compatible with server~$j$.
In particular, the stationary distribution
of the \gls{CTMC} defined by the system's microstate
is given by~\eqref{prop:oi}.
In the special case where $\zeta = 0$ (i.e., no abandonment),
we obtain the queueing system introduced in \cite{G16}
and described as an \gls{OI} queue in \cite{BC17}.
Under appropriate stability conditions,
closed-form expressions for the normalizing constant
as well as various performance metrics
have been derived in~\cite{G17-1,G17-2,BCM17}.
In the remainder, we focus instead on the case where $\zeta_i > 0$ for each $i \in \llbracket 1, n \rrbracket$.
Stability is guaranteed by the abandonments, irrespective of the relative values of $\lambda$ and $\mu$,
but deriving a closed-form expression for the normalizing constant is an open question that seems challenging.

Assume a reward $r_i > 0$ is received
every time a class-$i$ customer completes service,
for each $i \in \llbracket 1, n \rrbracket$.
If a customer is rejected or abandons before service completion,
no reward is received.
Our goal is to maximize the long-run average reward.
This problem can written as a special case of
the admission control problem~\eqref{eq:problem}
(albeit with an infinite state space) as follows.
For each $s = s_1s_2\cdots s_\ell \in \cS$
and $i \in \llbracket 1, n \rrbracket$ such that $|s|_i > 0$,
we define,
with $p = \argmin\{q \in \llbracket 1, \ell \rrbracket; s_q = i\}$,
and $t = s_1 s_2 \cdots s_{p-1} s_{p+1} \cdots s_\ell$,
\begin{align*}
	\rdisc(s, t) &= r_i \gamma, \quad \text{where }
	\gamma
	= \frac
	{\sum_{j \in \mathcal{J}} \mu_j}
	{\sum_{j \in \mathcal{J}} \mu_j + \zeta_i (d - p + 1)}, \\
	\mathcal{J}
	&= \{
	j \in \llbracket 1, m \rrbracket;
	b_{i, j} = 1
	\text{ and } b_{s_q, j} = 0 \text{ for each } q \in \llbracket 1, p-1 \rrbracket
	\}, \text{ and} \\
	d &= \max\left\{ q \in \llbracket p, \ell \rrbracket; c_r = i \text{ for each } r \in \llbracket p, q \rrbracket \right\}.
\end{align*}
Here, $\gamma$ is the probability that the departure of the class-$i$ customer
in the transition from state~$s$ to state~$t$ is due to a service completion,
$\mathcal{J}$ is the set of servers
that are processing a class-$i$ customer in state~$s$,
and $d$ is the latest position in the sequence of class-$i$ customers
that follow the oldest one.
Other rewards are zero.

Finding an optimal policy for this problem
is made complicated by the fact that the state space is infinite.
More specifically, computing the optimal policy is unfeasible
(e.g., using the linear programming approach
mentioned in \Cref{sec:admission_control_problem_and_policies},
or dynamic programming),
not only because the state space is infinite, but also
because we do not have a closed-form expression for long-run indicators
such as a the normalizing constant of the stationary distribution
and the state value function.

Also because the state space is infinite,
the results of \Cref{sec:optimal_balanced_policies}
cannot be applied directly,
though we conjecture that the best balanced policy
is again deterministic.
In the remainder of this section, we explain how to search
for the best balanced policy, by adapting
a \gls{PGRL} algorithm introduced in \cite{CJSS24}.
We will see that, while classical \gls{RL} algorithms without function approximation
again fall short in this problem,
the \gls{PGRL} algorithm we describe below exploits the form of the stationary distribution to provide an estimator of the gradient that works without function approximation, even when the state space is infinite.

\subsection{Reinforcement learning} \label{sec:case_study_rl}

Let us cast this admission control problem
into the framework of \gls{RL}.
Though this paragraph was written to be self-contained,
the interested reader can refer to
\cite{SB18} for a more complete introduction to \gls{RL}.
Since the only decision points are customer arrival times,
which are distributed as a Poisson process with rate $\sum_{i = 1}^n \nu_i$,
we can reformulate the problem as a (discrete-time)
\gls{MDP} by focusing on the state observed at arrival times:
\begin{itemize}
	\item State space $\cS \times \llbracket 1, n \rrbracket$:
	The \gls{MDP} is in state $(s, i) \in \cS \times \llbracket 1, n \rrbracket$
	if the state observed by the incoming customer is~$s$
	and if this customer belongs to class~$i$.
	\item Action space $\cA = \{0, 1\}$:
	Action~$0$ means that the incoming customer is rejected
	and action~$1$, that the customer is admitted.
	\item Reward space $\cR = \{\sum_{i = 1}^n x_i r_i\,;\, x \in \bN^n\}$:
	The reward is $\sum_{i = 1}^n x_i r_i$
	if, for each $j \in \llbracket 1, n \rrbracket$,
	$x_j$ is the number of class-$j$ customers that leave
	\emph{due to service completion}
	between the (class-$i$) arrival for which we make a decision
	and the next arrival.
\end{itemize}

According to the PASTA property~\cite{W82},
the stationary distribution of the \gls{DTMC}
of the system state observed at arrival times of customers (admitted or not)
is again given by~$\Pi_\theta$.
Therefore, the gain function can be rewritten as
\begin{align} \label{eq:pgrl_gain}
	G_\theta
	&= \sum_{s, i, a, r, s'} \Pi_\theta(s) \frac{\nu_i}{\sum_j \nu_j}
	\pi_\theta(s, i, a) p(r, s' | s, i, a) r,
	\quad \text{for each } \theta \in \bR^d,
\end{align}
where the sum is over $s, s' \in \cS$, $i \in \llbracket 1, n \rrbracket$,
$a \in \cA$, and $r \in \cR$.
Here, we let $\pi_\theta$ denote the policy in the sense of \gls{RL}, given by 
$\pi_\theta(s, i, a)$ is the probability
of taking action $a$ in state $(s, i)$, that is,
$$\pi_\theta(s, i, a)
= \gamma_{\theta, i}(s) \indicator{a = 1}
+ (1 - \gamma_{\theta, i}(s)) \indicator{a = 0},$$
where $\gamma_{\theta, i}(s)$ is the probability of admitting a class-$i$ customer in state~$s$, as before. Furthermove,
we let $p$ denote the transition kernel of the \gls{MDP}. 
We do not explicitate $p$ since its form will not play a role in the remainder
(we will only need to sample from this kernel, which is done by simulating the process).

We now describe three \gls{RL} algorithms that we will compare numerically in \Cref{sec:case_study_num}: \gls{SAGE}, \gls{AC}, and Q-learning. \gls{SAGE} exploits the product-form nature of the model, while \gls{AC} and Q-learning are classical model-free \gls{RL} algorithms.

\subsubsection{Score-aware gradient estimator (SAGE)} \label{sec:sage}

\Acrfullpl{SAGE} and the associated \gls{SAGE}-based policy-gradient algorithm (\gls{SAGE} for short) were introduced in \cite{CJSS24}. While the original paper focuses on stationary distributions that form an exponential family, here we consider an extension to accommodate stationary distributions of the form given in \Cref{theo:balanced_policies_preserve_quasi_reversibility}.

\begin{algorithm}
	\caption{\Acrfull{SAGE}}
	\label{algo:sage}
	\begin{algorithmic}[1]
		\State \textbf{Input:}
		\begin{minipage}[t]{.8\textwidth}
			Positive, differentiable
			policy parametrization~$\pi_\theta$
		\end{minipage}
		\vspace{0.25em}
		
		\State \textbf{Parameters:}
		\begin{minipage}[t]{.8\textwidth}
			$\bullet$
			Update times $0 \triangleq t_0 < t_1 < t_2 < \ldots$ \\
			$\bullet$
			Step size $\alpha_\Theta \in (0, 1)$ \\
			$\bullet$
			Initial policy parameter $\Theta_0 \in \Omega$ \\
			$\bullet$
			Balance function parametrization $\Gamma_\theta$
		\end{minipage}
		\vspace{0.25em}
		
		\State \textbf{Initialization:}
		\begin{minipage}[t]{.7\textwidth}
			State $(S_0, I_0) \in \cS \times \llbracket 1, n \rrbracket$
		\end{minipage}
		\vspace{0.25em}
		
		\State \textbf{Main loop:}
		\For{$m = 0, 1, 2, \ldots$}
		\For{$t = t_m, \ldots, t_{m+1} - 1$} \label{step:batch-start}
		\State Sample $A_t \sim \pi_{\Theta_m}(S_t, I_t, \cdot)$
		\State Take action $A_t$ and observe $R_{t+1}, S_{t+1}, I_{t+1}$
		\EndFor \label{step:batch-end}
		\State Update $\Theta_{m+1} \gets \Theta_m + \alpha_\Theta \textsc{GradientEstimate}(m)$ \label{step:sga}
		\EndFor
		\vspace{0.25em}
		
		\Procedure{GradientEstimate}{$m$}
		\State $N \gets t_{m+1} - t_m$
		\State $\overline{R} \gets \frac1N \sum_{t = t_m}^{t_{m+1} - 1} R_{t+1}$
		\State $\overline{C} \gets \frac1{N - 1} \sum_{t = t_m}^{t_{m+1} - 1} (R_{t+1} - \overline{R}) \nabla_\theta \log \Gamma_{\Theta_m}(S_t)$ \label{step:C}
		\State $\overline{E} \gets \frac1N \sum_{t = t_m}^{t_{m+1} - 1} R_{t+1} \nabla_\theta \log \pi_{\Theta_m}(S_t, I_t, A_t)$ \label{step:E}
		\EndProcedure
	\end{algorithmic}
\end{algorithm}

\Gls{SAGE} is a \gls{PGRL} algorithm~\cite[Chapter~13]{SB18}, which means that it optimizes the policy parameter~$\theta \in \bR^d$ via a stochastic gradient ascent, as shown on Line~\ref{step:sga} of \Cref{algo:sage}. Here, $\alpha_\Theta \in (0, 1)$ is the step size; the policy parameter is written in capitals ($\Theta$ instead of $\theta$) because it is a random variable, and an index~$m$ because it is updated over time. As the name suggests, the \textsc{GradientEstimate}($m$) procedure computes an estimate of the gradient, using the trajectory observed between times~$t_m$ and~$t_{m+1}$, sampled on Lines~\ref{step:batch-start}--\ref{step:batch-end}. Below, we justify the form of this procedure.

Consider a balance function parameterization $\theta \in \bR^d \mapsto \Gamma_\theta \in \cB^{\text{ac}}_{|\cS|}$
as described in \Cref{sec:gradient_of_parameterized_systems},
where $d \in \bN_{> 0}$.
\Cref{prop:pgrl_derivative} below
builds on \Cref{theo:gradient}
to derive an expression for the gradient $\nabla_\theta G_\theta$,
which is next used to justify the form
of the procedure \textsc{GradientEstimate} in \Cref{algo:sage}.
This result can be seen as an extension of \cite[Equation~(13) in Theorem~1]{CJSS24}.
For each $\theta \in \bR^d$, we let $(S_\theta, I_\theta, A_\theta, R_\theta)$
denote a stationary 4-tuple under policy~$\pi_\theta$, that is, having the following distribution:
for each $s \in \cS$, $i \in \llbracket 1, n \rrbracket$, $a \in \cA$, and $r \in \cR$,
\begin{align}
	\mathbb{P}_\theta(s,i,a,r) &=\mathbb{P}(S_\theta = s, I_\theta = i, A_\theta = a, R_\theta = r)\notag\\
	&= \Pi_\theta(s) \frac{\nu_i}{\sum_j \nu_j} \pi_\theta(s, i, a) \sum_{s' \in \cS} p(r, s' | s, i, a).\label{eq:distribsiar}
\end{align}
For each $k \in \llbracket 1, d \rrbracket$,
we let $\partial_k$ denote the partial derivative with respect to $\theta_k$.

\begin{proposition} \label{prop:pgrl_derivative}
	For each $\theta \in \bR^d$, we have
	\begin{align} \label{eq:pgrl_derivative}
		\partial_k G_\theta
		&= \mbox{\text{Cov}}(R_\theta, \partial_k \log \Gamma_\theta(S_\theta))
		+ \bE{R \, \partial_k \log \pi_\theta(S_\theta, I_\theta, A_\theta)}.
	\end{align}
\end{proposition}

\begin{proof}
	Let $k \in \llbracket 1, d \rrbracket$ and $\theta \in \bR^d$.
	Observe that we have for all $s \in \cS$, $i \in \llbracket 1, n \rrbracket$, $a \in \{0, 1\}$,  
	\begin{align*}
		\partial_k\left(\Pi_\theta(s)\pi_\theta(s, i, a)\right) &=\Pi_\theta(s)\pi_\theta(s, i, a) \partial_k\log
		\left(\Pi_\theta(s)\pi_\theta(s, i, a)\right),\\
		&=\Pi_\theta(s)\pi_\theta(s, i, a)\left(\partial_k \log \Pi_\theta(s)+\partial_k\log \pi_\theta(s, i, a)\right).
	\end{align*}
	Thus, 
	taking the partial derivative of~\eqref{eq:pgrl_gain} yields
	\begin{align*}
		\partial_k G_\theta
		&= \sum_{s, i, a, r, s'} \Pi_\theta(s) \frac{\nu_i}{\sum_j \nu_j}
		\pi_\theta(s, i, a) p(r, s' | s, i, a) r
		\left[\partial_k \log \Pi_\theta(s) + \partial_k \log \pi_\theta(s, i, a) \right].
	\end{align*}
	Therefore, by injecting \Cref{theo:gradient} and recalling \eqref{eq:distribsiar}, we obtain
	\begin{multline*}
		\partial_k G_\theta
		= \sum_{s, i, a, r}
		\mathbb{P}_\theta(s,i,a,r)
		r\left[ \partial_k \log \Gamma_\theta(s)
		- \bE{\partial_k \log \Gamma_\theta(S_\theta)}
		+ \partial_k \log \pi_\theta(s, i, a) \right],
	\end{multline*}
	which is equivalent to~\eqref{eq:pgrl_derivative}.
\end{proof}

The mapping between \Cref{algo:sage} and \Cref{prop:pgrl_derivative} is immediate:
the $k$-th component of $\overline{C}$ on Line~\ref{step:C} is an estimate of
the covariance $\mbox{\text{Cov}}(R_\theta, \partial_k \log \Gamma_\theta(S_\theta))$,
and the $k$-th component of $\overline{E}$ on Line~\ref{step:E} is an estimate of
$\bE{R \, \partial_k \log \pi_\theta(S_\theta, I_\theta, A_\theta)}$.
These estimates are biased in general because, even if the policy is fixed
between times~$t_m$ and~$t_{m+1}$,
the Markov chain has not (yet) reached its stationary distribution.

\subsubsection{Actor-critic (AC)}

\Acrfull{AC} is a classical \gls{RL} algorithm,
that also belongs to the family of \gls{PGRL} algorithms.
The pseudocode also appears in \Cref{algo:ac} for completeness,
but a more detailed description can be found in \cite[Section~13.6]{SB18}.
Like under \gls{SAGE}, the policy parameter is updated
via a stochastic gradient ascent step shown on Line~\ref{step:ac-sga}.
The two main differences with \gls{SAGE} are as follows.
First, the gradient is updated at every time step,
instead of being updated at pre-defined update times,
so that the step size with be changed accordingly.
Second, the \textsc{GradientEstimate} procedure is defined differently:
it is based on another estimator that does not
exploit the product-form nature of the stationary distribution,
and instead relies on the state value function~$v$,
which gives the long-run advantage of a particular state
compared to a typical state under the stationary distribution.

\begin{algorithm}
	\caption{Actor-critic}
	\label{algo:ac}
	\begin{algorithmic}[1]
		\State \textbf{Input:}
		\begin{minipage}[t]{.8\textwidth}
			Positive, $\theta$-differentiable
			policy parametrization~$\pi_\theta$
		\end{minipage}
		\vspace{0.25em}
		
		\State \textbf{Parameters:}
		\begin{minipage}[t]{.8\textwidth}
			$\bullet$
			Step sizes $\alpha_\Theta, \alpha_{\overline{R}}, \alpha_v \in (0, 1)$ \\
			$\bullet$
			Initial policy parameter $\Theta_0 \in \Omega$
		\end{minipage}
		\vspace{0.25em}
		
		\State \textbf{Initialization:}
		\begin{minipage}[t]{.7\textwidth}
			$\bullet$
			State $(S_0, I_0) \in \cS \times \llbracket 1, n \rrbracket$ \\
			$\bullet$
			Average reward $\overline{R} \gets 0$ \\
			$\bullet$
			Value function $v[s, i] \gets 0$ for each $s \in \cS, i \in \llbracket 1, n \rrbracket$
		\end{minipage}
		\vspace{0.25em}
		
		\State \textbf{Main loop:}
		\For{$t = 0, 1, 2, \ldots$}
		\State Sample $A_t \sim \pi_{\Theta_m}(S_t, I_t, \cdot)$
		\State Take action $A_t$ and observe $R_{t+1}, S_{t+1}, I_{t+1}$
		\State Update $\Theta_{t+1} \gets \Theta_t + \alpha_{\Theta} \textsc{GradientEstimate}(t)$ \label{step:ac-sga}
		\EndFor
		\vspace{0.25em}
		
		\Procedure{GradientEstimate}{$t$}
		\State $\delta \gets R_{t+1} - \overline{R} + v[S_{t+1}, I_{t+1}] - v[S_t, I_t]$
		\State Update $\overline{R} \gets \overline{R} + \alpha_{\overline{R}} \delta$
		\State Update $v[S_t, I_t] \gets v[S_t, I_t] + \alpha_v \delta$
		\State \textbf{return}
		$\delta \, \nabla \log \pi_{\Theta_t}(S_t, I_t, A_t)$
		\EndProcedure
	\end{algorithmic}
\end{algorithm}

\subsubsection{Q-learning}

Lastly, we consider the celebrated
Q-learning algorithm introduced in~\cite{W89}.
While \cite[Section~6.5]{SB18} provides
an intuitive introduction to Q-learning
for \glspl{MDP} with discounted rewards,
here we consider the average-reward variant
described in \Cref{algo:q} below.
The basic idea behind Q-learning is to learn
a so-called state-action value function,
denoted by~$q$ here (Line~\ref{step:q-value}),
that gives for each state-action pair,
the long-run advantage of taking this action in this state,
compared to the average reward under the best policy.
If we knew the $q$-function perfectly,
the optimal policy would consist in systematically
taking the action with the highest advantage.
Here, since we have uncertainty on the $q$-function,
we only follow the action given by the estimated $q$-function with some probability (Line~\ref{eq:q-case2}),
otherwise we sample an action uniformly at random (Line~\ref{eq:q-case1}),
which guarantees that we keep exploring even if we misestimate the $q$-function.

\begin{algorithm}
	\caption{Q-learning algorithm}
	\label{algo:q}
	
	\begin{algorithmic}[1]
		\State \textbf{Parameters:}
		\begin{minipage}[t]{.8\textwidth}
			$\bullet$
			Update times $0 \triangleq t_0 < t_1 < t_2 < \ldots$ \\
			$\bullet$
			Step sizes $\alpha_{\overline{R}}, \alpha_q \in (0, 1)$ \\
			$\bullet$
			Exploration rates $\epsilon_0 \ge \epsilon_1 \ge \epsilon_2 \ge \ldots$
		\end{minipage}
		\vspace{0.25em}
		\State \textbf{Initialization:}
		\begin{minipage}[t]{.7\textwidth}
			$\bullet$
			State $(S_0, I_0) \in \cS \times \llbracket 1, n \rrbracket$ \\
			$\bullet$
			Average reward $\overline{R} \gets 0$ \\
			$\bullet$
			Value function $q[s, i, a] \gets 0$ for $s \in \cS$, $i \in \llbracket 1, n \rrbracket$, $a \in \{0, 1\}$
		\end{minipage}
		\vspace{0.25em}
		
		\State \textbf{Main loop:}
		\For{$m = 0, 1, 2, \ldots$}
		\For{$t = t_m, \ldots, t_{m+1} - 1$}
		\State Choose action~$A_t$ as follows:
		\State $\bullet$
		With probability $\epsilon_m$, sample $A_t \in \{0, 1\}$ uniformly at random \label{eq:q-case1}
		\State $\bullet$
		Otherwise, choose $A_t \in \argmax_{a \in \{0, 1\}} q(S_t, a)$ \label{eq:q-case2}
		\State Take action $A_t$ and observe $R_{t+1}, S_{t+1}, I_{t+1}$
		\State $\delta \gets R_{t+1} - \bar{R} + \max_{a \in \{0, 1\}} q(S_t, I_t, a) - q(S_t, I_t, A_t)$
		\State Update $\bar{R} \gets \bar{R} + \alpha_{\overline{R}} \delta$
		\State Update $q(S_t, I_t, A_t) \gets q(S_t, I_t, A_t) + \alpha_q \delta$ \label{step:q-value}
		\EndFor
		\EndFor
	\end{algorithmic}
\end{algorithm}

\subsection{Numerical results} \label{sec:case_study_num}

We compare numerically
the performance of the \gls{SAGE}, \gls{AC}, and Q-learning algorithms
described in \Cref{sec:case_study_rl},
in the admission control example of \Cref{sec:case_study_problem}.

\subsubsection{Evaluation setup}

To compare \gls{SAGE}, \gls{AC}, and Q-learning,
we ran each algorithm 10 times independently, for $10^7$ steps.
In the plots showing trajectories, the x-axis is in logarithmic scale,
the solid curves show the mean,
and the shaded areas show the standard deviation around the mean.
Furthermore, to keep a reasonable memory occupation, starting from time $t = 10^3$,
the results are reported only at times multiple of $10^{i-2}$
between $10^i$ and $10^{i+1}$
(e.g., between $10^3$ and $10^4$, results are reported at times
1000, 1010, 1020, 1030, etc.).

We evaluate \gls{SAGE} and \gls{AC} with the following policy parametrizations:
\begin{itemize}
	\item \textit{Static}: We consider the (balanced) static policy
	described in \Cref{sec:examples_of_balanced_arrival_rates}.
	The balance function is given by
	$\Gamma_\theta(x) = \prod_{i = 1}^n \sigma(\theta_i)^{x_i}$,
	where $\sigma: \bR \to (0, 1)$ is the logistic function,
	given by $\sigma(\phi) = 1 / (1 + \textrm{e}^{-\phi})$
	for each $\phi \in \bR$,
	and $\theta \in \bR^n$.
	\item \textit{Semi-static}:
	We consider the (balanced) semi-static policy
	described in \Cref{sec:examples_of_balanced_arrival_rates}.
	The balance function is given by
	$\Gamma_\theta(x)
	= \prod_{i = 1}^n \sigma(\theta_i)^{x_i }
	\prod_{j = 1}^n \sigma(\theta_{i, j})^{x_i x_j}$,
	$\theta \in \bR^{n + n^2}$.
	\item \textit{Dynamic}:
	We consider the cum-prod balanced policy described in \Cref{sec:examples_of_balanced_arrival_rates},
	with the balance function $\Gamma_\theta(x) = \prod_{y \in |\cS|: y \le x} \sigma(\theta_y)$.
	In principle, such a parametrization requires an infinite-dimensional parameter~$\theta$,
	as there are as many components as there are macrostates in the system.
	In practice, we initialize the parameter on a subset of the state space
	around the empty state,
	and we expand it as needed.
	\item \textit{Imbalanced}: Only for \gls{AC}, we consider
	an additional parametrization that we call \textit{imbalanced},
	because it does not satisfy the balance condition (\Cref{def:balanced_policies}).
	We define the class-$i$ admission probability in microstate~$s$ as
	$\gamma_{\theta, i}(s) = \sigma(\theta_{s, i})$.
	Like for the dynamic parametrization,
	this leads to an infinite-dimensional parameter vector,
	but we tackle this issue by only defining the components
	as we need them, with a predefined initial value.
\end{itemize}
Thus, we will say for instance ``\gls{SAGE} static''
to refer to the \gls{SAGE} algorithm applied with the static parametrization.
The use of the logistic function (also call sigmoid or softmax)
to transform a real parameter
into a number between~$0$ and~$1$ is standard in \gls{RL}~\cite{SB18}.
Under all parameterizations,
deterministic policies can be obtained asymptotically,
when the components of the parameter $\theta$ tend to $\pm \infty$, which
is common in the context of \gls{PGRL},
see for instance \cite[Section~13.1]{SB18}.

The remaining parameters are as follows.
In \gls{SAGE}, we use update times multiples of $10^2$,
that is, $t_{m+1} - t_m = 10^2$ for each $m \in \bN$,
and step size~$\alpha_\Theta = 10^{-1}$.
In \gls{AC}, we use step sizes $\alpha_\Theta = 10^{-3}$
(because we update the parameter $10^2$ times more frequently than in \gls{SAGE})
and $\alpha_{\overline{R}} = \alpha_v = 10^{-2}$.
Lastly, in Q-learning, we again use update times multiples of $10^2$,
step sizes $\alpha_{\overline{R}} = \alpha_q = 10^{-2}$,
and exploration rates defined recursively by
$\epsilon_0 = 10^{-1}$
and $\epsilon_{m+1} = \max(10^{-4}, \epsilon_m - 2 \cdot 10^{-5})$
for each $m \in \bN$
(so that the sequence reaches its final value $10^{-4}$ at $m = 10^6$).
Lastly, the arrays $v$ (in \gls{AC}) and $q$ (in Q-learning) are initialized to zero on a subset of the state(-action) space and are augmented as needed.

\subsubsection{Adversarial scenario} \label{sec:adversarial}

We focus on
the compatibility graph of \Cref{fig:case_study_graph},
with arrival rate vector $\nu = (0.5, 0.5, 0.5)$,
abandonment rate vector $\zeta = (0.1, 0.2, 0.5)$,
service rate vector $\mu = (0.5, 0.5, 0.5)$,
and reward rate vector $r = (1, 2, 16)$.
We call this scenario \textit{adversarial},
because the class with the largest abandonment rate (class~3)
and the smallest (maximal) service rate
also brings the largest reward.
In \Cref{sec:nonadversarial},
we will also consider a non-adversarial scenario,
where the abandonment rates of classes~1 and~3 are exchanged.
The numerical results for the adversarial scenario
are shown in \Cref{fig:reward,fig:policy_per_agent,fig:policy_per_class}.
In the discussion below, we will draw two conclusions:
(i) by exploiting the product form of the stationary distribution,
\gls{SAGE} converges faster than \gls{AC} and Q-learning, and
(ii) higher-dimensional policy-parametrizations
yield a better performance eventually,
at the cost of a (much) slower convergence.

\paragraph*{Running average reward}

\begin{figure}[t]
	\centering
	\subfloat[Trajectory\label{fig:reward_time}]
	{\includegraphics[width=.75\textwidth]{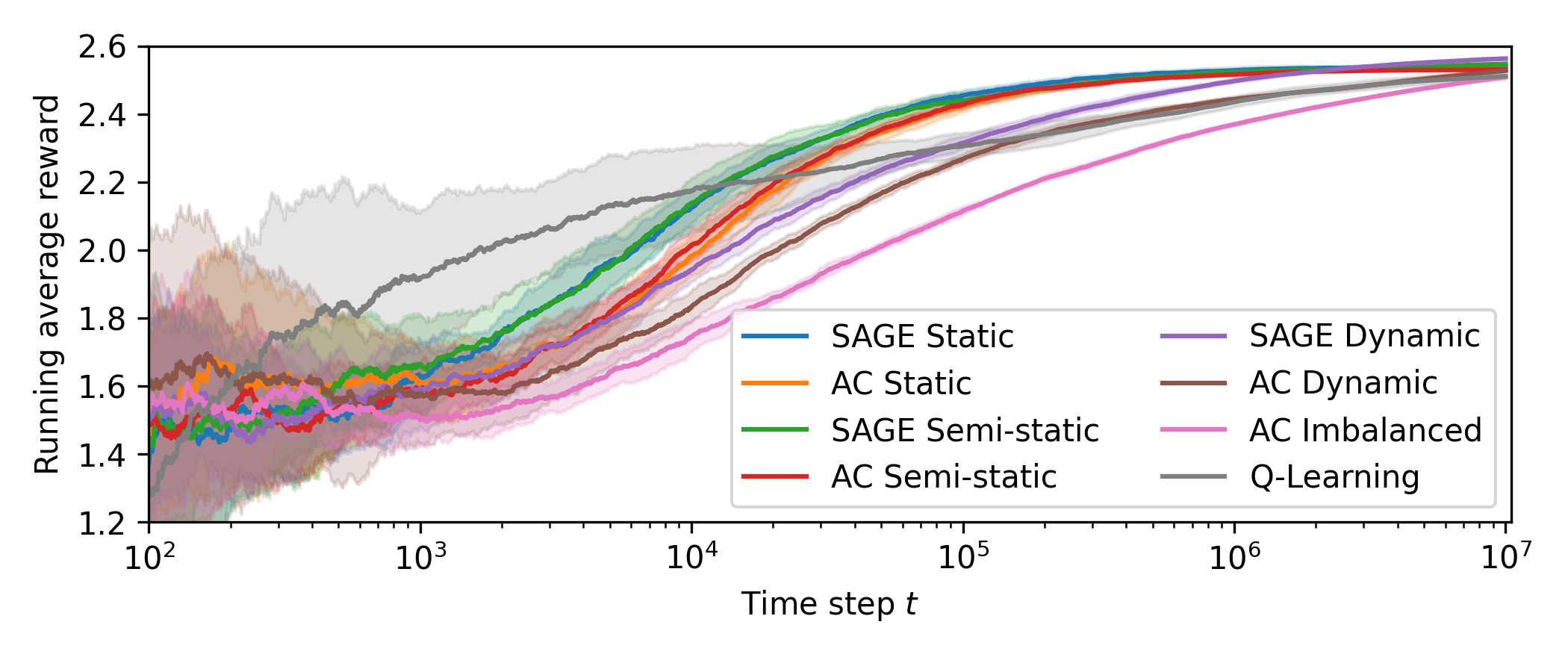}}
	\\
	\subfloat[Empirical distribution at time~$10^4$\label{fig:reward_time_10_power_4}]
	{\includegraphics[width=.48\textwidth]{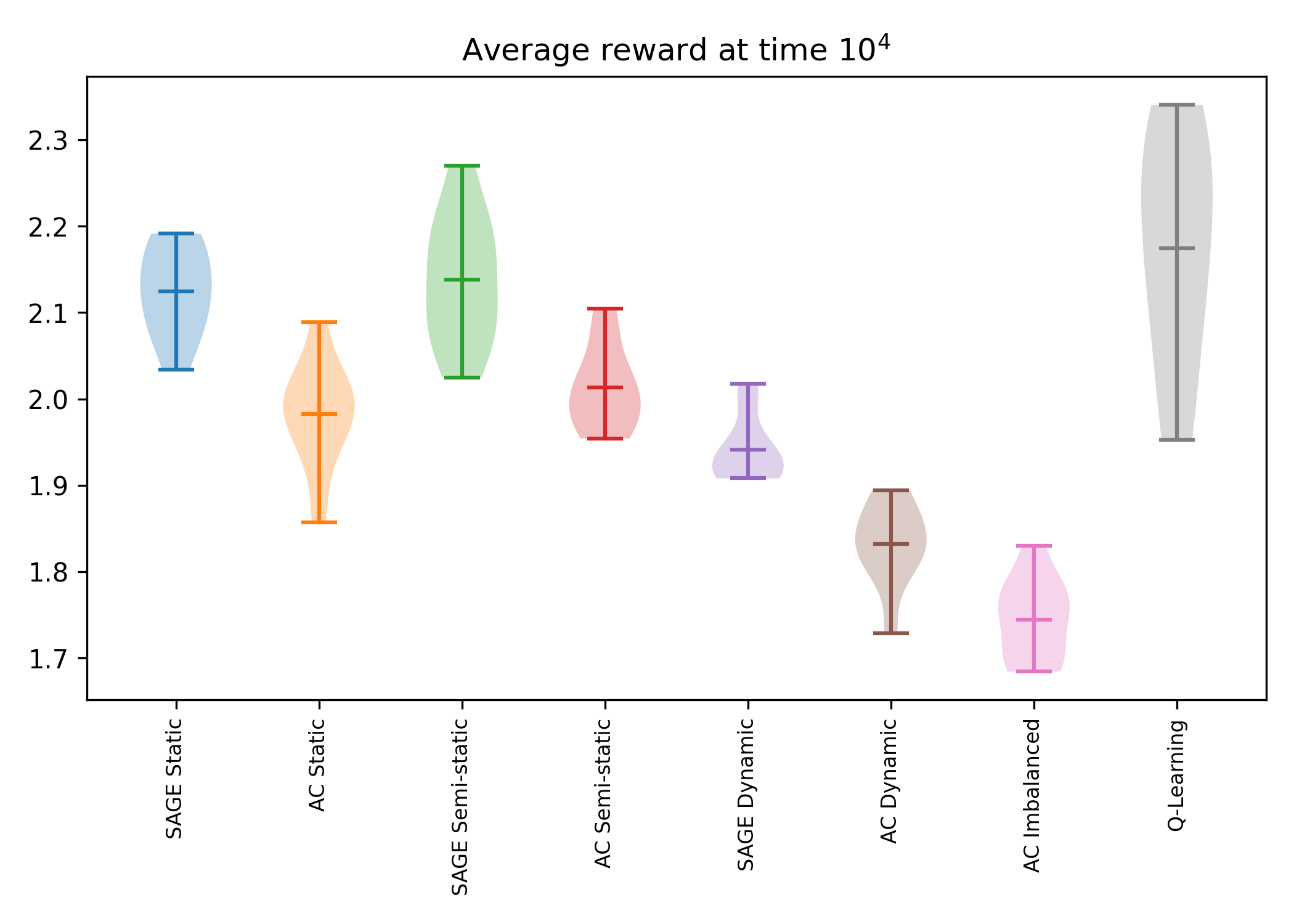}}
	\hfill
	\subfloat[Empirical distribution at time~$10^7$\label{fig:reward_time_10_power_7}]
	{\includegraphics[width=.48\textwidth]{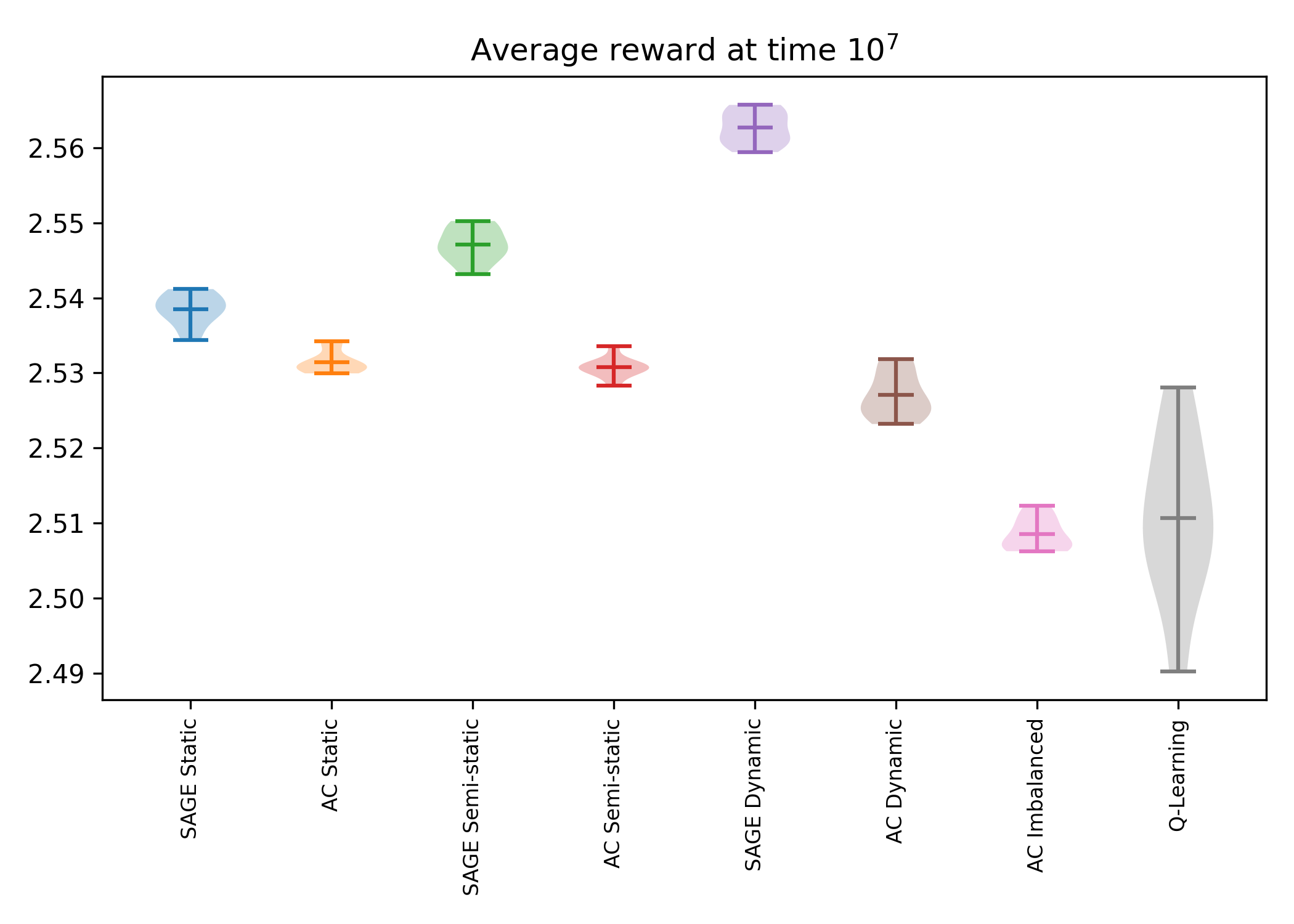}}
	\caption{Running average reward. In each violin plot, the three horizontal lines show the minimium, mean, and maximum value, respectively.}
	\label{fig:reward}
\end{figure}

Let us first focus on the running average reward.
\Cref{fig:reward_time} shows the trajectory over time, while \Cref{fig:reward_time_10_power_4,fig:reward_time_10_power_7}
give the empirical distribution at two particular times.
In \Cref{fig:reward_time}, we first observe that,
while Q-learning shows a rapid increase until roughly $10^4$ steps,
it is later caught up and overtaken by \gls{SAGE} and \gls{AC}.
This trend is confirmed by \Cref{fig:reward_time_10_power_4,fig:reward_time_10_power_7}.
Now focusing on \gls{SAGE} and \gls{AC}, we make several observations.
First, in \Cref{fig:reward_time}, we see that
\gls{SAGE} static and \gls{SAGE} semi-static have a similar evolution,
and so do \gls{AC} static and \gls{AC} semi-static.
\Cref{fig:reward_time_10_power_7} also reveals that,
at the end of the simulation,
\gls{SAGE} semi-static outperforms \gls{SAGE} static,
which outperforms \gls{AC} static,
which in turn outperforms \gls{AC} semi-static.
\gls{SAGE} dynamic, \gls{AC} dynamic, and \gls{AC} imbalanced
have a slower convergence than their static and semi-static counterparts,
but \gls{SAGE} dynamic eventually has the best performance.
\gls{AC} dynamic and \gls{AC} imbalanced have not yet caught up after~$10^7$ steps.

\paragraph*{Admission probability}

\begin{figure}[t]
	\includegraphics[width=\textwidth]{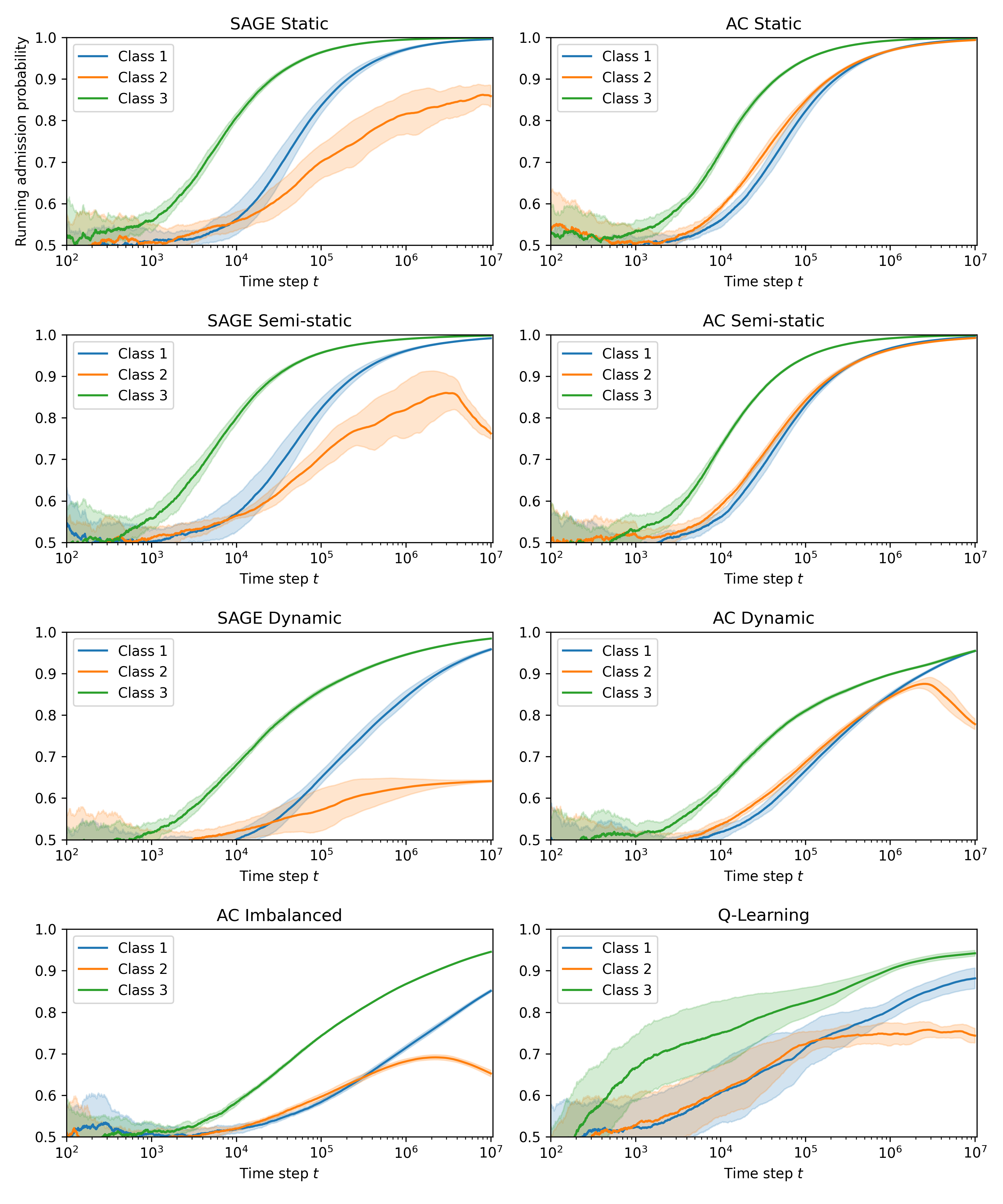}
	\caption{Evolution over time of the running average admission probability, per agent}
	\label{fig:policy_per_agent}
\end{figure}

\begin{figure}[t]
	\includegraphics[width=\textwidth]{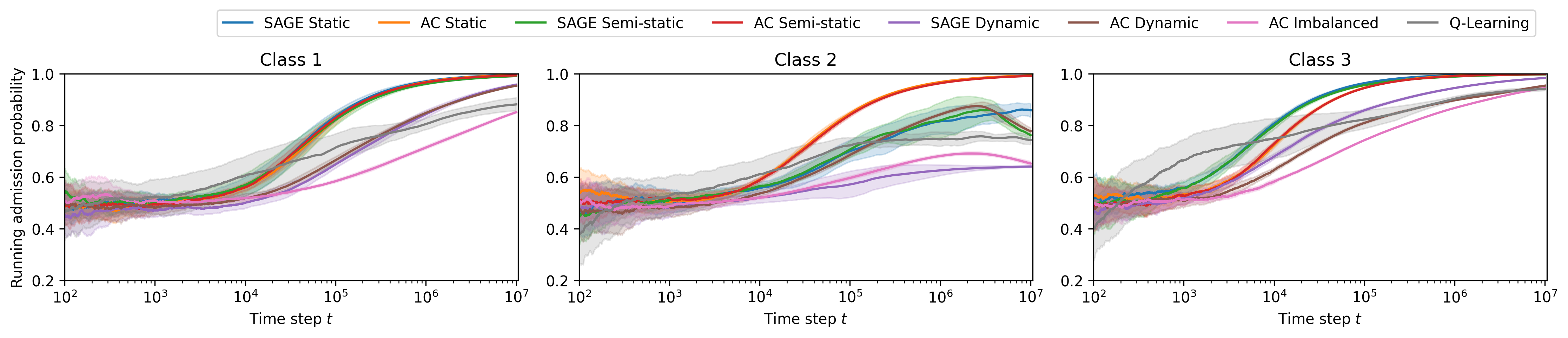}
	\caption{Evolution over time of the running average admission probability, per class}
	\label{fig:policy_per_class}
\end{figure}

\Cref{fig:policy_per_agent,fig:policy_per_class} both
show the evolution of the empirical admission probability,
one grouped by algorithm-parametrization pair, and the other by class.
First looking at \Cref{fig:policy_per_agent}, we observe that the dynamics differ significantly from one algorithm-parametrization pair to another. We observed earlier that \gls{SAGE} dynamic achieved the best performance at the end of the $10^7$ steps, and we see here that it does so by admitting almost all customers of classes~1 and~3, and by rejecting a positive fraction of class-2 customers. \gls{AC} static and \gls{AC} semi-static seem to converge towards a suboptimal policy (admit all customers, at least in states visited frequently), compared to \gls{SAGE} static and \gls{SAGE} semi-static. \gls{AC} dynamic and \gls{AC} imbalanced seem to converge towards a policy that is similar to \gls{SAGE} dynamic, but at a slower pace. Q-learning initially achieves good performance because it quickly increases the admission probability of class~3, but afterwards its convergence is slower.

\subsubsection{Non-adversarial scenario} \label{sec:nonadversarial}

For completeness, let us finally show numerical results
for a non-adversarial scenario.
The compatibility graph
is again that of \Cref{fig:case_study_graph},
with arrival rate vector $\nu = (0.5, 0.5, 0.5)$,
abandonment rate vector $\zeta = (0.5, 0.2, 0.1)$,
service rate vector $\mu = (0.5, 0.5, 0.5)$,
and reward rate vector $r = (1, 2, 16)$.
The abandonment rates of classes~1 and~3 are exchanged
compared to \Cref{sec:adversarial}.
Intuitively, we expect that this scenario is simpler,
because class~3, the high-reward class,
has a smaller abandonment rate.
The running average reward is shown in \Cref{fig:app_reward},
and the average admission probability
in \Cref{fig:app_policy_per_agent,fig:app_policy_per_class}.

\paragraph*{Running average reward}

\begin{figure}[htb]
	\centering
	\subfloat[Trajectory\label{fig:app_reward_time}]
	{\includegraphics[width=.75\textwidth]{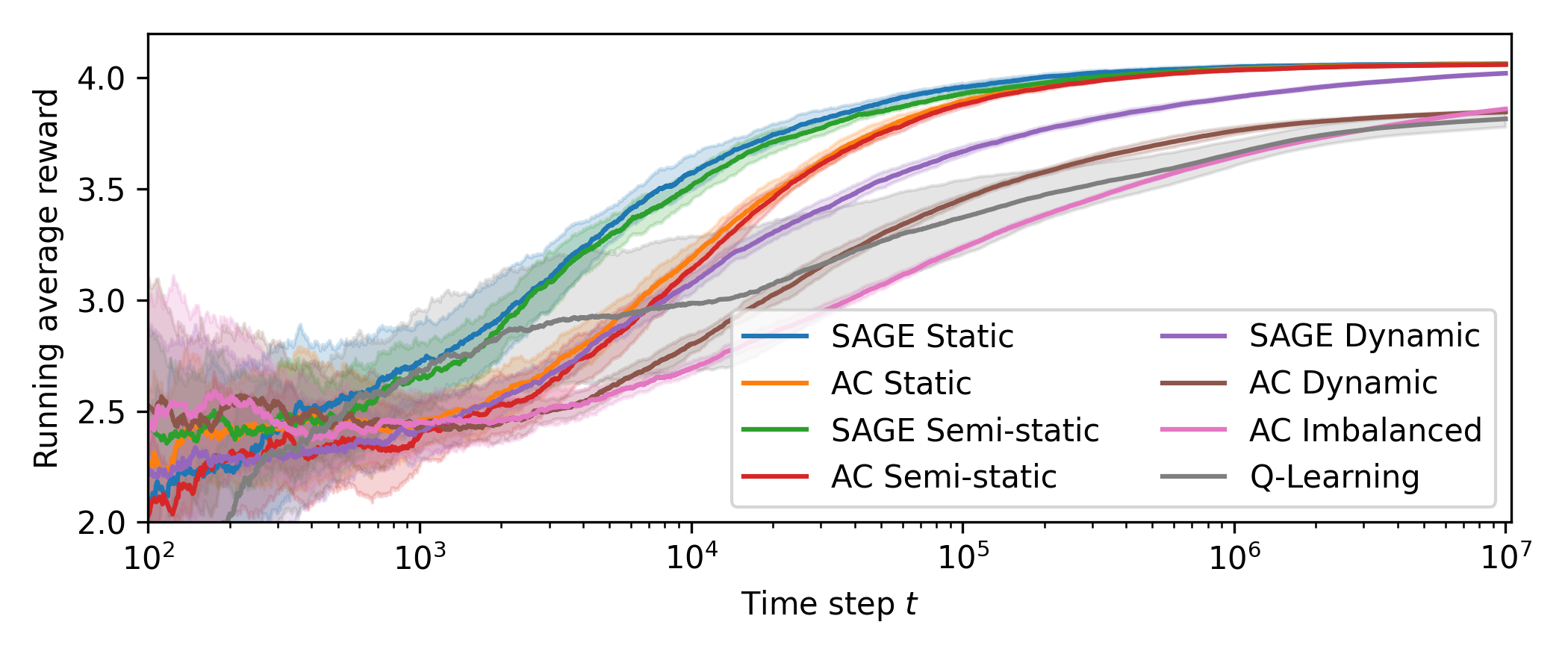}}
	\\
	\subfloat[Empirical distribution at time~$10^4$\label{fig:app_reward_time_10_power_4}]
	{\includegraphics[width=.48\textwidth]{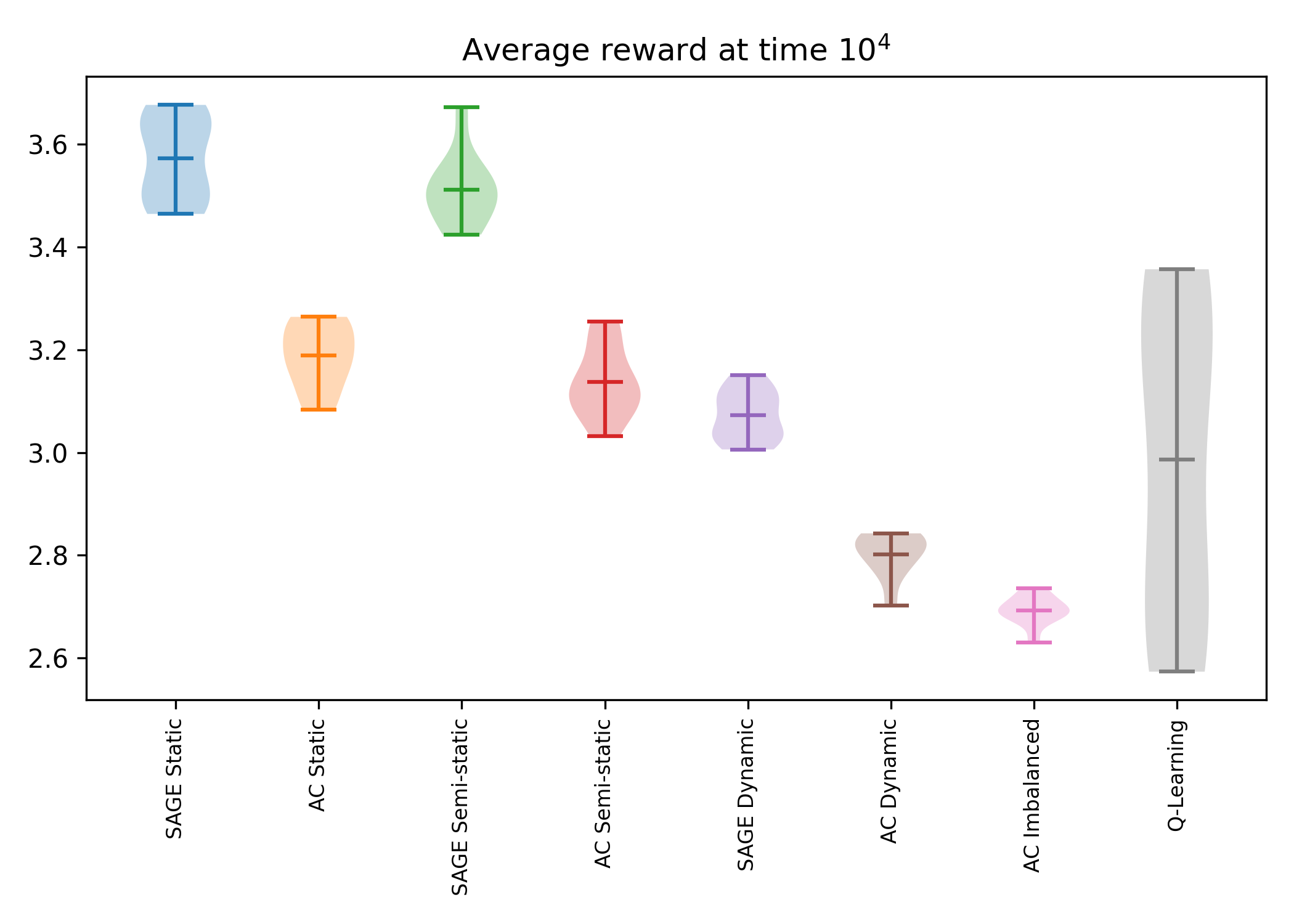}}
	\hfill
	\subfloat[Empirical distribution at time~$10^7$\label{fig:app_reward_time_10_power_7}]
	{\includegraphics[width=.48\textwidth]{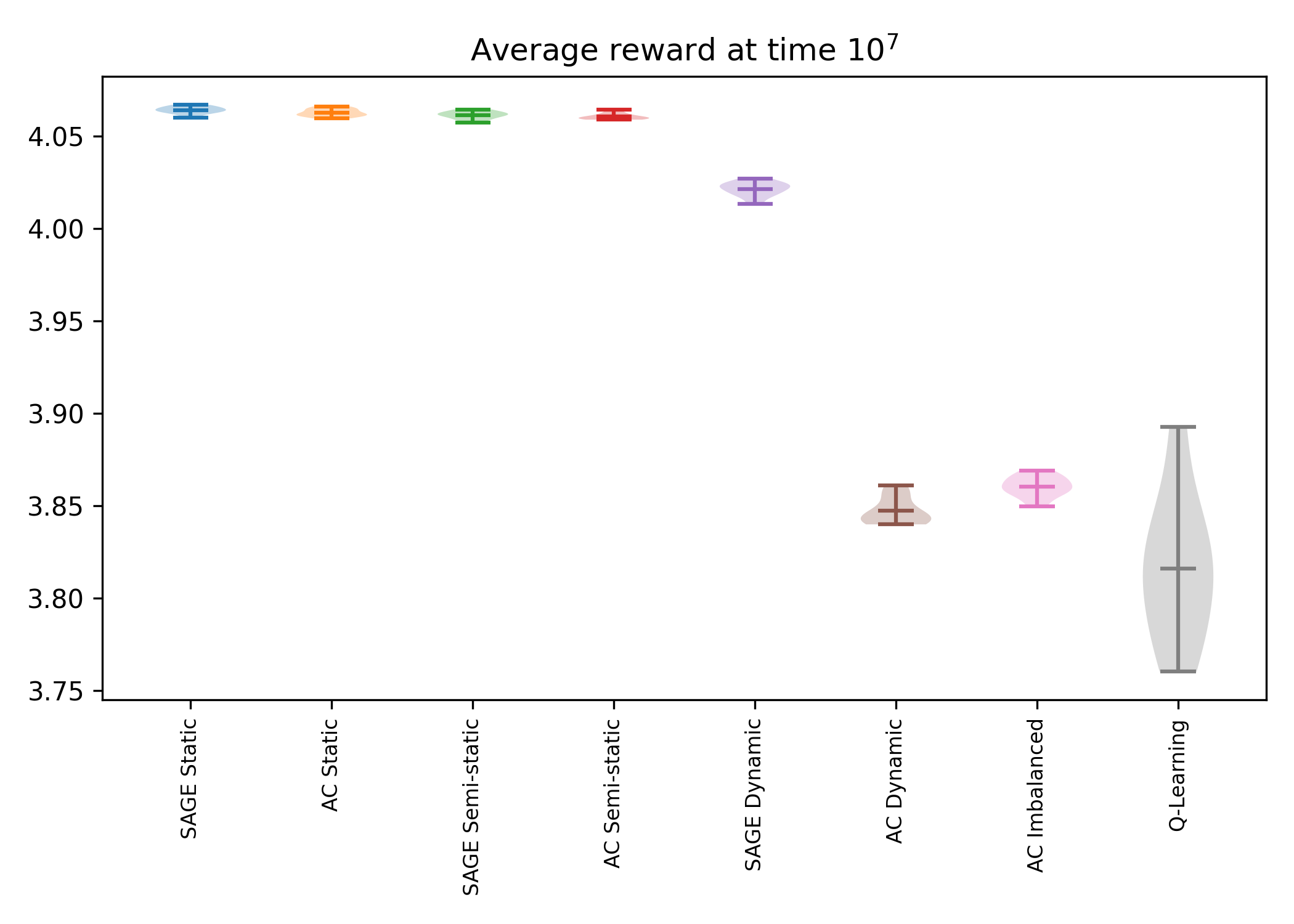}}
	\caption{Running average reward. In each violin plot, the three horizontal lines show the minimium, mean, and maximum value, respectively.}
	\label{fig:app_reward}
\end{figure}

Let us first focus on the running average reward.
\Cref{fig:app_reward_time} shows the trajectory over time,
while \Cref{fig:app_reward_time_10_power_4,fig:app_reward_time_10_power_7}
zoom in on two particular times.
Some observations of the adversarial scenario still hold.
In particular, \gls{SAGE} static and \gls{SAGE} semi-static
have a similar behavior,
\gls{AC} static and \gls{AC} semi-static also do,
and \gls{SAGE} static/semi-static converges faster than \gls{AC} static/semi-static.
Collectively, they converge much faster than
\gls{SAGE} dynamic,
which in turn converges much faster than
\gls{AC} dynamic, \gls{AC} imbalanced, and Q-learning.
The main difference with the adversarial scenario
is that \gls{AC} static/semi-static
approximately catches up with \gls{SAGE} static/semi-static
already at the end of the $10^7$ steps,
and Q-learning never overtakes the other algorithms, even in the early stage.

\paragraph*{Admission probability}

\begin{figure}[htb]
	\includegraphics[width=\textwidth]{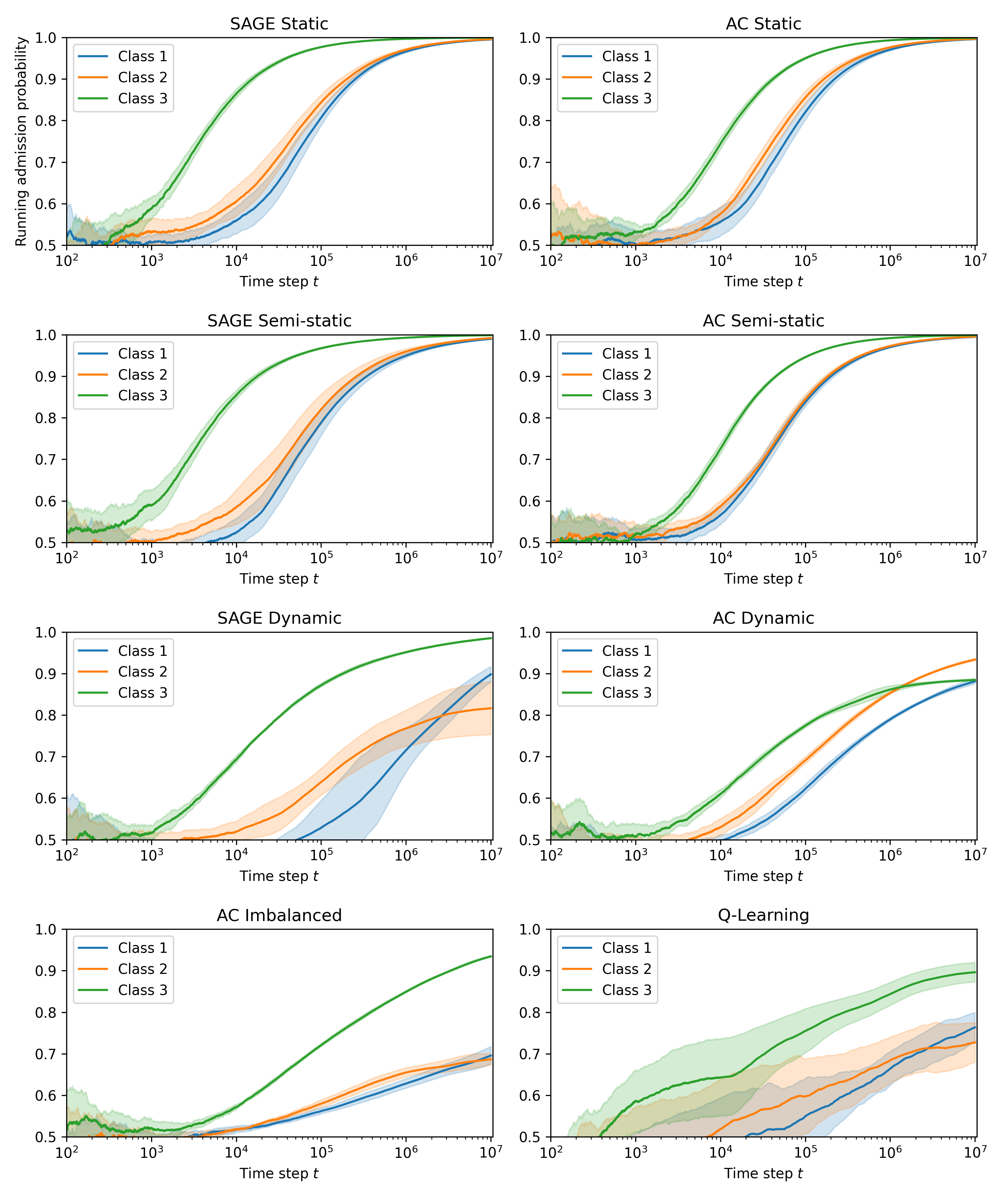}
	\caption{Evolution over time of the running average admission probability, per agent}
	\label{fig:app_policy_per_agent}
\end{figure}

\begin{figure}[htb]
	\includegraphics[width=\textwidth]{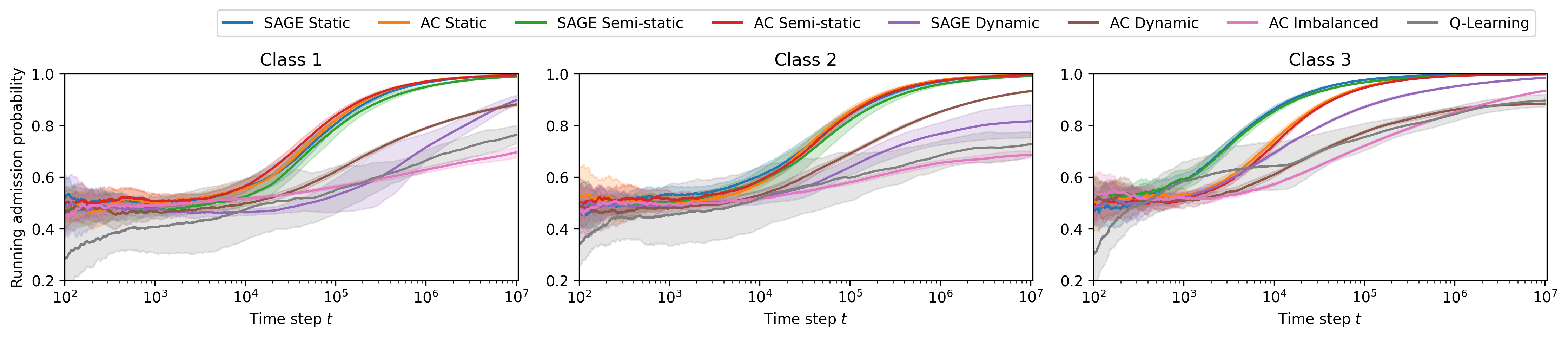}
	\caption{Evolution over time of the running average admission probability, per class}
	\label{fig:app_policy_per_class}
\end{figure}

Let us now consider the average admission probability
shown in \Cref{fig:app_policy_per_agent,fig:app_policy_per_class}.
The main difference with the adversarial scenario
is that \gls{SAGE} static/semi-static
also seems to converge towards the ``admit-everyone'' policy,
like \gls{AC} static/semi-static,
while in the adversarial scenario
\gls{SAGE} static/semi-static admitted
only a positive fraction of class-2 customers.%

\section{Conclusion}\label{sec:conclu}

In this paper, we have introduced a general definition of {\em quasi-reversibility} for queueing systems, that slightly differs from existing 
definitions in the literature. We have shown that the class of quasi-reversible systems embraces a wide range of 
systems of particular interest in the recent queueing literature: multi-class Whittle networks and Order-Independent queues, which in turn encompass redundancy systems and stochastic matching models, among others.

For this class of systems, we have introduced the notion of {\em balanced arrival control} policies, which generalizes the existing notion of balanced arrival rates for Whittle networks, to any quasi-reversible queueing system. We have shown that supplementing a quasi-reversible queueing system with a balanced arrival control policy, preserves the quasi-reversibility. This has allowed us, first, to obtain the stationary measures of the controlled system in an explicit form, and second, to write as a linear program the problem of finding the best balanced admission control in a quasi-reversible system. To illustrate the {\em cost of balance}, we have conducted  a comparison between the best balanced admission control and the optimal (unconstrained) control, on various simple examples. 

Last, we have shown the potential advantage of balanced admission control in quasi-reversible systems, to deploy a policy-gradient reinforcement learning (PGRL) algorithm in order to optimize the admission control. We have conducted a first, simple, case study, by implementing a PGRL algorithm to a redundancy model with compatibility constraints, and analyzed a set of numerical results on this model. 

This paper thus provides the theoretical basis for the optimization of quasi-reversible systems with RL algorithms, and consists of a first step in that direction. It naturally opens the way for future work. First, for a better understanding of the optimality gap of balanced policies (which are crucial in the present framework) and optimal, unconstrained, policies. This raises the following complementary open question: is there a broader class of policies (beyond balanced policies) that also preserve quasi-reversibility, and is still convenient to parameterize? 

Second, the last part of this work suggests that PGRL algorithms can be fruitfully applied to a wide range of multi-class Whittle networks and OI queues, among which, open (multi-class) Jackson networks, 
redundancy models, and matching queues. The optimization of admission control for the last two classes of models is the topic of our current investigations on this class of problems. 

Last, many {\em closed} multi-class systems also enjoy the quasi-reversibility property. A balanced control would then be interpreted as an {\em internal} load balancing rather than an admission control. While it appears that the procedure introduced in this work is also amenable to optimize this internal control, this is left for future work.



\appendix

\section{Linear programs for reversible queueing systems} \label{app:lp_reversible_queueing_system}

Consider a quasi-reversible queueing system $Q = (\cS, | \cdot |, q)$
with $| \cdot | = \textrm{Id}$,
so that $\cS = \cX$,
$\cS_x = \{x\}$ for each $x \in \cX$,
and $q$ is reversible (see \Cref{sec:reversibility}).
In this case, all policies trivially satisfy the microstate-independence
condition of \Cref{def:balanced_policies},
and the only difference between an arbitrary admission policy
and balanced policies is the balance condition.
Problem~\eqref{eq:lp_identity_1} below is a rewriting of~\eqref{eq:problem},
accounting for the fact that $\cS = \cX$:
\begin{subequations}
	\label{eq:lp_identity_1}
	\begin{align}
		\label{eq:lp_identity_1a}
		\underset{\substack{\gamma: \cX \to [0, 1]^n, \\ \Pi: \cX \to [0, 1]}}{\textrm{Maximize}} \quad
		&G_\gamma \triangleq \sum_{x \in \cX} \Pi(x) \rcont(x)
		+ \sum_{x, y \in \cS} \Pi(x) q_\gamma(x, y) \rdisc(x, y), \\
		\label{eq:lp_identity_1b}
		\textrm{Subject to} \quad
		& \Pi(x) \sum_{y \in \cS} q_\gamma(x, y)
		= \sum_{y \in \cS} \Pi(y) q_\gamma(y, x),
		\quad \text{for each } x \in \cX, \\
		\label{eq:lp_identity_1c}
		&\sum_{x \in \cS} \Pi(x) = 1.
	\end{align}
\end{subequations}
This, in turn, is equivalent to the following \gls{LP}:
\begin{subequations}
	\label{eq:lp_identity_2}
	\begin{align}
		\label{eq:lp_identity_2a}
		\underset{\substack{\eta: \cX \to [0, 1]^n, \\ \Pi: \cX \to [0, 1]}}{\textrm{Maximize}} \quad
		&G_\gamma \triangleq \sum_{x \in \cX} \Pi(x) \rcont(x)
		+ \sum_{x, y \in \cX} \eta(x, y) \rdisc(x, y), \\
		\label{eq:lp_identity_2b}
		\textrm{Subject to} \quad
		& \sum_{y \in \cX} \eta(x, y)
		= \sum_{y \in \cX} \eta(y, x),
		\quad \text{for each } x \in \cX, \\
		\label{eq:lp_identity_2c}
		&\sum_{x \in \cX} \Pi(x) = 1, \\
		\label{eq:lp_identity_2d}
		&\eta(x, x + e_i) \le \Pi(x) q(x, x + e_i),
		\quad \text{for each } x \in \cX \text{ and } i \in \un, \\
		\label{eq:lp_identity_2e}
		&\eta(x, y) = \Pi(x) q(x, y),
		\quad \text{for each } x \in \cX \text{ and } y \in \cX \setminus \{x + e_i; i \in \un\}.
	\end{align}
\end{subequations}
Equivalence between
Problems~\eqref{eq:lp_identity_1} and~\eqref{eq:lp_identity_2}
is established by injecting the definition~\eqref{eq:tilde-q} of $q_\gamma$,
and letting
\begin{align*}
	\eta(x, y) &= \begin{cases}
		\Pi(x) q(x, x + e_i) \gamma_i(x)
		&\text{if $y = x + e_i$ for some $i \in \un$}, \\
		\Pi(x) q(x, y)
		&\text{otherwise}.
	\end{cases}
\end{align*}
Problem~\eqref{eq:lp_identity_2} can be further simplified
by recalling that $q(x, y) = 0$
for each $x, y \in \cX$
such that $y \notin \{x - e_i; i \in \un\} \cup \{x\} \cup \{x + e_i; i \in \un\}$.
Lastly, one can verify that in this special case, restricting the problem to balanced policies boils down to replacing the balance equations~\eqref{eq:lp_identity_2b} with local balance equations, so as to obtain:
\begin{subequations}
	\label{eq:lp_identity_3}
	\begin{align}
		\label{eq:lp_identity_3a}
		\underset{\substack{\eta: \cX \to [0, 1]^n, \\ \Pi: \cX \to [0, 1]}}{\textrm{Maximize}} \quad
		&G_\gamma \triangleq \sum_{x \in \cX} \Pi(x) \rcont(x)
		+ \sum_{x, y \in \cX} \eta(x, y) \rdisc(x, y), \\
		\label{eq:lp_identity_3b}
		\textrm{Subject to} \quad
		& \eta(x, y) = \eta(y, x),
		\quad  \text{ and } x, y \in \cX, \\
		\label{eq:lp_identity_3c}
		&\sum_{x \in \cX} \Pi(x) = 1, \\
		\label{eq:lp_identity_3d}
		&\eta(x, x + e_i) \le \Pi(x) q(x, x + e_i),
		\quad  \text{ and } x \in \cX \text{ and } i \in \un, \\
		\label{eq:lp_identity_3e}
		&\eta(x, y) = \Pi(x) q(x, y),
		\quad  \text{ and } x \in \cX \text{ and } y \in \cX \setminus \{x + e_i; i \in \un\}.
	\end{align}
\end{subequations}

\section{Proofs of \Cref{theo:polytope}, \Cref{prop:polytope}, and \Cref{prop:polytope_bis}} \label{app:polytope}

We prove the results in the following order.
First, we prove \Cref{prop:polytope}, which gives the vertices of the set $\cB^{\textrm{ac}}_\cX$.
Then, we build on this proof to show \Cref{prop:polytope_bis}, which characterizes the vertices of the set $\tilde\cB^{\textrm{ac}}_\cX$.
Lastly, we use \Cref{prop:polytope_bis} to prove our final result, \Cref{theo:polytope}. 

\begin{proof}[Proof of \Cref{prop:polytope}]
	Consider a finite Ferrers set $\cX \subseteq \bN^n$
	and let $m = \card(\cX)$.
	We proceed in three steps.
	
	\vspace{.2cm}
	
	\noindent \emph{$\cB_\cX$ is a convex polytope.}
	It follows from the fact that $\cB_\cX$
	can be seen as a set of vectors (with coordinates in~$\cX$ and values in~$\bR_{\ge 0}$)
	defined by linear constraints,
	both equality ($\Gamma(0) = 1$)
	and inequality ($\Gamma$ is non-negative and component-wise non-increasing).
	
	\vspace{.2cm}
	
	\noindent \textit{Deterministic function are extreme points of~$\cB_\cX$.}
	Let $\cA \in \cF_\cX$.
	Assume for the sake of contradiction that there exist
	$\alpha \in (0, 1)$ and two distinct functions $\Gamma, \Delta \in \cB_\cX$
	such that $\one_\cA = \alpha \Gamma + (1 - \alpha) \Delta$.
	Since $\Gamma \neq \Delta$, there exists $x \in \cX$ such that $\Gamma(x) \neq \Delta(x)$.
	By exchanging symbols if necessary, we can assume without loss of generality that
	$\Gamma(x) > \Delta(x)$.
	We now distinguish two cases:
	\begin{itemize}
		\item If $x \notin \cA$, then
		$0 = \one_\cA(x) = \alpha \Gamma(x) + (1 - \alpha) \Delta(x) > \Delta(x)$.
		It follows that $\Delta(x) < 0$,
		which contradicts our assumption that $\Delta \in \cB_\cX$.
		\item If $x \in \cA$, then
		$1 = \one_\cA(x) = \alpha \Gamma(x) + (1 - \alpha) \Delta(x) < \Gamma(x)$.
		As $\Gamma$ is non-increasing, it follows that $\Gamma(0) > 1$,
		which contradicts our assumption that $\Gamma \in \cB_\cX$.
	\end{itemize}
	Both cases lead to a contradiction,
	hence $\one_\cA$ is an extreme point of~$\cB_\cX$.
	
	\vspace{.2cm}
	
	\noindent \textit{Every element of~$\cB_\cX$ is
		a convex combination of deterministic balance functions.}
	Let $\Gamma \in \cB_\cX$.
	We first describe our approach to write $\Gamma$
	as a convex combination of deterministic balance functions,
	and then we prove the validity of this approach
	(e.g., that the coefficients of the combination
	are indeed non-negative and sum up to one).
	
	Starting with $\cV_0 = \cX$, let us define recursively,
	for each $k \in \{1, 2, \ldots, m\}$,
	\begin{align}
		\label{eq:yn}
		&y_{k-1} \in \argmin_{x \in \partial \cV_{k-1}} \Gamma(x), \\
		\label{eq:Vn}
		&\cV_k = \cV_{k-1} \setminus \{y_{k-1}\},
	\end{align}
	where $\partial \cV = \{x \in \cV; x \not\le y \text{ for each } y \in \cV\}$
	is the set of maximal (or non-dominated) points of~$\cV$, for each $\cV \in \cF_\cX$.
	By construction, $y_{m - 1} = 0$ and $\cV_{m} = \emptyset$,
	as we remove maximal elements one by one starting from $\cV_0 = \cX$.
	Furthermore, $\cV_0$, $\cV_1$, \ldots, $\cV_{m-1}$
	are Ferrers subsets of~$\cX$,
	hence $\one_{\cV_k} \in \cB_\cX$ for each $k \in \{0, 1, \ldots, m - 1\}$.
	We will verify that we can write $\Gamma$ as follows:
	\begin{align} \label{eq:convex-combination}
		\Gamma(x)
		&= \sum_{k = 0}^{m - 1} \alpha_k \one_{\cV_k}(x)
		\quad \text{for each } x \in \cX,
	\end{align}
	where $\alpha_0 \triangleq \Gamma(y_0)$,
	and $\alpha_k \triangleq \Gamma(y_k) - \Gamma(y_{k-1})$
	for each $k \in \{1, 2, \ldots, m - 1\}$.
	We will also verify that $\alpha_k \ge 0$ for each $k \in \{0, 1, \ldots, m - 1\}$
	and that $\sum_{k = 0}^{m - 1} \alpha_k = 1$,
	which will conclude the proof.
	Let us prove the required points:
	\begin{itemize}
		\item \textit{Proof of~\eqref{eq:convex-combination}.}
		Let us first verify~\eqref{eq:convex-combination} for each $x \in \cX$.
		Since the sequence $y_0$, $y_1$, \ldots, $y_{m - 1}$
		visits every element of~$\cX$,
		there exists (a unique) $\ell \in \{0, 1, \ldots, m - 1\}$
		such that $x = y_\ell$.
		By~\eqref{eq:Vn}, if follows that
		$x \in \cV_k$ for each $k \in \{0, 1, \ldots, \ell\}$
		and $x \notin \cV_k$ for each $k \in \{\ell + 1, \ell + 2, \ldots, m - 1\}$.
		The conclusion follows by simplifying the right-hand side of~\eqref{eq:convex-combination}:
		\begin{align*}
			\sum_{k = 0}^{m - 1} \alpha_k \one_{\cV_k}(x)
			&= \sum_{k = 0}^{\ell} \alpha_k \one_{\cV_k}(y_\ell)
			= \Gamma(y_0) + \sum_{k = 0}^\ell (\Gamma(y_k) - \Gamma(y_{k-1}))
			= \Gamma(y_\ell) = \Gamma(x).
		\end{align*}
		\item \textit{Proof that $\alpha_k \ge 0$ for each $k \in \llbracket 0,m - 1 \rrbracket$.}
		We have directly $\alpha_0 = \Gamma(y_0) \ge 0$
		since $\Gamma$ takes values in~$\bR_{\ge 0}$.
		Now, given $k \in \llbracket 1,m - 1 \rrbracket$,
		we want to prove that $\alpha_k \ge 0$,
		i.e., that $\Gamma(y_k) \ge \Gamma(y_{k-1})$.
		We know that $y_k \in \cV_{k-1}$
		(because $y_k \in \cV_k = \cV_{k-1} \setminus \{y_{k-1}\}$),
		hence it suffices to prove that
		$\Gamma(x) \ge \Gamma(y_{k-1})$ for each $x \in \cV_{k-1}$.
		For each $x \in \cV_{k-1}$,
		we can verify that $\Gamma(x) \ge \Gamma(y) \ge \Gamma(y_{k-1})$,
		where $y \in \delta \cV_{k-1}$ is such that $x \le y$;
		the first inequality follows because $\Gamma$ is component-wise non-increasing,
		and the second follows from the definition of $y_{k-1}$.
		\item \textit{Proof that $\sum_{k = 0}^{m - 1} \alpha_k = 1$.}
		This follows by observing that this is a telescoping sum
		and recalling that $\Gamma(y_{m - 1}) = \Gamma(0) = 1$. \qedhere
	\end{itemize}
\end{proof}

\begin{proof}[Proof of \Cref{prop:polytope_bis}]
	Consider a finite Ferrers set $\cX \subseteq \bN^n$
	and let $m = \card(\cX)$.
	Buliding on the proof of \Cref{prop:polytope},
	we again proceed in three steps.
	
	\vspace{.2cm}
	
	\noindent \emph{$\tilde\cB_\cX$ is a convex polytope.}
	Like for $\cB_\cX$, it follows from the fact that $\tilde\cB_\cX$
	can be seen as a set of vectors
	(with coordinates in~$\cX$ and values in~$\bR_{\ge 0}$)
	defined by linear constraints,
	both equality ($\sum_{x \in \cX} \Pi(x) \Gamma(x) = 1$)
	and inequality ($\Gamma$ is non-negative and component-wise non-increasing).
	
	\vspace{.2cm}
	
	\noindent \textit{Deterministic function are extreme points of~$\tilde\cB_\cX$.}
	Let $\cA \in \cF_\cX$ and
	assume for the sake of contradiction that there exist
	$\alpha \in (0, 1)$
	and two distinct functions $\Gamma, \Delta \in \tilde\cB_\cX$
	such that $\Gamma_\cA = \alpha \Gamma + (1 - \alpha) \Delta$.
	Let $\mu = 1 / (\sum_{y \in \cA} \Pi(y))$.
	
	Let us first verify that
	$\Gamma(x) = \Delta(x) = 0$
	for each $x \in \cX \setminus \cA$.
	Assume for the sake of contradiction that
	there exists $x \in \cX \setminus \cA$ such that
	$\Gamma(x) \neq \Delta(x)$.
	By exchanging symbols if necessary,
	we can assume without loss of generality
	that $\Gamma(x) > \Delta(x)$.
	It follows that
	$0 = \Gamma_\cA(x) = \alpha \Gamma(x) + (1 - \alpha) \Delta(x) > \Delta(x)$,
	which contradicts our assumption that $\Delta \in \tilde\cB_\cX$.
	Therefore, for each $x \in \cX \setminus \cA$, we have
	$\Gamma(x) = \Delta(x)$ which, combined with the fact that
	$0 = \Gamma_\cA(x) = \alpha \Gamma(x) + (1 - \alpha) \Delta(x)$,
	implies that $\Gamma(x) = \Delta(x) = 0$.
	
	Since $\Gamma \neq \Delta$, there exists $x \in \cX$ such that $\Gamma(x) \neq \Delta(x)$.
	The previous paragraph implies that $x \in \cA$.
	By exchanging symbols if necessary, we can assume without loss of generality that $\Gamma(x) > \Delta(x)$ which,
	combined with $\mu = \Gamma_\cA(x) = \alpha \Gamma(x) + (1 - \alpha) \Delta(x)$, implies
	$\Gamma(x) > \mu > \Delta(x)$.
	Since $\Gamma$ is non-increasing, it follows that
	$\Gamma(0) > \mu$.
	Since we also have
	$\mu = \Gamma_\cA(0) = \alpha \Gamma(0) + (1 - \alpha) \Delta(0)$,
	we obtain $\Delta(0) < \mu$.
	Since $\Delta$ is also non-increasing,
	we obtain $\Delta(x) < \mu$ for each $x \in \cA$.
	All in all, we obtain
	\begin{align*}
		\sum_{x \in \cX} \Pi(x) \Delta(x)
		&\overset{\text{(a)}}=
		\sum_{x \in \cA} \Pi(x) \Delta(x)
		\overset{\text{(b)}}<
		\sum_{x \in \cA} \Pi(x) \mu
		\overset{\text{(c)}}=
		1,
	\end{align*}
	where (a) follows from the fact that $\Delta(x) = 0$ for each $x \in \cX \setminus \cA$,
	(b) from the fact that $\Delta(x) < \mu$ for each $x \in \cA$,
	and (c) from the definition of $\mu$.
	This contradicts our assumption that $\Delta \in \tilde\cB_\cX$.
	
	\vspace{.2cm}
	
	\noindent \textit{Every element of~$\tilde\cB_\cX$ is
		a convex combination of deterministic balance functions.}
	Let $\Gamma \in \tilde\cB_\cX$.
	Start off from the decomposition derived
	in the proof of \Cref{prop:polytope}, we have
	\begin{align*}
		\Gamma(x) &= \sum_{k = 0}^{m-1} \alpha_k \indicator{x \in \cV_k},
		\quad \text{for each } x \in \cX,
	\end{align*}
	where $y_0, y_1, \ldots, y_{m-1}$ and $\cV_0, \cV_1, \ldots, \cV_{m-1}$
	are built according to~\eqref{eq:yn}--\eqref{eq:Vn},
	$\alpha_0 = \Gamma(y_0)$,
	and $\alpha_k = \Gamma(y_k) - \Gamma(y_{k-1})$
	for each $k \in \llbracket 1,m - 1 \rrbracket$.
	We also showed in the proof of \Cref{prop:polytope}
	that $\alpha_k \ge 0$ for each $k \in \llbracket 0,m - 1 \rrbracket$.
	We can rewrite this equality as
	\begin{align*}
		\Gamma(x)
		&= \sum_{k = 0}^{m-1} \tilde\alpha_k \Gamma_{\cV_k}(x),
		\quad \text{for each } x \in \cX,
	\end{align*}
	where $\tilde\alpha_k = \alpha_k \sum_{y \in \cV_k} \Pi(y)$
	for each $k \in \llbracket 0,m - 1 \rrbracket$.
	We have $\tilde\alpha_k \ge 0$ for each $k \in \llbracket 0,m - 1 \rrbracket$.
	All that remains is to check that $\sum_{k = 0}^{m-1} \tilde\alpha_k = 1$:
	\begin{align*}
		\sum_{k = 0}^{m-1} \tilde\alpha_k
		&\overset{\text{(a)}}=
		\sum_{k = 0}^{m-1} \alpha_k \sum_{y \in \cV_k} \Pi(y), \\
		&\overset{\text{(b)}}=
		\Gamma(y_0) \sum_{y \in \cV_0} \Pi(y)
		+ \sum_{k = 1}^{m-1} (\Gamma(y_k) - \Gamma(y_{k-1})) \sum_{y \in \cV_k} \Pi(y), \\
		&\overset{\text{(c)}}=
		\Gamma(y_0) \sum_{y \in \cV_0} \Pi(y)
		+ \sum_{k = 1}^{m-1} \Gamma(y_k) \sum_{y \in \cV_k} \Pi(y)
		- \sum_{k = 0}^{m-2} \Gamma(y_k) \sum_{y \in \cV_{k+1}} \Pi(y), \\
		&\overset{\text{(d)}}=
		\sum_{k = 0}^{m-2}
		\Gamma(y_k) \left( \sum_{y \in \cV_k} \Pi(y) - \sum_{y \in \cV_{k+1}} \Pi(y) \right)
		+ \Gamma(y_{m-1}) \sum_{y \in \cV_{m-1}} \Pi(y), \\
		&\overset{\text{(e)}}=
		\sum_{k = 0}^{m-2}
		\Gamma(y_k) \Pi(y_k)
		+ \Gamma(0) \Pi(0)
		\overset{\text{(f)}}=
		1,
	\end{align*}
	where (a) follows from the definition of $\tilde\alpha_k$
	for $k \in \llbracket 0,m - 1 \rrbracket$,
	(b) follows from the definition of $\alpha_k$
	for $k \in \llbracket 0,m - 1 \rrbracket$,
	(c) follows by applying the distributivity property and then making a change of variable in the second sum,
	(d) follows by rearranging the terms,
	(e) follows by applying \eqref{eq:yn}--\eqref{eq:Vn},
	and in particular by recalling that
	$\cV_{k+1} = \cV_k \setminus \{y_k\}$ for each $k \in \llbracket 0,m - 2 \rrbracket$, $y_{m-1} = 0$, and $\cV_{m-1} = \{0\}$,
	and (f) follows by recalling that $y_0, y_1, \ldots, y_{m-1}$ is a permutation of $\cX$ and from our assumption that $\Gamma \in \tilde\cB_\cX$.
\end{proof}

\begin{proof}[Proof of \Cref{theo:polytope}]
	Consider an admission control problem
	$\mathcal{P} = (Q, \rcont, \rdisc)$,
	where $Q = (\cS, |\cdot|, q)$.
	Let $\Pi$ be a stationary measure of~$Q$.
	Since Problem~\eqref{eq:problem_balanced}
	is a linear optimization problem
	whose set of feasible solutions is $\tilde\cB_{|\cS|}$,
	the conclusion follows
	from the fundamental theorem of linear programming
	and by applying \Cref{prop:polytope_bis} to $|\cS|$.
\end{proof}

\section{Proof of \Cref{lemma:whittle-queueing}}
\label{sec:preuvewhittle-queueing}
First observe that \eqref{eq:transitionWhittle} implies that, from any state $s\in\cS$, the only possible transitions are to 
a state $t$ of $\cS_{|s|}$ (if $t$ is of the form $t=s-e_{ik}+e_{i\ell}$), of $\bigcup_{i = 1}^n \cS_{|s| + e_i}$ if $t$ is of the form $s+e_{ik}$, 
or of $\bigcup_{i = 1}^n \cS_{|s| - e_i}$ if $t$ is of the form $s-e_{ik}$. Now, observe the following consequence of irreducibility which, as we will see below, guarantees that the model at stake satisfies 
Assumption \ref{ass:unichain}. 
\begin{lemma}\label{lemma:whittle-irred}
	Let $s\in\cS\setminus\{\mathbf 0\}$. Then, for any $i\in\llbracket 1,m \rrbracket$ such that $|s|_i>0$, 
	\begin{enumerate}
		\item There exist a positive integer $r$, and $u_0,u_1,\cdots,u_{r-1}\in\cS$ such that $|u_0|+e_i=|u_1|=\cdots=|u_{r-1}|=|s|$, and such that 
		\[\left(\prod_{j=0}^{r-2}q(u_j,u_{j+1})\right)q(u_{r-1},s)>0.\]
		\item There exist a positive interger $m$, and 
		$v_{1},\cdots,v_{m}\in\cS$ such that $|v_{m}|+e_i=|v_{m-1}|=\cdots=|v_{1}|=|s|$, and such that 
		\[q(s,v_{1})\left(\prod_{j=1}^{m-1}q(v_{j},v_{j+1})\right)>0.\]
	\end{enumerate}
\end{lemma}
\begin{proof}
	Fix $s\in\cS$ and $i\in\llbracket 1,n \rrbracket$ such that $|s|_i>0$. Let $k\in\llbracket 1,n \rrbracket$ be such that $s_{ik}>0$, and take $p\equiv p_{ik}$ and $k_1,\cdots k_{p}$ be such that \eqref{eq:whittle-irred} holds. Let $r\in\llbracket 1,p \rrbracket$ be such that $k_{r}=k$. 
	It then readily follows that, by defining the descending 
	sequence
	\[\begin{cases}
		u_{r} &=s;\\
		u_{j} &=u_{j+1}+e_{ik_j}-e_{ik_{j+1}},\quad j\in\llbracket 1,r-1 \rrbracket,\mbox{ if }r\ge 2;\\
		u_0 &=u_1-e_{ik_1},
	\end{cases}\] we have 
	\[|u_0|+e_i=|u_1|=\cdots =|u_{r-1}|=|s|.\]
	Moreover, it follows from \eqref{eq:whittle-irred} that 
	\begin{equation*}
		\prod_{j=0}^{r-1} q(u_j,u_{j+1})
		= P_{0,ik_1}\phi_0(u_1-e_{ik_1})\left(\prod_{j=1}^{r-1}P_{ik_j,ik_{j+1}}\phi_{ik_j}(u_j) \right)>0,
	\end{equation*}
	proving the first assertion. To show the second one we reason similarly, by letting 
	\[\begin{cases}
		v_{0}&=s;\\
		v_{j+1} &=v_{j}+e_{ik_{r+j+1}}-e_{ik_{r+j}},\quad j=\llbracket 0,p-r-1 \rrbracket,\mbox{ if }r\le p-1;\\
		v_{p-r+1}&=v_{p-r}-e_{ik_{p}},
	\end{cases}\]
	in a way that 
	\[|v_{p-r+1}|+e_i=|v_{p-r}|=\cdots =|v_{1}|=|s|,\]
	and 
	\begin{equation*}
		\prod_{j=0}^{p-r} q(v_{j},v_{j+1})
		= \left(\prod_{j=0}^{p-r-1}P_{ik_{r+j},ik_{r+j+1}}\phi_{ik_{r+j}}(v_{j}) \right)P_{ik_{p},0}\phi_{ik_{p}}(v_{p-r})>0,
	\end{equation*}
	concluding the proof by setting $m=p-r+1$. 
\end{proof}
We are now ready to show that multi-class Whittle networks are indeed a queueing system, in the sense of \Cref{sec:queueing_systems}. 
\begin{proof}[Proof of \Cref{lemma:whittle-queueing}] 
	We have to check that the system satisfies Assumptions \ref{ass:unichain}, and we show that it is so for $\cS_\rec\equiv\cS.$ 
	For any $s$ it suffices to apply Assertion 1 of \Cref{lemma:whittle-irred}, 
	a number $\|s\|:=\sum_{i\in\llbracket 1,n \rrbracket,k\in\llbracket 1,m \rrbracket} s_{ik}$ of times, to obtain, by induction, the existence of 
	$\|s\|$ positive integers $r^1,\cdots,r^{\| s \|}$ and of 
	$\|s\|$ array of elements $\left(\left(u^i_0,...,u^i_{r^i-1}\right)\right)_{i=1,\cdots,\|s\|}$, such that $u^1_0=\mathbf 0$ and such that, 
	by setting $u^{\|s\|+1}_{0}:=s$ we have 
	$$|u^i_0|<|u^i_{1}|=|u^i_{2}|=\cdots=|u^i_{r^i-1}|=|u^{i+1}_{0}|,\,\,\mbox{ for all }i\in \llbracket 1,\|s\| \rrbracket$$
	and
	\begin{equation*}
		\prod_{i=1}^{\|s\|}\left(\prod_{j=0}^{r^i-2}q(u^i_j,u^i_{j+1})\right)q(u^i_{r^i-1},u^{i+1}_{0})>0,
	\end{equation*}
	implying Assertion 1.
	Likewise, Assertion 3 follows by applying Assertion 2 of \Cref{lemma:whittle-irred}, 
	a number $\|s\|:=\sum_{i\in\llbracket 1,n \rrbracket,k\in\llbracket 1,m \rrbracket} s_{ik}$ of times, to obtain, by induction, the existence of 
	$\|s\|$ positive integers $m^1,\cdots,m^{\|s\|}$ and of 
	$\|s\|$ array of elements $\left(\left(v^i_{1},...,v^i_{m^i}\right)\right)_{i=1,\cdots,\|s\|}$, such that $v^{\|s\|}_{m^{\|s\|}}=\mathbf 0$ and such that, by setting $v^{0}_{m^0}:=s$ we have 
	$$|v^{i-1}_{m^{i-1}}|=|v^i_{1}|=\cdots=|v^i_{m^i-1}|>|v^i_{m^i}|,\,\,\mbox{ for all }i\in \llbracket 1,\|s\| \rrbracket$$
	and
	\begin{equation*}
		\prod_{i=1}^{\|s\|}q(v^{i-1}_{m^{i-1}},v^{i}_{1})\left(\prod_{j=1}^{m^i-1}q(v^i_{j},v^i_{j+1})\right)>0,
	\end{equation*}
	in view of Assertion 3. 
	To finish, to check Assertion 4 of Assumptions \ref{ass:unichain}, take two elements $u,s\in\cS$ such that $|u|\le |s|$. 
	As above, we can apply Assertion 2 of \Cref{lemma:whittle-irred} a number $\|u\|$ of times, to construct a path from $\mathbf 0$ to $u$ 
	having positive probability,
	whence the fact that $\cS$ is coordinate-convex. This concludes the proof of \Cref{ass:unichain}.  
\end{proof}

\section{Proof of \Cref{lemma:OI-queueing}}
\label{sec:preuveOI-queueing}
In view of \eqref{eq:transitionWhittle}, from any state $s=s_1\cdots c_\ell\in\cS$, the only possible transitions are to 
a state of the form $s_1\cdots s_\ell i$, i.e., an element of $\bigcup_{i = 1}^n \cS_{|s| + e_i}$, or to a state  
of the form $s_1\cdots s_{p-1}s_{p+1}\cdots s_\ell$, that is, an element of $\bigcup_{i = 1}^n \cS_{|s| - e_i}$.  

First, fix $s=s_1\cdots s_\ell\in\cS\setminus\{\varnothing\}$. 
Set $t_1=\varnothing$ and, for any $k\in \llbracket 2,\ell+1 \rrbracket$, $t_k=s_1\cdots s_{k-1}$. Then, for all $k\in\llbracket 1,\ell \rrbracket$, as $\cS$ is prefix-stable, 
$t_k$ is an element of $\cS$. Moreover we have $|t_k|=|t_{k+1}|-e_{s_k}\le |t_{k+1}|$, and also $q(t_k,t_{k+1})=\nu_{s_k}>0$. This shows Assumption \ref{ass:unichain-1}. Then, to show Assumption \ref{ass:unichain-3}, take 
again $s=s_1\cdots s_\ell\in\cS\setminus\{\varnothing\}$. Observe that, by setting $s_1\cdots s_0 \equiv \varnothing$, 
\[\sum_{p=1}^{\ell}\Delta\mu(s_1\cdots s_p)=\sum_{p=1}^{\ell} \left(\mu(s_1\cdots s_p) - \mu(s_1\cdots s_{p-1}) \right)=\mu(s)> 0,\]
implying that $\Delta\mu(s_1\cdots s_p)>0$ for some $p\in\llbracket 1,\ell \rrbracket$. Therefore, setting $t_1=s$, for such $p$ there exists $t_2=s_1\cdots s_{p-1}s_{p+1}\cdots s_\ell$, such that 
$|t_2|=|t_1|-e_{s_p}\le |t_1|$ and $q(t_1,t_2)=\Delta\mu(s_1\cdots s_p)>0$. We can then iterate this argument by induction, to prove the existence of a sequence of elements 
$s=t_1,t_2,\cdots,t_{\ell+1}=\emptyset$, such that $q(t_k,t_{k+1})>0$ and $|t_k|\ge |t_{k+1}|$ for all $k\in\llbracket 1,\ell \rrbracket$, completing the proof of Assertion \ref{ass:unichain-3}. 

It remains to check Assertion \ref{ass:unichain-4}. Let $s=s_1\cdots s_\ell\in\cS\setminus\{\varnothing\}$, and let 
$y\in \in\bN^n$ be such that $y\le |s|$. This means exactly that for all $i\in\un$, $|s|_i - y_i\ge 0$. Then, set $\cI=\{i\in\un\,;´\,|s|_i-y_i>0\}$, and 
for any $i\in\cI$, $\alpha_i=|s|_i-y_i>0$. Then, if we denote by $i_1,\cdots,i_q$, the elements of $\cI$, let $\tilde s$ be a permuted version of 
$s$ ending with the suffix
\[\underline s:=\underbrace{i_1\cdots i_1}_{\alpha_{i_1} \mbox{ \scriptsize{times} }}\underbrace{i_2\cdots i_2}_{\alpha_{i_2} \mbox{ \scriptsize{times} }} \cdots \underbrace{i_q\cdots i_q}_{\alpha_{i_q} \scriptsize{\mbox{ times }}}.\] 
Then, by permutation stability, $\tilde s$ is an element of $\cS$. Moreover, the prefix $\breve s$ obtained from $\tilde s$ by erasing the suffix $\underline s$, 
is such that $|\breve s|=y$ by construction, and it is an element of $\cS$ by prefix-stability. This concludes the proof of Assumptions \ref{ass:unichain}.

\end{document}